\documentclass[final,nopreprintline]{elsarticle}

\usepackage{hyperref} 
\usepackage{amssymb}
\usepackage{amsmath}
\usepackage{amsthm, ulem}
\usepackage{tabularx}
\usepackage{mathrsfs}

\newtheorem{theorem}{Theorem}[section]
\newtheorem{lemma}[theorem]{Lemma}

\newtheorem{definition}{Definition}[section]
\newtheorem{remark}{Remark}[section]
\newtheorem{example}{Example}[section]
\usepackage{geometry}
\usepackage{color}
\journal{Arxiv}
\allowdisplaybreaks

\begin{document}

\begin{frontmatter}

\title{From Score Learning to Discretized Sampling: An End-to-End Generalization  Analysis of Diffusion Models} 
% Deep Limit Convergence from Deep Neural Networks with Dense Layer Connectivities to Integral Equations
\tnotetext[mytitlenote]{huangjsh@mail.nankai.edu.cn; ymjiangnk@nankai.edu.cn; wucl@nankai.edu.cn.}

\author{Jinshu Huang$^{1}$, Yiming Jiang$^{1}$ and Chunlin Wu$^{1,*}$}

\address{$^{1}$ School of Mathematical Sciences, Nankai University, Tianjin, China}

%\ead{\mailto{2120190037@mail.nankai.edu.cn}, \mailto{xtai@norceresearch.no}, \mailto{wucl@nankai.edu.cn}}

%\author{Jinshu Huang\fnref{myfootnote}}
%\address{Nankai university, Tianjin}


\begin{abstract}
Despite the empirical success of score-based diffusion models, a complete theoretical understanding of how finite-sample learning, network parameterization, and numerical discretization jointly dictate generative quality remains underdeveloped. Existing sampling analyses often evaluate the generative performance conditional on an oracle score or a pre-specified error threshold. In this work, we establish a unified convergence and generalization framework for score-based diffusion models parameterized by practical ResNet-type architectures. 
We analyze the generalization and convergence properties from the practical finite-sample, discrete-time learning problem of the score function to the ideal continuous-time, population-level objective. 
Based on the generalization result of the learning problem of score function, we analyze the sampling process induced by the learned score function and provide an end-to-end total variation distance estimate for the generated terminal distribution. This estimate explicitly decomposes the overall generative error into four interpretable components: the truncation error of the forward process, the reverse-time discretization error, the generalization error incorporating both finite data and forward-time discretization, and the training optimization gap. Our results quantitatively characterize how the training sample size, temporal discretization grids, and optimization accuracy jointly control the final fidelity of samples generated by diffusion models.

\end{abstract}


\end{frontmatter}

%\linenumbers

\section{Introduction}
\label{sec: 1}
Diffusion models (DMs) are a class of generative models that synthesize high-quality data samples (e.g. images) from random noise. Owing to their strong empirical performance, DMs have been widely adopted in both research and practical applications \cite{yang2023comprehensivesurvey, chen2024overview, huang2025survey}, and have already been incorporated into large-scale systems such as DALL$\cdot$E \cite{ramesh2022hierarchical}, Imagen \cite{saharia2022photorealistic}, Stable Diffusion\cite{rombach2022high} and Sora \cite{openai2024sora}.

Informally, a diffusion model consists of two time-reversed processes:
\begin{itemize}
	\item  {Forward process}. A clean data sample is gradually perturbed by noise until it approaches an isotropic Gaussian distribution. This process is typically used to construct training pairs.
	
	\item {Backward process}. The generation task is to learn and simulate the reverse-time dynamics that transform noise back into realistic data samples. The central challenge is to estimate the unknown reverse drift from finite data.
\end{itemize}
Early works \cite{song2019generative, ho2020denoising, song2021denoising} described diffusion models in a discrete-time Markovian framework. More recently, the stochastic differential equation (SDE) viewpoint has provided a natural and unified continuous-time formulation for both analysis and algorithm design \cite{song2021maximum, song2021scorebased, tang2025score}. In this paper, we work within the continuous-time score-based SDE framework.

Despite the remarkable empirical success of diffusion models, their theoretical foundations remain incomplete. In particular, it is still not fully understood how finite-sample learning, model parametrization, and numerical discretization jointly affect the final quality of generated samples. Existing sampling analyses \cite{lee2022convergence, de2022convergence, chen2023sampling, chen2023improved, lee2023convergence, tang2024contractive} typically assume access to an accurate or even exact score function, and therefore focus mainly on the discretization and stability properties of the reverse process. These results provide important convergence guarantees, but they do not fully capture the effect of learning the score from finite data.

A number of recent works have begun to address the connection between learning and sampling in diffusion models. For example, the work \cite{block2020generative} derives sample complexity bounds for score estimation and relates score approximation to sampling accuracy, while \cite{chen2023score} studies distribution estimation when the data lie on an unknown low-dimensional linear subspace. Li et al. \cite{li2023generalization} further connect training dynamics and generalization to the Kullback-Leibler divergence between generated and true distributions, but their analysis is limited to linear networks and convex optimization settings. More broadly, existing results do not yet provide a single framework that simultaneously captures finite-sample generalization, time discretization, optimization accuracy, and the temporal regularity of the SDE coefficients.
In this work, we develop a unified convergence and generalization theory for score-based diffusion models from both the learning and sampling perspectives. Our analysis is built on a learnable score function parameterized by practical ResNet-type architectures, and it explicitly propagates the learning error into the sampling error. The main contributions are summarized as follows.
\begin{itemize}
	\item We investigate the learning problem of the score function under both the practical discrete-time empirical risk and the idealized continuous-time population risk. We derive a uniform generalization bound for the learned score functions and establish the consistency of the optimal values and the subsequential convergence of the corresponding optimal solutions.
	
	\item We analyze the reverse-time sampling process induced by the learned empirical score function and its Euler--Maruyama discretization, deriving an total variation bound for the generated terminal distribution. This result yields a precise error decomposition that partitions the total generative discrepancy into four fundamental sources: the truncation error of the forward process, the numerical discretization error of the reverse-time sampler, the statistical generalization error stemming from both finite training data and the discrete-time approximation of the forward process, and the optimization error.
\end{itemize}
These results propagate finite-sample learning guarantees through the sampling analysis and thus remove the common oracle assumption on the score. They provide a data-dependent and quantitative characterization of how the training sample size, the generalization behavior of the learned score network, the time discretization, and the temporal smoothness of the SDE coefficients jointly determine the fidelity of samples generated by diffusion models.

The remainder of this paper is organized as follows. Section~\ref{sec: related work} reviews the relevant literature. Section~\ref{Sec: 2} introduces score-based diffusion models and formalizes their corresponding learning problems. Section~\ref{Sec: 3} delivers our main theoretical contributions, including the consistency and generalization properties of the learning problems of score function, as well as establishing the end-to-end generative error estimation. Section~\ref{Sec: 4} provides the detailed mathematical proofs for our main results. Finally, Section~\ref{sec: conclusion} concludes the paper.
% Section~\ref{Sec: 5} reports numerical experiments illustrating the theoretical findings. 

\textbf{Notations}: Below we collect a few notations that will be used throughout. Let $\mathcal{N}$ and $\mathcal{R}$ denote the set of natural and real numbers, respectively. 
Scalars are denoted with italic letters (e.g. $N \in \mathcal{N}$), vectors and matrices are denoted with lowercase and capital in straight letters (e.g., $\mathrm{c} \in \mathbb{R}^N$ and $\mathrm{W} \in \mathbb{R}^{N \times N}$). Vector- and matrix-valued functions of time variable $t$ are written in boldface, like $\bold{c}(\cdot)$ and $\bold{W}(\cdot)$.
For a continuous function $f \in \mathcal{C}([0,T], \mathcal{R})$, we denote by $\omega_f(\cdot)$ its modulus of continuity on $[0,T]$.
For a Lipschitz continuous function	$g$, its Lipschitz constant is denoted by $\mathrm{Lip}_g$.
{For a Lipschitz continuous function $f$, its Lipschitz constant is denoted by $\mathrm{Lip}_f$.}
The $l^2$ norms for vectors or matrices in Euclidean space are denoted as $\|\cdot\|_2$. 
Let $\mathcal{X}$ be a vector space. The notation $\mathcal{X}^N$ repents the Cartesian product space of $\underbrace{\mathcal{X} \times \ldots \times \mathcal{X}}_{N}$. The space of functions that are continuous on $\Omega \subset \mathbb{R}^d$ is denoted by $\mathcal{C}(\Omega)$. 

\section{Related Works}
\label{sec: related work}

While the theoretical literature on diffusion models is vast and rapidly evolving, our work primarily intersects with two active areas of research: the numerical analysis of the generative sampling process and the statistical theory of score estimation. 

\subsection{Sampling Theory for Diffusion Models}
A major focus in the existing literature is understanding the convergence properties of the reverse-time SDE as a generative mechanism. To establish these theoretical guarantees, early foundational works (e.g., \cite{lee2022convergence}) typically relied on the premise that the score function is well-approximated, deriving convergence rates in Wasserstein or total variation distances under strict data smoothness or log-concavity assumptions. 
Subsequent studies have significantly refined these bounds by relaxing data regularity requirements and improving path-measure techniques. For instance, De Bortoli~\cite{de2022convergence} extended convergence guarantees to distributions supported on low-dimensional manifolds, while Chen et al.~\cite{chen2023improved} established sampling bounds under minimal smoothness assumptions via refined KL divergence analysis. Chen et al.~\cite{chen2023sampling} quantified the precise dependency of the total variation (TV) distance on the $L^2$ score estimation error, showing that generative convergence is tractable provided the score error satisfies a specific threshold. Other works have further investigated improved rates and contractivity under various structural conditions \cite{lee2023convergence, tang2024contractive}. While these works rigorously justify the numerical SDE solvers, their sampling guarantees remain inherently conditional on achieving a pre-specified score estimation accuracy.

\subsection{Score Learning and Generalization} 
Another  line of research investigates how accurately the score function can be learned from empirical data. To connect this statistical estimation with the generative process, early work by Block et al.~\cite{block2020generative} provided sample complexity bounds that bridge score approximation with sampling performance.  Further studies have explored score matching under specific geometric priors or support constraints; for instance, Chen et al.~\cite{chen2023score} investigated score learning on unknown low-dimensional linear subspaces, while Pidstrigach~\cite{pidstrigach2022score} established conditions under which diffusion models robustly sample from an underlying data manifold and studied its implications for data memorization.  
To investigate parameterized score functions, Li et al.~\cite{li2023generalization} studied the evolution of the generalization gap along training dynamics, establishing bounds that explicitly depend on the training sample size and the network width under linear models. While their analysis elegantly demonstrates how early-stopped training can evade the curse of dimensionality, it remains restricted to linear architectures and focuses heavily on the network's capacity without factoring in the crucial numerical time-discretization errors inherent in practical SDE deployment. %In contrast, our work accommodates practical nonlinear deep networks (specifically ResNet-type architectures) within a realistic diffusion training framework. %Crucially, instead of isolating the generalization gap to network width, we jointly characterize how the final generative performance is co-constrained by both the training sample size and the number of time-discretization steps, providing a comprehensive, end-to-end convergence theory that closely aligns with practical diffusion implementations.

\subsection{Consistency and Discretization of Learning Problems}
Training a score-based DM requires solving a family of continuous-time, finite-sample stochastic optimization problems. The well-posedness and asymptotic consistency of empirical risk minimizers are core topics in classical statistical learning, typically analyzed via M-estimation theory \cite{van2000asymptotic} or variational limits such as $\Gamma$-convergence \cite{bo2022large, thorpe2018deep, huang2024on}.
In Section~\ref{Sec: 3.2}, we establish the existence and subsequential convergence of minimizers across the population, continuous-empirical, and fully discretized regimes. Crucially, while classical M-estimation typically abstracts away numerical grids, our framework explicitly characterizes how the temporal moduli of continuity of the SDE coefficients interact with statistical estimation errors when transitioning from continuous-time loss functions to their discrete-time training counterparts.

%\subsection{Position of This Work}
%The primary contribution of our work is integrating these traditionally isolated facets into a single, cohesive framework. Rather than assuming an oracle score, we explicitly decompose the end-to-end sampling discrepancy into four components: prior mismatch, discretization error, statistical generalization, and the optimization gap. This provides a clear, quantitative characterization of how training data volume, temporal discretization grids, and network optimization jointly dictate generative fidelity.

\section{Score-based diffusion models}
\label{Sec: 2}
Score-based diffusion models (SDMs) are a class of deep generative models that learn to approximate the score functions and use these estimates to transform noise into realistic samples. In this section, we review their formulation through the lens of SDEs.

\subsection{Forward SDE and its time reversal}
Let $p_{\rm data}$ denote the target data distribution on $\mathbb{R}^d$ and fix a time horizon $T>0$.  Consider the forward (noising) SDE
where \begin{equation}
	\mathrm{d} \bold{x}_t = \bold{f}(t, \bold{x}_t) \mathrm{d} t+ \boldsymbol{g}(t) \mathrm{d} \bold{B}_t, \quad \bold{x}_0 = \mathrm{x} \sim p_{\rm data},
	\label{equation: forward-time SDE}
\end{equation}
where $\bold{f}: \mathbb{R}_{+} \times \mathbb{R}^d \mapsto \mathbb{R}^d$ is a time-dependent drift function, $\boldsymbol{g}: \mathbb{R}_{+} \mapsto \mathbb{R}$ is a scalar-valued diffusion coefficient, and $\{\textbf{B}_t\}_{t \ge 0}$ is a $d$-dimension Wiener process. Under standard Lipschitz/growth conditions, the SDE \eqref{equation: forward-time SDE} is well-posed (see, e.g., \cite{le2016brownian,oksendal2013sde}). 
We write $\boldsymbol{p}(t,\cdot)$ for the law of $\textbf{x}_t$  (so that $\boldsymbol{p}(0,\cdot) = p_{\rm data}$), $\boldsymbol{p}(t,\bold{x}_t \mid \bold{x}_s)$ for the forward transition density from time $s$ to $t$, and $\boldsymbol{p}(t,\bold{x}_t \mid \mathrm{x})$ for the conditional density given the initialization $\bold{x}_0=\mathrm{x}$.
%(Here, $\boldsymbol{p}(t, \bold{x}_t) = \int p(s,\bold{x}_s) \boldsymbol{p}(|s, \textbf{x}_t|\textbf{x}_s) \mathrm{d} \bold{x}_s$, $\boldsymbol{p}(t, \textbf{x}_t|\textbf{x}_s)$ describes the probability density of the state $\textbf{x}_t$ at time $t$ given that at some earlier time $s$ the state was	$\textbf{x}_s$.})

To generate data in $\boldsymbol{p}(0,\cdot) = p_{\rm data}$, we can start with $\boldsymbol{p}(T,\cdot)$ and run the process $\bold{x}$ backward for time $T$. More precisely, consider the time reversal $\bar{\bold{x}}_t = \bold{x}_{T-t}$ for $t \in [0,T]$. The reverse-time SDE corresponding to \eqref{equation: forward-time SDE} reads 
\begin{equation}
\mathrm{d} \bar{\bold{x}}_t \! = \! \left[- \bold{f}(T-t, \bar{\bold{x}}_t) \! + \!  \boldsymbol{g}^2(T-t) \nabla \ln \boldsymbol{p}(T-t, \bar{\bold{x}}_t)\right] \mathrm{d} t \! +\! \boldsymbol{g}(T-t) \mathrm{d} {\bold{B}}_t , \ \bar{\bold{x}}_0 = \bar{\mathrm{x}} \sim \boldsymbol{p}(T, \cdot),
\label{equation: reverse-time SDE}
\end{equation} 
where $\nabla \ln \boldsymbol{p}$ denotes the gradient of $\ln \boldsymbol{p}$ with respect to $\bold{x}$.  
%%%%%%%%%%%%%%%%%%%%%%%%%%%%%%%%%%%%%%%%%%%%%%%%%%%%%%%%%%%%%%%%%%%%%%
It has been shown in \cite{anderson1982reverse, tang2025score} that the reverse-time SDE \eqref{equation: reverse-time SDE} has the same marginal distributions as the forward SDE \eqref{equation: forward-time SDE} under suitable conditions on $\bold{f}, \boldsymbol{g}$ and $\boldsymbol{p}(t,\cdot)$. 
%%%%%%%%%%%%%%%%%%%%%%%%%%%%%%%%%%%%%%%%%%%%%%%%%%%%%%%%%%%%%%%%%%%%%%%%%%%%%%%%%%%%%%%%%%%%%%%%%%%%%%%

\begin{example}[Ornstein–Uhlenbeck–type (OU-type) forward process]
%OU-type 过程是目前许多常用的生成模型的一类SDE建模
A commonly used forward SDE is the OU-type process
\begin{equation}
	\mathrm{d}\bold{x}_t = \boldsymbol{f}(t)\,\bold{x}_t\,\mathrm{d}t + \boldsymbol{g}(t)\,\mathrm{d}\bold{B}_t,\quad \bold{x}_0 = \mathrm{x} \sim p_{\rm data}.
	\label{equation: OU-process}
\end{equation}
A number of widely used forward noising processes in generative modeling fall within this OU-type formulation.
For instance, the variance-preserving (VP) SDE underlying DDPM \cite{ho2020denoising}, the variance-exploding (VE) SDE used in score matching with Langevin dynamics (SMLD) \cite{song2019generative}, and the contractive diffusion probabilistic models (CDPM) \cite{tang2024contractive} are all special cases obtained by suitable choices of the coefficients $\boldsymbol{f}$ and $\boldsymbol{g}$.
\end{example}

\subsection{Score-based diffusion models}
In generative AI applications, the data distribution \(p_{\rm data}\) and the associated densities \(\boldsymbol{p}(t,\cdot)\) are analytically intractable due to the high dimensionality and structural complexity of real-world data. Consequently, the score function $\nabla \ln \boldsymbol{p}(t,\cdot)$ appearing in the reverse SDE \eqref{equation: reverse-time SDE} cannot be evaluated in closed form. 
In practice this score is approximated by a parametrized neural network $\bold{s}_{\Theta}$, and the learned score is substituted into the reverse dynamics to obtain the \emph{generative} SDE
\begin{equation}  
\mathrm{d} \hat{\bold{x}}_t  =\left[- \bold{f}(T-t, \hat{\bold{x}}_t)+ \boldsymbol{g}^2(T-t) \bold{s}_{\Theta}(T-t, \hat{\bold{x}}_t)\right] \mathrm{d} t + \boldsymbol{g}(T-t) \mathrm{d} {\bold{B}}_t.  
\label{equation: reverse-time generative SDE}  
\end{equation}  

A canonical objective to fit $\boldsymbol{s}_\Theta$ is the \emph{explicit score matching} (ESM) risk
\begin{equation}
\begin{aligned}
	\ell^{\rm ESM}(\Theta) %& = \mathbb{E}_{t \sim \mathcal{U}(0,T)} \left\lbrace \pmb{\lambda}(t) \mathbb{E}_{\bold{x}_t \sim \boldsymbol{p}(t, \cdot)} \left[  \|\bold{s}_{\Theta}(t, \bold{x}_t) - \nabla \ln \boldsymbol{p}(t, \bold{x}_t)\|^2_2 \right] \right\rbrace     \\
	& = \frac{1}{T}\int_{0}^{T} \pmb{\lambda}(t) \mathbb{E}_{\bold{x}_t \sim \boldsymbol{p}(t, \cdot)}  \big[\left\| \bold{s}_{\Theta}(t, \bold{x}_t) - \nabla \ln \boldsymbol{p}(t, \bold{x}_t) \right\|_2^2\big] \mathrm{d}t, 
	%	& \sout{= \frac{1}{T}\int_{0}^{T} \int_{\bold{x} \in \mathbb{R}^d} \pmb{\lambda}(t) \left\| \bold{s}_{\Theta}(t, \bold{x}) - \nabla \ln \boldsymbol{p}(t, \bold{x}) \right\|_2^2 \boldsymbol{p}(t, \bold{x}) \mathrm{d} \bold{x} \mathrm{d}t},
\end{aligned}
\label{loss: ESM}
\end{equation}
where $\pmb{\lambda}(t) \ge 0$ is a weighting function. %(\ref{loss: ESM}) is also known as the explicit score matching (ESM) objective.
Because the score $\nabla \ln \boldsymbol{p}(t,\bold{x}_t)$ cannot be evaluated in closed form, the explicit score-matching objective \eqref{loss: ESM} is replaced by the \emph{denoising score matching} (DSM) objective
\begin{equation}
\ell^{\rm DSM}(\Theta)= %\mathbb{E}_{t \sim \mathcal{U}(0,T)} \left\lbrace  \pmb{\lambda}(t) \mathbb{E}_{\mathrm{x} \sim p_{\rm data}} \! \left[\mathbb{E}_{\bold{x}_t \sim \boldsymbol{p}(t, \cdot|\mathrm{x} )}\!\left[\left\|\bold{s}_{\Theta}(t, {\bold{x}_t}) -\nabla \ln \boldsymbol{p}(t, \bold{x}_t|\mathrm{x})\right\|_2^2\right]\right] \right\rbrace , \\
\frac{1}{T}\int_{0}^{T}   \pmb{\lambda}(t) \mathbb{E}_{\mathrm{x} \sim p_{\rm data}} \! \left[\mathbb{E}_{\bold{x}_t \sim \boldsymbol{p}(t, \cdot|\mathrm{x} )}\!\left[\left\|\bold{s}_{\Theta}(t, {\bold{x}_t}) -\nabla \ln \boldsymbol{p}(t, \bold{x}_t|\mathrm{x})\right\|_2^2\right]\right]  \mathrm{d}t,
\label{loss: continuous-infinite-score-loss-function}
\end{equation}
where the gradient $\nabla \ln \boldsymbol{p}(t, \bold{x}_t|\mathrm{x})$ is with respect to $\bold{x}_t$.
This replacement follows from the denoising score-matching identity established in \cite{vincent2011connection} and forms the basis of diffusion-model training \cite{song2019generative, song2021scorebased}.
It is proved that (\ref{loss: continuous-infinite-score-loss-function}) is equivalent to (\ref{loss: ESM}) up to a constant independent of $\Theta$ \cite{vincent2011connection}. 
%%%%%%%%%%%%%%%%%%%%%%%%%%%%%%%%%%%%%%%%%%%%%
%%%%%%%%%%%%%%%%%%%%%%%%%%%%%%%%%%%%%%%%%%%%%
%Considering $\Omega_{{\Theta}}$ as a subset of some Euclidean space (given later), the learning task for diffusion models under the denoising score-matching loss in (\ref{loss: continuous-infinite-score-loss-function}) can be written as the following optimal control problem

For the OU-type process \eqref{equation: OU-process}, the condition probability $\boldsymbol{p}(t, \bold{x}_t|\mathrm{x})$ has a closed-form expression
$$
\boldsymbol{p}(t, \bold{x}_t|\mathrm{x})=\mathcal{N}\left(\bold{x}_t ; \boldsymbol{r}(t) \mathrm{x}, \boldsymbol{r}^2(t) \boldsymbol{v}^2(t) \mathrm{I}_d\right),
$$
where $\boldsymbol{r}(t)=\exp\!\big(\int_0^t \boldsymbol{f}(s)\,\mathrm{d}s\big)$ and
$\boldsymbol{v}(t)=\big(\int_0^t \boldsymbol{g}^2(s)/\boldsymbol{r}^2(s)\,\mathrm{d}s\big)^{1/2}$.
Hence the conditional score $\nabla \ln \boldsymbol{p}(t, \bold{x}_t|\mathrm{x})$ 
is available in closed form and can be used directly in network training (see Appendix Lemma~\ref{lemma: OU process}).
At the population level, the training problem then takes the form of the following continuous-time optimization problem
\begin{equation}
\begin{aligned}
	(\mathcal{P}): \	\left\{ \begin{array}{l}
		\inf \limits_{{\Theta} \in \Omega_{{\Theta}} } 	\ell^{\rm DSM}(\Theta)  \\
		\text{ subject to:} \
		\bold{s}_{\Theta} \text{ is realized by some DNN architecture with parameter } \Theta,
	\end{array} \right.\\
\end{aligned}
\label{controlP: true score learning problem}
\end{equation}
where $\Omega_{{\Theta}}$ is the feasible parameter set.

%%%%%%%%%%%%%%%%%%%%%%%%%%%%%%%%%%%%%%%%%%%%%%%%%%%
%%%%%%%%%%%%%%%%%%%%%%%%%%%%%%%%%%%%%%%%%%%%%%%%%%%
In practice, both the time variable and the data are discretized. Let $0=t_N^0< t_N^1<\dots<t_N^{N}=T$ be a time grid with step size $h_N$, and let $\mathcal{S} = \{\mathrm{x}^{(s)}\}_{s=1}^S$ be $S$ i.i.d.\ samples from $p_{\rm data}$.  For each training sample $\mathrm{x}^{(s)}$, one can simulate the forward SDE (\ref{equation: forward-time SDE}) to obtain noisy observations at grid times; this leads to the empirical, discrete-time DSM loss
\begin{equation}
\ell^{\rm DSM}_{N,S}(\Theta)\!  = \!  \frac{h_N}{TS} \sum_{k=0}^{N \!-\!1}\! \pmb{\lambda}(t_{N}^k) \!\sum_{s=1}^{S} \!\mathbb{E}_{\bold{x}_{t_{N}^k} \! \sim   \boldsymbol{p}\big(t_{N}^i, \cdot|\mathrm{x}^{(s)}\!\big)} \! \left[\big\|\bold{s}_{\Theta}(t_{N}^k, {\bold{x}_{t_{N}^k}}) \!- \! \nabla \ln \boldsymbol{p}(t_{N}^k \!, \bold{x}_{t_{N}^k}|\mathrm{x}^{(s)})\big\|_2^2\right].
\label{loss: discrete-time-finite-sample loss}
\end{equation}
Note that \(\bold{x}_{t_{N}^k}\) is the exact forward-SDE solution starting from \(\bold{x}_0 = \mathrm{x}^{(s)}\). In practical generative-model training, the coefficients \(\bold{f}\) and \(\boldsymbol{g}\) are selected so that the forward process admits an analytic Gaussian transition; consequently, \(\bold{x}_{t_{N}^k}\) can be sampled directly from its closed-form distribution rather than computed via numerical discretization.
The empirical optimization problem solved in implementations is therefore
\begin{equation}
\begin{aligned}
	(\mathcal{P}_{N,S}): \	\left\{ \begin{array}{l}
		\inf \limits_{{\Theta} \in \Omega_{\Theta}} 
		\ell^{\rm DSM}_{N,S}(\Theta) 
		\\
		\text{ subject to:} \
		\bold{s}_{\Theta}(\cdot, \cdot) \text{ is realized by some DNN architecture with parameter } \Theta.
	\end{array} \right.\\
\end{aligned}
\label{controlP: discrete-time-finite-sample score learning problem} 
\end{equation}

%%%%%%%%%%%%%%%%%%%%%%%%%%%%%%%%%%%%%%%%%%%%%%%
%%%%%%%%%%%%%%%%%%%%%%%%%%%%%%%%%%%%%%%%%%%%%%%

The forward coefficients $\bold{f},\boldsymbol{g}$ and the horizon $T$ are pre-specified so that the law $\boldsymbol{p}(T,\cdot)$ is approximately matched to a known prior $\pi$.  This design allows the reverse dynamics to be initialized from $\hat{\bold{x}}_0 = \mathrm{\bold{x}}\sim\pi$ in practice. The learned reverse SDE is implemented by a time discretization with the grid $0=t_N^0<\dots<t_N^N=T$ and step size $h_N$, a standard Euler–Maruyama sampler reads
\begin{equation}
\begin{aligned}
	\check{\bold{x}}_{t_N^{k+1}}  \! = \! \check{\bold{x}}_{t_N^{k}} \! + \! h_N \! \left[-\bold{f}(T \! - \! t_N^k, \check{\bold{x}}_{t_N^k}) \! + \! \boldsymbol{g}^2(T \! - \! t_N^k) \bold{s}_{\Theta}(T \!- \! t_N^k, \check{\bold{x}}_{t_N^k})\right] \! + \! \int_{t_N^{k}}^{t_N^{k+1}} \! \boldsymbol{g}(T-t) \mathrm{d} {\bold{B}}_t,
\end{aligned}
\label{equation: Euler-Maruyama scheme}
\end{equation}
where $k=0,1,\ldots,N-1$. In practical sampling, one draws $\check{\bold{x}}_0 = \check{\mathrm{x}}\sim\pi$ and iterates the discrete update \eqref{equation: Euler-Maruyama scheme} to obtain $\check{\bold{x}}_T$; the law of $\check{\bold{x}}_T$ is then used as an approximation to $p_{\rm data}$.

%%%%%%%%%%%%%%%%%%%%%%%%%%%%%%%%%%%%%%%%%%%%%%%
%%%%%%%%%%%%%%%%%%%%%%%%%%%%%%%%%%%%%%%%%%%%%%%%%

In this work we study the convergence and generalization properties of $(\mathcal{P}_{N,S})$ and $(\mathcal{P})$,
and their implications for the generative reverse process. The preceding formulations show that the sampling error decomposes naturally into three contributions: (i) the truncation error arising from approximating the reverse-time initial distribution; %在学习方面，我们研究通过在离散时间DSM目标上进行经验风险最小化，得分函数在训练中的泛化误差。
(ii) the score-learning error, which itself includes the time-discretization error, the finite-sample generalization error, and the optimization error associated with solving $(\mathcal{P}_{N,S})$; and (iii) the discretization error introduced when simulating the learned reverse process. In the next section, we establish a unified framework that quantifies these contributions and show how the learning and sampling stages jointly determine the overall generative error.

%%%%%%%%%%%%%%%%%%%%%%%%%Section%%%%%%%%%%%%%%%%%%%%%%%%%%%%%%%%%%%%%%%%%%%%%%%%%%%%%%%%%%%%%%%%%%%%%%%%%%%%

\section{Main results: Generalization and convergence  properties of SDMs}
\label{Sec: 3}
This section presents our main theoretical results on the learning and sampling errors of SDMs. 
To develop a unified framework quantifying how the learning stage influences the performance of the sampling stage, we organize the analysis along two complementary components.  
(i) On the \emph{learning side}, we study how well the score function can be generalized through empirical risk minimization on the discrete-time DSM objective.  
(ii) On the \emph{sampling side}, we quantify the error induced when simulating the learned reverse dynamics numerically.  
The combination of these two components yields the end-to-end generative error bounds presented later.

\subsection{Assumptions and preliminary results}
We begin by specifying the parameterization of the learnable score function and by introducing the assumptions required for our theoretical analysis.

\textbf{Neural network parameterization.}
The learnable score function $\bold{s}_{\Theta}$ is modeled as a time–dependent residual neural network.  
Given an input pair $(t,\mathrm{x}) \in [0,T] \times \mathbb{R}^{n_\mathrm{d}}$, the network computes a sequence of hidden states $\{\bold{z}^l\}_{l=0}^L$ via
\begin{equation} 
\begin{aligned}
	\bold{z}^{l} &= \bold{z}^{l-1} + 
	\mathrm{W}^{l}\psi \circ (\mathrm{V}^{l}\bold{z}^{l-1} + \mathrm{U}^{l} \boldsymbol{e}(t)), \ 1\leq l \leq L . \\
	\bold{z}^{0} &=  \mathrm{x}, %\in \mathbb{R}^{n}
\end{aligned}
\label{equation: ResNets}
\end{equation} 
where $\boldsymbol{e}:[0,T] \rightarrow \mathbb{R}^{n_{\boldsymbol{e}}}$ is a time-embedding map, and $\psi: \mathcal{R} \to \mathcal{R}$ is an elementwise activation function.  
For each layer $l$, the learnable parameters are collected as
$$(\mathrm{U}^{l}, \mathrm{V}^{l}, \mathrm{W}^{l}) =: {\Theta}^{l}  \in \mathcal{E}:= \mathbb{R}^{n \times n_{\boldsymbol{e}}} \times \mathbb{R}^{n \times n_{\rm d}} \times \mathbb{R}^{n_{\rm d} \times n}.$$ 
To emphasize the dependence of the network with input $(t,\mathrm{x})$, we rewrite the state $\bold{z}^{l}$ as $\bold{z}^{l}(t,\mathrm{x})$. Therefore, the score network $\bold{s}_{\Theta}$ is defined as the final output of the residual recursion:
$$
\bold{s}_{\Theta}(t, \mathrm{x}) : = \bold{z}^{L}(t,\mathrm{x}).
$$

Let $L\in \mathcal{N}$ be the total number of layers and 
${\Theta} = \left({\Theta}^{1}, {\Theta}^{2}, \ldots, {\Theta}^{L} \right) \in \mathcal{E}^{L}$
collect the parameters of the entire network.  We introduce the norms
\begin{equation}
\|{\Theta}^{l}\|_{\infty} = \max\{ \|\mathrm{U}^{l}\|_{2}, \|\mathrm{V}^{l}\|_{2}, \|\mathrm{W}^{l}\|_{2} \}, \quad \|{\Theta}\|_{\mathcal{E}^{L}} := \max_{1\leq l\leq L}  \|{\Theta}^{l}\|_{\infty}.
\end{equation}

\textbf{Assumptions}. The following assumptions provide the foundation for establishing uniform generalization bounds and rigorous convergence guarantees.
\begin{itemize}
\item[(A1)] $(\Omega, \mathcal{F}, P) $ is a probability space and the data distribution $p_{\rm data}$ is supported on $B_R(0):=\{ \mathrm{x} \in \mathbb{R}^{n_\mathrm{d}}: \ \|\mathrm{x}\|_2 \leq R \}$.

\item[(A2)]  The drift function $\bold{f} \in \mathcal{C}^1([0,T]\times\mathbb{R}^{n_\mathrm{d}})$. Moreover, there exists a constant $C_{\bold{f}}>0$ and a continuous function $\pmb{\alpha} \in \mathcal{C}([0,T])$ such that $\forall \mathrm{x}, \tilde{\mathrm{x}} \in \mathbb{R}^{n_\mathrm{d}}$ and  $t, \tilde{t} \in[0, T]$, 
$$
\begin{aligned}
	&\|\bold{f}(t, \mathrm{x})\|_2 \leq  C_{\bold{f}}\left(1+\|\mathrm{x}\|_2\right), \\ 
	& \|\bold{f}(t, \mathrm{x}) - \bold{f}(t, \tilde{\mathrm{x}})\|_2 \leq  C_{\bold{f}} \| \mathrm{x} - \tilde{\mathrm{x}} \|_2, \\
	&\|\bold{f}(t, \mathrm{x}) - \bold{f}(\tilde{t}, {\mathrm{x}})\|_2 \leq  C_{\bold{f}} \| \mathrm{x}\|_2 \cdot \omega_{\pmb{\alpha}}(|t-\tilde{t}|).
\end{aligned}
$$

\item[(A3)] The function $\pmb{\lambda}, \boldsymbol{g} \in \mathcal{C}([0,T])$ and satisfy $|\pmb{\lambda}(t)|>0$ and $|\boldsymbol{g}(t)|>0$ for all $t \in[0, T]$.

\item[(A4)] For every $\mathrm{x}, \tilde{\mathrm{x}}, \mathrm{y} \in \mathbb{R}^{n_\mathrm{d}}$ the function $\nabla \ln \boldsymbol{p}(\cdot, \mathrm{x}|\mathrm{y}) \in \mathcal{C}([0,T])$ and there exists a constant $C_{\boldsymbol{p}}>0$ and a continuous function $\pmb{\gamma} \in \mathcal{C}([0,T])$ such that for all $ t, \tilde{t} \in[0, T]$,
\begin{equation}
	\begin{aligned}
		\left\|\nabla \ln \boldsymbol{p}(t, \mathrm{x})\right\|_2  \leq & C_{\boldsymbol{p}} \left(1+\|\mathrm{x}\|_2\right); \\
		\left\|\nabla \ln \boldsymbol{p}(t, \mathrm{x})-\nabla \ln \boldsymbol{p}(t, \mathrm{y})\right\|_2 \leq & C_{\boldsymbol{p}} \|\mathrm{x} - {\mathrm{y}} \|_2,
	\end{aligned}
	\label{assumption: assumption of p}
\end{equation}
and
\begin{equation}
	\begin{aligned}
		\| \nabla \ln \boldsymbol{p}\big(t, \mathrm{x}|\mathrm{y}\big) \|_2 \leq & C_{\boldsymbol{p}} (1+ \|\mathrm{x}\|_2 + \|\mathrm{y}\|_2);\\
		\| \nabla \ln \boldsymbol{p}\big(t, \mathrm{x}|\mathrm{y}\big) -  \nabla \ln \boldsymbol{p}\big(t, \tilde{\mathrm{x}}|\tilde{\mathrm{y}} \big) \|_2 \leq & C_{\boldsymbol{p}} (\| \mathrm{x} - \tilde{\mathrm{x}} \|_2 + \| \mathrm{y} - \tilde{\mathrm{y}} \|_2); \\
		\| \nabla \ln \boldsymbol{p}\big(t, \mathrm{x}|\mathrm{y}\big) -  \nabla \ln \boldsymbol{p}\big(\tilde{t}, {\mathrm{x}}|\mathrm{y}\big) \|_2 \leq & C_{\boldsymbol{p}} (1+ \|\mathrm{x}\|_2 + \|\mathrm{y}\|_2) \cdot \omega_{\pmb{\gamma}}(|t-\tilde{t}|),
	\end{aligned}
	\label{assumption: assumption of condition p}
\end{equation}
where $\omega_{\pmb{\gamma}}$ is the modulus of continuity of $\pmb{\gamma}$.

\item[(A5)]  The parameter set is restricted to $\Omega_{{\Theta}} = \{\Theta \in \mathcal{E}: \ \|{\Theta}\|_{\mathcal{E}^{L}} \leq B_{\Theta}\}$, where $B_{\Theta}>0$ is a given constant.

\item[(A6)] The time embedding function $\boldsymbol{e}$ is $\mathrm{Lip}_{\boldsymbol{e}}$-Lipschitz continuous. The activation function $\psi$ is non-decreasing, $\mathrm{Lip}_{\psi}$-Lipschitz continuous, and satisfies $\psi(0)=0$.  
\end{itemize}
\begin{remark}
\label{remark: proeprty of assumptions}
The bounded-support condition on $p_{\rm data}$ in (A1) is standard in theoretical analyses of diffusion models and generalization theory; see, for example, \cite{de2022convergence, li2023generalization, lee2023convergence}.  
The Lipschitz and growth conditions in (A2) and (\ref{assumption: assumption of p}) of (A4) ensure the existence and uniqueness of strong solutions to both the forward- and reverse-time SDEs.  
When the forward SDE (\ref{equation: forward-time SDE}) is an OU-type process, the assumption (A2) are automatically satisfied. 
The conditions~\eqref{assumption: assumption of condition p} can be verified for
a wide range of initialization $p_{\rm data}$ (see
Lemma~\ref{lemma: OU process} in the Appendix for a detailed discussion).  
In contrast, the conditions~\eqref{assumption: assumption of p} depend explicitly
on both the regularity of the initial distribution $p_{\rm data}$ and the
coefficients of the forward SDE.  
A sufficient case is when $p_{\rm data}$ is Gaussian.  Under some OU-type forward dynamics \cite{paton2023investigating}, such process remain Gaussian and the conditions in \eqref{assumption: assumption of p} are satisfied.
Under (A3) and (A6), we introduce the shorthand notations
$$
b_{\boldsymbol{g}} := \min_{ t \in [0, T]} |\boldsymbol{g}(t)|, \quad  
B_{\boldsymbol{g}} := \max_{ t \in [0, T]} |\boldsymbol{g}(t)|, \quad  
B_{\pmb{\lambda}} := \max_{ t \in [0, T]} |\pmb{\lambda}(t)|, \quad  
B_{\boldsymbol{e}} := \max_{ t \in [0, T]} |\boldsymbol{e}(t)|.
$$  
For convenience, we additionally assume $B_{\boldsymbol{g}} \le C_{\bold{f}}$.  
\end{remark}

We next collect some preliminary results. 
The following lemma summarizes key properties of the network $\bold{s}_{\Theta}$.
\begin{lemma}[Properties of learnable score function $\bold{s}_{\Theta}$]
\label{lemma: property of DNN}

Assume that assumptions $(A5)(A6)$ hold.  
For any learnable parameters $\Theta, \tilde{\Theta} \in \Omega_{\Theta}$, inputs of neural network $(t, \mathrm{x}), (\tilde{t}, \tilde{\mathrm{x}}) \in  [0,T] \times \mathbb{R}^{n_{\mathrm{d}}}$,
the learnable score functions $\bold{s}_{\Theta}, \bold{s}_{\tilde{\Theta}}$ defined by (\ref{equation: ResNets}) satisfy
\begin{equation}
	\begin{aligned}
		&\|\bold{s}_{\Theta}(t, \mathrm{x})\|_2  \leq B_{\bold{s}} (\mathrm{x}), \\
		& 	\|\bold{s}_{\Theta}(t, \mathrm{x}) - \bold{s}_{\tilde{\Theta}}(t, \mathrm{x}) \|_2  \leq   C_{\bold{s}}(\mathrm{x}) \| \Theta - \tilde{\Theta} \|_{\mathcal{E}^{L}}, \\
		%%%%%%%%%%%%%%%%%%%%%%%%%%%%%%%%%%%%%%%%%%%%%%%%%%%%%%%%
		&\|\bold{s}_{\Theta}(t, \mathrm{x}) - \bold{s}_{{\Theta}}(\tilde{t}, \tilde{\mathrm{x}}) \|_2  \leq \Big( \|\mathrm{x} \!-\! \tilde{\mathrm{x}}\|_2  + L \mathrm{Lip}_{\psi} \mathrm{Lip}_{\boldsymbol{e}}  B_{\Theta}^2 |t - \tilde{t}|  \Big)  \exp(L \mathrm{Lip}_{\psi}  B_{\Theta}^2),
	\end{aligned}
\end{equation}
where $B_{\bold{s}} (\mathrm{x}) =  \big( \|\mathrm{x}\|_2 + L \mathrm{Lip}_{\psi} B_{\boldsymbol{e}}  B_{\Theta}^2  \big)  \exp(L \mathrm{Lip}_{\psi}  B_{\Theta}^2)$,  $
C_{\bold{s}}(\mathrm{x}) =  2L\mathrm{Lip}_{\psi}  B_{\Theta} (B_{\bold{s}} (\mathrm{x}) + B_{\boldsymbol{e}}) \cdot $  $ \exp(L \mathrm{Lip}_{\psi}  B_{\Theta}^2) $.
%	$$
%	\begin{aligned}
	%		B_{\bold{s}} (\mathrm{x}) = & \big( \|\mathrm{x}\|_2 + L \mathrm{Lip}_{\psi} B_{\boldsymbol{e}}  B_{\Theta}^2  \big)  \exp(L \mathrm{Lip}_{\psi}  B_{\Theta}^2), \\
	%		C_{\bold{s}}(\mathrm{x}) = & 2L\mathrm{Lip}_{\psi}  B_{\Theta} (B_{\bold{s}} (\mathrm{x}) + B_{\boldsymbol{e}})  \exp(L \mathrm{Lip}_{\psi}  B_{\Theta}^2) .
	%	\end{aligned}
%	$$
\end{lemma}
%%%%%%%%%%%%%%%%%%%%%%%%%%%%%%%%%%%%%%%%%%%%%%%%%%%%%%%%%%%%%%%%%
%%%%%%%%%%%%%%%%%%%%%%%%%%%Next Lemma%%%%%%%%%%%%%%%%%%%%%%%%%%%%%%

\begin{lemma}[Bound estimates of SDEs]
\label{lemma: property of SDE-1}
Assume that assumptions (A1)-(A3) and (A5)(A6) hold. Let ${\bold{x}}_{t}$ be the state variables of SDE (\ref{equation: forward-time SDE}) with initialization ${\bold{x}}_{0} = \mathrm{x} \sim p_{\rm data}$ and let $\hat{\bold{x}}_{t}$ be the state variables of SDE (\ref{equation: reverse-time generative SDE})  with initialization $\hat{\bold{x}}_{0} = \hat{\mathrm{x}}  \sim \boldsymbol{p}(T, \cdot)$ , $ t, \tilde{t} \in [0,T]$.
Then there exist constants 
\begin{equation}
	\begin{aligned}
		B_{\bold{x}, 2} = (1+R^2) e^{12C_{\bold{f}}^2T}; & \ \mathrm{Lip}_{\bold{x}, 2}= 8 C_{\bold{f}}^2 (T + 1) \big[ 1+ \big(1+R^2 \big) e^{12C_{\bold{f}}^2T} \big]   e^{12C_{\bold{f}}^2 T}, \\
		B_{\hat{\bold{x}}, 2} = (1+B_{\bold{x}, 2}) e^{12C_{\hat{\bold{f}}}^2T}; & \ \mathrm{Lip}_{\hat{\bold{x}}, 2}= 8 C_{\hat{\bold{f}}}^2 (T + 1) \big[ 1+ \big(1+ B_{\bold{x}, 2}\big) e^{12C_{\hat{\bold{f}}}^2T} \big]   e^{12C_{\hat{\bold{f}}}^2 T},
	\end{aligned}
	\label{coefficients: B_x2&Lip_x2}
\end{equation}
with $C_{\hat{\bold{f}}} = C_{\bold{f}} +B_{\boldsymbol{g}} ^2 \exp(L \mathrm{Lip}_{\psi}  B_{\Theta}^2) \cdot \max \{1, L \mathrm{Lip}_{\psi} B_{\boldsymbol{e}}  B_{\Theta}^2 \}$ such that, for any $t,\tilde{t} \in [0,T]$,
\begin{equation}
	\begin{aligned}
		\mathbb{E} \left[\big\|{\bold{x}}_{t}\big\|_2^2\right] \leq  & B_{\bold{x}, 2}, \
		%%%%%%%%%%%%%%%%
		\mathbb{E} \left[\big\|{\bold{x}}_{t} -{\bold{x}}_{\tilde{t}} \big\|_2^2\right] \leq  % 4 C_{\bold{f}}^2 \big[(t-\bar{t}) + 1\big] \big( 1+ \mathbb{E}(\|{\bold{x}}_{\bar{t}}\|^2)\big) (t-\bar{t})  e^{6C_{\bold{f}}^2 (t-\bar{t})} \\
		%	\leq & 4 C_{\bold{f}}^2 (T + 1) \big( 1+ \big(1+\mathbb{E}[\|\mathrm{x}_0\|^2_2] \big) e^{6C_{\bold{f}}^2T} )\big)   e^{6C_{\bold{f}}^2 T} (t-\bar{t})^{n/2} \\
		\mathrm{Lip}_{\bold{x}, 2}|t-\tilde{t}|, \\
		%%%%%%%%%%%%%%%%%%%%%%%%%%%%%%%%%%%%%%%%%%
		%%%%%%%%%%%%%%%%%%%%%%%%%%%%%%%%%%%%%%%%%%%
		\mathbb{E} \left[\big\|\hat{\bold{x}}_{t}\big\|_2^2\right] \leq  & B_{\hat{\bold{x}}, 2}, \ 
		%%%%%%%%%%%%
		\mathbb{E} \left[\big\|\hat{\bold{x}}_{t} -\hat{\bold{x}}_{\tilde{t}} \big\|_2^2\right] \leq \mathrm{Lip}_{\hat{\bold{x}}, 2}|t-\tilde{t}|.
	\end{aligned}
	\label{proof inequality: properties of SDE}
\end{equation}
\end{lemma}

%%%%%%%%%%%%%%%%%%%%%%%%%%%%%%%%%%%%%%%%%%%%%%%%%%%%%%%%%%%%%%%%%
%%%%%%%%%%%%%%%%%%%%%%%%%%%Next Lemma%%%%%%%%%%%%%%%%%%%%%%%%%%%%%%

Building on Lemma \ref{lemma: property of DNN}, the next lemma establishes moment bounds for the random quantities $B_{\bold{s}}^2(\bold{x}_t)$ and $C_{\bold{s}}^2(\bold{x}_t)$ associated with the forward SDE states $\bold{x}_t$. These bounds serve as key ingredients in the subsequent error analysis.
\begin{lemma}
\label{lemma: proeprty of dnn state}
Assume that assumptions $(A1)-(A6)$ hold and  
$\bold{x}_t$ is the solution of (\ref{equation: forward-time SDE}) initialized with ${\bold{x}}_{0}  = \mathrm{x}\sim p_{\rm data}$, then the following bounds hold for all $t \in [0,T]$,
\begin{equation}
	\begin{aligned}
		\mathbb{E}[(B_{\bold{s}}(\bold{x}_t))^2] 
		\leq & 2 \big[ \big( 1 \! + \!  R^2  \big) \exp(12C_{\bold{f}}^2 T)\!  + \! L^2 \mathrm{Lip}_{\psi}^2 B_{\boldsymbol{e}}^2  B_{\Theta}^{4}  \big] \!  \exp(2L \mathrm{Lip}_{\psi}  B_{\Theta}^2)\!  =:\!  B_{\bold{s},2}; \\
		%%%%%%%%%%%%%%%%%%%%%%%%%%%%%%%%%%%%%%%%%%%%%%%%%%%
		\mathbb{E}[(C_{\bold{s}}(\bold{x}_t))^2] 
		\leq & 8L^2\mathrm{Lip}_{\psi}^2 \big(B_{\bold{s},2} + B_{\boldsymbol{e}}^2 \big)  B_{\Theta}^2 \exp(2L \mathrm{Lip}_{\psi}  B_{\Theta}^2)
		=: C_{\bold{s},2},
	\end{aligned}
	\label{proof inequality: estimation of expection of states}
\end{equation}
and for all $\Theta \in \Omega_{\Theta}$,
\begin{equation}
	\begin{aligned}
		\mathbb{E} \left[\big\|\bold{s}_{\Theta}(t, \bold{x}_{t}) \!- \! \nabla \ln \boldsymbol{p}(t, \bold{x}_{t}|\mathrm{x})\big\|_2^2\right] 
		\! \leq \! 2 B_{\bold{s},2}\! + \!6C_{\boldsymbol{p}} \Big(1\!+\!(1\!+ \! R^2) e^{12C_{\bold{f}}^2T} \! +\! R^2 \Big)\! =: \! B_{\rm score\_err}.	
	\end{aligned}
	\label{proof inequality: estimation of difference of scores}
\end{equation}
\end{lemma}

\begin{remark}
The bounds in Lemma~\ref{lemma: property of DNN}-Lemma~\ref{lemma: proeprty of dnn state} are proved in Appendix~\ref{appendix: proof of lemmas}.
The constant $B_{\rm score\_err}$ in \eqref{proof inequality: estimation of difference of scores} should be regarded as a coarse a priori bound on the score error for an arbitrary parameter $\Theta \in \Omega_{{\Theta}}$.
It is not intended to reflect the accuracy achieved after training and is typically much looser than the empirical score error.
This distinction is important, as several works (e.g., \cite{lee2022convergence, lee2023convergence, chen2023sampling}) directly assume an explicit upper bound on the score error when establishing sampling-error guarantees.
The priori bound also ensures the training loss is uniformly bounded over~$\Omega_{\Theta}$, providing a stability needed for subsequent generalization analysis.
\end{remark}

%%%%%%%%%%%%%%%%%%%%%%%%%%%%%%%%%%%%%%%%%%%%%%%%%%%%%%%%%%%%%%%
%%%%%%%%%%%%%%%%%%%%%%%%%%%%%%%%%%%%%%%%%%%%%%%%%%%%%%%%%%%%%%%

\subsection{Convergence and generalization of SDMs: Learning perspective}
\label{Sec: 3.2}
We now turn to the theoretical properties of the learning problems associated with SDMs. For the subsequent analysis, we introduce an auxiliary continuous-time, finite-sample score-learning problem defined as:
\begin{equation}
\begin{aligned}
	(\mathcal{P}_{S}): \	\left\{ \begin{array}{l}
		\inf \limits_{{\Theta} \in \Omega_{{\Theta}}} \ell_{S}^{\rm DSM}(\Theta)  \\
		\text{ subject to:} \			\bold{s}_{\Theta}(\cdot, \cdot) \text{ is given in } (\ref{equation: ResNets}).
	\end{array} \right.\\
\end{aligned}
\label{controlP: continuous-time-finite-sample score learning problem}
\end{equation}
Here the loss functional $\ell_{S}^{\rm DSM}$ is defined by
\begin{equation}
\ell_{S}^{\rm DSM}(\Theta) :=  \frac{1}{TS} \int_{0}^{T} \pmb{\lambda}(t)  \sum_{s=1}^{S} \mathbb{E}_{\bold{x}_t \sim \boldsymbol{p}(t, \cdot|\mathrm{x}^{(s)})} \left[\big\|\bold{s}_{\Theta}(t, {\bold{x}}_t) -\nabla \ln \boldsymbol{p}(t, \bold{x}_t|\mathrm{x}^{(s)})\big\|_2^2\right] \mathrm{d} t.
\label{loss: continuous-time-finite-sample loss}
\end{equation}
The empirical loss $\ell_{N,S}^{\rm DSM}$ introduced in (\ref{loss: discrete-time-finite-sample loss}) is precisely the time-discretized counterpart of $\ell_{S}^{\rm DSM}$ obtained on the grid $0=t_N^1<\dots<t_N^N=T$ with step size $h_N$. The well-posedness of the learning problems is established in the following lemma.
\begin{lemma}[Existence of solutions for learning problems]
\label{lemma: Existence of solutions for learning problems}
Suppose assumptions (A1)-(A6) hold. For any given $N,S \in \mathcal{N}$, the minimizers of $(\mathcal{P}_{N,S})$, $(\mathcal{P}_{S})$, $(\mathcal{P})$ exist in $\Omega_{\Theta}$.
\end{lemma}
By Lemma~\ref{lemma: Existence of solutions for learning problems}, we denote $\ell^*_{N,S}$ and $\ell^*$ as the optimal values of problems $(\mathcal{P}_{N,S})$ and $(\mathcal{P})$, respectively, and $\mathcal{O}^*_{N,S}$ and $\mathcal{O}^*$ as their corresponding sets of optimal solutions.
For any $\Theta^*_{N,S} \in \mathcal{O}^*_{N,S}$, we also let 
$$\mathrm{dist}(\Theta^*_{N,S}, \mathcal{O}^*):= \inf_{\Theta \in \mathcal{O}^*} \| \Theta^*_{N,S} - \Theta \|_{\mathcal{E}^L}.$$ 
Noting that $\mathcal{O}^*$ is a compact set, the infimum is reachable.
%%%%%%%%%%%%%%%%%%%%%%%%%%%%%%%%%%%%%%%%%%%%%%%%%%%%%%%%%%%%%%%%
%%%%%%%%%%%%%%%%%%%%%%%%%%%%%%%%%%%%%%%%%%%%%%%%%%%%%%%%%%%%%%%%

The relationship between problems 
$(\mathcal{P}_{N,S})$ and $(\mathcal{P}_{S})$ is characterized by the following lemma.
\begin{lemma}[Discretization error]
Suppose assumptions (A1)-(A6) hold.
Then, for all $\Theta \in \Omega_{\Theta}$,
\begin{equation}
	\begin{aligned}
		\big|\ell_{N,S}^{\rm DSM}(\Theta) - 	\ell_{S}^{\rm DSM}(\Theta) \big| 
		\leq  B_{\rm score\_err}\omega_{\pmb{\lambda}}(h_N) +   B_{\pmb{\lambda}} \sqrt{6B_{\rm score\_err}} \Big[ C_1  \omega_{\pmb{\gamma}}(h_N)  +  C_2 h_N^{1/2} + C_3 h_N \Big],
	\end{aligned}
\end{equation}
where $B_{\rm score\_err}$ is given in (\ref{proof inequality: estimation of difference of scores}), and the constants $C_1, C_2, C_3$ are defined by
\begin{equation}
	\begin{aligned}
		C_1 =& C_{\boldsymbol{p}}   \sqrt{3\big( 1  + B_{\bold{x}, 2} + R^2 \big)}, \ C_2 = \sqrt{2 \mathrm{Lip}_{\bold{x},2} \exp(2L \mathrm{Lip}_{\psi}  B_{\Theta}^2)   + C_{\boldsymbol{p}}^2 \mathrm{Lip}_{\bold{x},2}  } ,  \\
		C_3 = & 2L\mathrm{Lip}_{\psi} \mathrm{Lip}_{\boldsymbol{e}}  B_{\Theta}^2\exp(L \mathrm{Lip}_{\psi}  B_{\Theta}^2). 
	\end{aligned}
	\label{coefficients: C_1_2_3}
\end{equation}
%	with 
%	\begin{equation}
	%		B_{\bold{x}, 2} = (1+R^2) e^{12C_{\bold{f}}^2T}, \ \mathrm{Lip}_{\bold{x}, 2}= 8 C_{\bold{f}}^2 (T + 1) \big[ 1+ \big(1+R^2 \big) e^{12C_{\bold{f}}^2T} \big]   e^{6C_{\bold{f}}^2 T}.
	%		\label{coefficients: B_x2&Lip_x2}
	%	\end{equation}
\label{lemma: Convergence from {P}_{N,S} to {P}_{S}}
\end{lemma}

%%%%%%%%%%%%%%%%%%%%%%%%%%%%%%%%%%%%%%%%%%%%%%%%%%%%%%%%%%%%%%%%%
%%%%%%%%%%%%%%%%%%%%%%%%%%%%%%%%%%%%%%%%%%%%%%%%%%%%%%%%%%%%%%%%%
A central objective in studying SDMs is to train a neural network $\boldsymbol{s}_{\Theta}$ to approximate the unknown score function $\nabla \ln \boldsymbol{p}$. To formalize the learning tasks, we introduce the  population and empirical risks at time $t \in [0,T]$ by
\begin{equation}
\begin{aligned}
	\mathfrak{R}_{p_{\rm data}} (\bold{s}_{\Theta}(t, {\bold{x}}_t)) & = \mathbb{E}_{\mathrm{x} \sim p_{\rm data}} \! \left[\mathbb{E}_{\bold{x}_t \sim \boldsymbol{p}(t, \cdot|\mathrm{x} )}\!\left[\left\|\bold{s}_{\Theta}(t, \bold{x}_t) -\nabla \ln \boldsymbol{p}(t, \bold{x}_t|\mathrm{x})\right\|_2^2\right]\right]
	,\\
	%%%%%%%%%%%%%%%%%%%%%%%%%%%%%%%%%%%%%%%%%%%%%%%%%%%%%%%%%%%
	%%%%%%%%%%%%%%%%%%%%%%%%%%%%%%%%%%%%%%%%%%%%%%%%%%%%%%%%%%%
	\hat{\mathfrak{R}}_{\mathcal{S}} (\bold{s}_{\Theta}(t, {\bold{x}}_t)) & = \frac{1}{S} \sum_{s=1}^{S} \mathbb{E}_{\bold{x}_t \sim \boldsymbol{p}(t, \cdot|\mathrm{x}^{(s)})} \left[\big\|\bold{s}_{\Theta}(t, {\bold{x}}_t) -\nabla \ln \boldsymbol{p}(t, \bold{x}_t|\mathrm{x}^{(s)})\big\|_2^2\right].
\end{aligned}
\label{equation: definition of training loss}
\end{equation}
%%%%%%%%%%%%%%%%%%%%%%%%%%%%%%%%%%%%%%%%%%%%%%%%%%%%%%%%
%%%%%%%%%%%%%%%%%%%%%%%%%%%%%%%%%%%%%%%%%%%%%%%%%%%%%%%%

Using the above notations, we have the following compact form of the loss functions 
$$
\ell^{\rm DSM} (\Theta) =\frac{1}{T} \int_{0}^{T} \pmb{\lambda}(t)  \mathfrak{R}_{p_{\rm data}} (\bold{s}_{\Theta}(t, {\bold{x}}_t))  \mathrm{d} t,
\quad
\ell^{\rm DSM}_{S}(\Theta) = \frac{1}{T} \int_{0}^{T}  \pmb{\lambda}(t) 	\hat{\mathfrak{R}}_{\mathcal{S}} (\bold{s}_{\Theta}(t, {\bold{x}}_t))  \mathrm{d} t.
$$
The following lemma provides bound for the generalization gap $\big| \mathfrak{R}_{p_{\rm data}} (\bold{s}_{\Theta}(t, {\bold{x}}_t)) - \hat{\mathfrak{R}}_{\mathcal{S}} (\bold{s}_{\Theta}(t, {\bold{x}}_t)) \big|$.
\begin{lemma}[Generalization error bound for score functions]
\label{lemma: generalization error of score functions}
Suppose assumptions (A1)-(A6) hold. For all $t\in[0,T]$ and any $\delta\in(0,1)$, with probability at least $1-\delta$, every 
$\bold{s}_{\Theta}$ defined by (\ref{equation: ResNets}) with $\Theta \in \Omega_{\Theta}$ satisfies
\begin{equation}
	\begin{aligned}
		\big| \mathfrak{R}_{p_{\rm data}} (\bold{s}_{\Theta}(t, {\bold{x}}_t)) - \hat{\mathfrak{R}}_{\mathcal{S}} (\bold{s}_{\Theta}(t, {\bold{x}}_t)) \big| \leq  \frac{C_4}{\sqrt{S}} 	 + 3 B_{\rm score\_err} \sqrt{\frac{2 \ln (4/\delta)}{S}}.
	\end{aligned}
	\label{equation: uniform generalization error bound for score functions}
\end{equation}
Here $B_{\rm score\_err}$ is defined in (\ref{proof inequality: estimation of difference of scores}), and
\begin{equation}
	\begin{aligned}
		C_4  = B_{\rm score\_err} {4\sqrt{2Ln(2n_{\rm d}+n_{\boldsymbol{e}})}}  \Psi(8B_{\Theta}\sqrt{ C_{\bold{s},2} /B_{\rm score\_err}}), 
	\end{aligned}
	\label{coefficients: C_4}
\end{equation}
with $	\Psi(u)= \sqrt{\ln(1 + u) + u \big(\ln(1+ u) - \ln(u)\big)}$.
\end{lemma}
%\begin{proof}
%	The proof is given in Section~\ref{}
%\end{proof}
%
%\begin{remark}
%	Lemma~\ref{lemma: generalization error of score functions} provides a uniform generalization bound of order $O(S^{-1/2})$. 
%	 This theorem guarantees that, with high probability, the discrepancy between empirical and population risks is uniformly controlled across the entire hypothesis class of score networks.
%\end{remark} 

%%%%%%%%%%%%%%%%%%%%%%%%%%%%%%%%%%%%%%%%%%%%%%%%%%%%%%%%%%%%%%%%%%
%%%%%%%%%%%%%%%%%%%%%%%%%%%%%%%%%%%%%%%%%%%%%%%%%%%%%%%%%%%%%%%%%%
%%%%%%%%%%%%%%%%%%%%%%%%%Next Theorem%%%%%%%%%%%%%%%%%%%%%%%%%%%%%
%In probability theory, an event is said to happen almost surely (sometimes abbreviated as a.s.) if it happens with probability 1 (with respect to the probability measure)
We are now ready to state the result characterizing the generalization and convergence properties of SDMs from the learning-problem perspective.
\begin{theorem}[generalization and convergence properties of learning problems]
\label{theorem: Convergence of learning problems}
Assume that assumptions (A1)-(A6) hold, and let $\Theta_{N,S}^*$ be an optimal solution of problem ($P_{N,S}$). Then: 
\begin{itemize}
	\item[(i)] For every $\delta \in (0,1)$, with probability at least $1-\delta$, the following bound holds uniformly for all $\Theta \in \Omega_{\Theta}$,
	\begin{equation}
		\begin{aligned}
			\big|\ell^{\rm DSM}_{N,S}(\Theta) \! - \! \ell^{\rm DSM}(\Theta) \big| \!\leq &    B_{\pmb{\lambda}} \! \sqrt{6B_{\rm score\_err}} \Big[ C_1  \omega_{\pmb{\gamma}}(h_N)  \!+\!  C_2 h_N^{1/2} \!+\! C_3 h_N \Big] \\
			& \ + \!    B_{\pmb{\lambda}}B_{\rm score\_err}  \, \omega_{\pmb{\lambda}}(h_N) + \frac{B_{\pmb{\lambda}}C_4}{\sqrt{S}} + 3B_{\pmb{\lambda}}B_{\rm score\_err}\sqrt{\frac{2 \ln (4/\delta)}{S}},
		\end{aligned}
		\label{equation: difference of losses}
	\end{equation}
	where $C_1, C_2, C_3$ and $C_4$ are defined in (\ref{coefficients: C_1_2_3}) and (\ref{coefficients: C_4}), respectively.
	
	\item[(ii)] The optimal values $\ell^*_{N,S}$ converge in probability to $\ell^*$,  as $(N,S)\to(\infty,\infty)$. Moreover, for any $\Theta^*_{N,S} \in \mathcal{O}^*_{N,S}$, $\mathrm{dist}(\Theta^*_{N,S}, \mathcal{O}^*)$ converges in probability to $0$,  as $(N,S)\to(\infty,\infty)$. 
	
	\item[(iii)] There exist subsequences $\{\ell^*_{N_k, S_k}\}_{k}$ and the associated optimal solutions $\{\Theta^*_{N_k, S_k}\}_{k}$ such that $\ell^*_{N_k, S_k} \! \to \! \ell^*$ almost surely, and every cluster point $\Theta^*$ of $\{\Theta^*_{N_k, S_k}\}_{k}$ is almost surely an optimal solution of problem $(\mathcal{P})$ .
\end{itemize}
\label{theorem: convergence of dis-time-finite problem}
\end{theorem}
%%%%%%%%%%%%%%%%%%%%%%%%%%%%%%
%\begin{proof}
%	The proof is given in Section~\ref{}
%\end{proof}

\begin{remark}
Assertion (i) of Theorem~\ref{theorem: Convergence of learning problems} establishes a uniform generalization bound of $\mathcal{O}(S^{-1/2})$ for the learned score functions. This result rigorously guarantees that, with high probability, the discrepancy between the empirical and population risks is uniformly controlled across the entire hypothesis class of score networks.
\end{remark}

\begin{remark}
Assertions (ii) and (iii) of Theorem~\ref{theorem: Convergence of learning problems} establishes the asymptotic consistency of the empirical learning problems, supplying a rigorous theoretical justification for training score-based diffusion models using finite data and discrete time steps. This demonstrates that the learned score function converges to the population-level optimum. In contrast to classical results on the consistency of M-estimators (e.g., van der Vaart \cite[Theorem~5.7]{van2000asymptotic}), which necessitate the uniqueness of the population minimizer, our analysis dispenses with the uniqueness assumption and is specifically tailored to diffusion models. Furthermore, unlike Li et al.~\cite{li2023generalization}, who analyze generalization bounded by sample size and network width strictly under a linear score-network regime, Theorem~\ref{theorem: convergence of dis-time-finite problem} accommodates general ResNet-type architectures. It simultaneously quantifies the compounding effects of finite sample size and time discretization while establishing the convergence of the learning problems themselves.
\end{remark}

%\textbf{统计学中的一致性 (consistency).} 
%当样本量 $n \to \infty$ 时，估计量 $\hat{\theta}_n$ 收敛到真值 $\theta_0$，称为一致性。常见的两种形式为：
%\begin{itemize}
%	\item \textbf{点一致性 (pointwise consistency):} $\hat{\theta}_n \xrightarrow{P} \theta_0$。
%	\item \textbf{强一致性 (strong consistency):} $\hat{\theta}_n \xrightarrow{a.s.} \theta_0$。
%\end{itemize}
%对于 M-estimators（通过最小化经验风险定义的估计量），van der Vaart 使用一致性来表述“估计量收敛到总体最优解”。

\subsection{Convergence and generalization of SDMs: Sampling perspective}
\label{Sec: 3.3}
In this subsection, we analyze the errors arising from the sampling procedure of SDMs. As discussed in Section~\ref{sec: 1}, once the score function $\bold{s}_{\Theta}$ has been learned from a dataset $\mathcal{S}$, the reverse-time SDE (\ref{equation: reverse-time SDE}) can be approximated by replacing $\nabla \ln p$ with $\bold{s}_{\Theta}$, yielding an approximate process $\{\hat{\bold{x}}_t\}_{t \in [0, T]}$ defined by (\ref{equation: reverse-time generative SDE}).

Let $\mu$ denote the path measure of the true reverse process \eqref{equation: reverse-time SDE}.
For any $\Theta \in \Omega_{\Theta}$, we write
$\nu$ for the path measure of the learned process \eqref{equation: reverse-time generative SDE} starting from $\bar{\bold{x}}_0 \sim {\boldsymbol{p}}(T,\cdot)$;
$\nu_{N,S}$ for the path measure of the discretized process \eqref{equation: Euler-Maruyama scheme} starting from $\check{\bold{x}}_0 \sim {\boldsymbol{p}}(T,\cdot)$;
$\nu_{N,S}^{\pi}$ for the discretized process starting from $\check{\bold{x}}_0 \sim \pi$;
and $\check{\boldsymbol{p}}(t,\cdot)$ for the density of $\check{\bold{x}}_t$.
It is natural to quantify the discrepancy between the exact reverse process and its discretized approximation via their induced path measures.

\begin{lemma}[Discretization error of sampling process] 
\label{lemma: Discretization error of sampling process}
Assume (A1)--(A6), and let $\Theta\in\Omega_\Theta$ be fixed. %Let $\nu$ (resp.\ $\nu_{N,S}$) denote the path measure on $C([0,T];\mathbb{R}^d)$ induced by the continuous-time SDE \eqref{equation: reverse-time generative SDE} (resp.\ its Euler–Maruyama discretization \eqref{equation: Euler-Maruyama scheme}), both started from the same marginal law $\bar{\mathrm{x}}_0\sim p_T$. 
Then, 	
$$
\begin{aligned}
	\operatorname{KL}(\nu \| \nu_{N,S}) 
	\leq  C_5  h_N + C_6 \omega_{\boldsymbol{\alpha}}^2 (h_N)  + C_7 \omega_{\boldsymbol{g}}^2 (h_N)	 ,
\end{aligned}
$$
where the constants $C_5, C_6, C_7$ are defined by
\begin{equation}
	\begin{aligned}
		C_5 = & \frac{2T}{b_{\boldsymbol{g}}^2} \big(2 C_{\bold{f}}^2 \mathrm{Lip}_{\hat{\bold{x}}, 2} + 2 B_{\boldsymbol{g}}^2 \big(1+L^2\mathrm{Lip}^2_{\psi} \mathrm{Lip}^2_{\boldsymbol{e}}  B_{\Theta}^4 \big) \exp(2L \mathrm{Lip}_{\psi} B_{\Theta}^2 ) \big) \big( \mathrm{Lip}_{\hat{\bold{x}}, 2} + h_N  \big), \\
		C_6 = & \frac{2T}{b_{\boldsymbol{g}}^2} 2  C_{\bold{f}}^2 B_{\hat{\bold{x}},2} , \
		C_7 =  \frac{2T}{b_{\boldsymbol{g}}^2} 16B_{\boldsymbol{g}}^2 	\big( B_{\hat{\bold{x}},2} + L^2 \mathrm{Lip}_{\psi}^2 B_{\boldsymbol{e}}^2  B_{\Theta}^{4} \big) \exp(2L \mathrm{Lip}_{\psi}  B_{\Theta}^2).
	\end{aligned}
	\label{coefficients: C_5_6}
\end{equation}
with coefficients $B_{\hat{\bold{x}}, 2}, \mathrm{Lip}_{\hat{\bold{x}}, 2}, C_{\hat{\bold{f}}}$ given in Lemma~\ref{lemma: property of SDE-1},
%	\begin{equation}
	%		\begin{aligned}
		%		B_{\hat{\bold{x}}, 2} &= (1+B_{\bold{x}, 2}) e^{12C_{\hat{\bold{f}}}^2T}, \mathrm{Lip}_{\hat{\bold{x}}, 2}= 8 C_{\hat{\bold{f}}}^2 (T + 1) \big[ 1+ \big(1+ B_{\bold{x}, 2}\big) e^{12C_{\hat{\bold{f}}}^2T} \big]   e^{12C_{\hat{\bold{f}}}^2 T}, \\
		%		C_{\hat{\bold{f}}} &= C_{\bold{f}} +B_{\boldsymbol{g}} ^2 \exp(L \mathrm{Lip}_{\psi}  B_{\Theta}^2) \cdot \max \{1, L \mathrm{Lip}_{\psi} B_{\boldsymbol{e}}  B_{\Theta}^2 \},
		%	\end{aligned}
	%	\label{coefficients: B_hatx2&Lip_hatx2}
	%	\end{equation}
and $\omega_{\boldsymbol{g}}, \omega_{\boldsymbol{\alpha}}$ being the moduli of continuity for $\boldsymbol{g}$ and $\boldsymbol{\alpha}$.
\end{lemma}

We can now state the main sampling–error estimate.
\begin{theorem}[Sampling error]
\label{theorem: sampling error}
Fix $N, S \in \mathbb{N}$ and take $\pmb{\lambda}(\cdot) = \boldsymbol{g}^2(\cdot)$. Assume that (A1)-(A6) hold and there exists an integrable bounding function $\boldsymbol{h}: [0,T]\times \mathbb{R}^{d} \to \mathbb{R}_{+}$ such that for all $(t, \bold{x}_t ) \in [0,T]\times \mathbb{R}^{d}$,

%	Fix $N, S \in \mathbb{N}$ and take $\pmb{\lambda}(\cdot) = \boldsymbol{g}^2(\cdot)$. Assume that (A1)-(A6) hold and there exists a function $\boldsymbol{h}: [0,T]\times \mathbb{R}^{d}$ such that $\forall (t, \bold{x}_t ) \in  [0,T]\times \mathbb{R}^{d}$ 
\begin{equation}
	\begin{aligned}
		\| \boldsymbol{p}(0,\mathrm{x}) \boldsymbol{p}(t,\bold{x}_t|\mathrm{x})\|_2  \leq \boldsymbol{h}(t, \mathrm{x}), \ 
		\text{ and }	 \int_{\mathrm{x}}  \boldsymbol{h}(t, \mathrm{x})  \mathrm{d} \mathrm{x}< \infty.
	\end{aligned}
	\label{assumption: changable condition}
\end{equation}
%and $\mathbb{E}\Big[\exp\Big(\frac{1}{2} \int_{0}^{T} \|  \nabla \ln \boldsymbol{p}(T-t, \bar{\bold{x}}_t)  - \bold{s}_{\Theta}(T-t, \bar{\bold{x}}_t) \|_2^2 \mathrm{d}t  \Big)\Big] < \infty$ for some $\Theta \in \Omega_{\Theta}$, 
Then for any $\delta \in (0,1)$, with probability at least $1-\delta$,
\begin{equation}
	\begin{aligned}
		& \operatorname{TV}(\check{\boldsymbol{p}}(T,\cdot) \| p_{\rm data}) \\
		\leq  &  	\underbrace{\operatorname{TV}(\pi \| \boldsymbol{p}(T, \cdot))}_{\text{truncation error}} + \underbrace{C_8 \Big(
			\omega_{\boldsymbol{\alpha}} (h_N) + \omega_{\boldsymbol{g}} (h_N) + \sqrt{\omega_{\boldsymbol{g}^2}(h_N)} +  \sqrt{\omega_{\pmb{\gamma}}(h_N)}
			\Big) + C_9 h_N^{1/4}}_{\text{Discretization and score error}} \\
		& \  +  \underbrace{
			\frac{\sqrt{T B_{\pmb{\lambda}} C_4}}{4 S^{1/4}}
			+ \frac{(3 T B_{\pmb{\lambda}} B_{\rm score\_err})^{1/2}}{4}\Big(\frac{2 \ln (4/\delta)}{S}\Big)^{1/4}
		}_{\text{Generalization error}} + \underbrace{\sqrt{ \frac{T}{2}(\ell_{N, S}^{\rm DSM}(\Theta) - \ell_{N, S}^*)} }_{\text{optimization error}} + \sqrt{ \frac{T\ell_{N, S}^*}{2}},
	\end{aligned}
	\label{equation: total sampling error}
\end{equation}
where $\check{\boldsymbol{p}}(T, \cdot)$ is the generated terminal distribution of the Euler–Maruyama process \eqref{equation: Euler-Maruyama scheme}
with parameter $\Theta$, and the coefficients
$$
\begin{aligned}
	C_8 = & \max \left\lbrace  \sqrt{\frac{C_6}{2}}, \sqrt{\frac{C_7}{2}}, \frac{(TB_{\pmb{\lambda}}C_1)^{1/2} \sqrt[4]{6B_{\rm score\_err}}}{2} , \frac{(TB_{\pmb{\lambda}}B_{\rm score\_err})^{1/2}}{2} \right\rbrace \\
	C_9 = & \frac{(TB_{\pmb{\lambda}})^{1/2} \sqrt[4]{6B_{\rm score\_err}}}{4} \Big[\sqrt{C_2}  \!+\! \sqrt{C_3} h_N^{1/4} \Big] + \sqrt{\frac{C_5}{2}}  h_N^{1/4},
\end{aligned} 
$$ 
with $C_1, C_2, C_3$; $C_4, C_5;$ $C_6,C_7$ defined in (\ref{coefficients: C_1_2_3}),  (\ref{coefficients: C_4}) and (\ref{coefficients: C_5_6}), respectively.
\end{theorem}

\begin{remark}
Theorem~\ref{theorem: sampling error} establishes a unified framework that quantitatively bridges the training and sampling phases of diffusion models. While these two stages have traditionally been analyzed in isolation within the literature, our bound shows how they affect each other through simple and clear error terms.
\begin{itemize}
	\item {Prior mismatch:} The term $\operatorname{TV}(\pi \| \boldsymbol{p}(T,\cdot))$ quantifies the mismatch between the forward-time terminal distribution and the prior noise distribution. Under standard contractivity conditions for the forward SDE, this term decays exponentially with respect to the time horizon~$T$ \cite{chen2023sampling, tang2024contractive}.
	\item {Discretization and eegularity:} The quantities governed by $h_N$, $\omega_{\boldsymbol{\alpha}}$, $\omega_{\boldsymbol{g}}$, $\omega_{\boldsymbol{g}^2}$, and $\omega_{\pmb{\gamma}}$ capture the discretization error and the impact of score-smoothness within the reverse SDE. Their dependence on moduli of continuity underscores the critical role of temporal regularity in the forward and reverse coefficients, aligning with established empirical observations \cite{nichol2021improved}.
	\item{Generalization error:} The $\mathcal{O}(S^{-1/4})$ rate stems directly from the generalization bounds of the score estimator. The degradation from the classical $\mathcal{O}(S^{-1/2})$ rate is a mathematically necessary consequence of converting a KL-divergence-type control into a total variation bound.
	\item {Optimization error:} Finally, the terms $\sqrt{\ell_{N,S}^*}$ and $\sqrt{\ell_{N,S}^{\rm DSM}(\Theta)-\ell_{N,S}^*}$ rigorously encapsulate the empirical training error and the optimization gap of the learned score network.
\end{itemize}
\end{remark}

\begin{remark}[Comparison with prior work]
Existing theoretical analyses of diffusion models (e.g., \cite{lee2022convergence, lee2023convergence, chen2023sampling, tang2024contractive}) predominantly assume access to an oracle score function, narrowing their focus exclusively to the discretization error of the reverse SDE or ODE. In contrast, Theorem~\ref{theorem: sampling error} provides an end-to-end performance guarantee. We explicitly quantify how the interplay between the score approximation error, the generalization error (driven by finite sample size~$S$), the optimization error, and the time-discretization error (driven by~$N$) jointly dictates the final generative fidelity. This provides a comprehensive characterization of trained diffusion samplers in practice, moving beyond the idealized oracle-score setting to offer actionable insights into the trade-offs between data volume, network optimization, and numerical discretization.
\end{remark}

%%%%%%%%%%%%%%%%%%%%%%%%%%%%%%%%%%%%%%%%%%%%%%%%%%%%%%%%%%%%%%%%%%%%%%%%%%%%%%%%%%%%%%%%%%%%%%%%%%%%

\section{Proofs}
\label{Sec: 4}
In this section, we provide detailed proofs for the results in Section~\ref{Sec: 3.2} and Section~\ref{Sec: 3.3}.

\subsection{Proof of results in Section~\ref{Sec: 3.2}}

We first prove Lemma~\ref{lemma: Existence of solutions for learning problems}.
\begin{proof}[proof of Lemma~\ref{lemma: Existence of solutions for learning problems}]
The proof is straightforward by the compactness of the learnable parameters, the continuity of the loss functions, and the continuous dependency of the network $\bold{s}_{\Theta}$ on learnable parameters in Lemma~\ref{lemma: property of DNN}. Here, we omit the details.
\end{proof}

Now, we give a detailed proof for Lemma~\ref{lemma: Convergence from {P}_{N,S} to {P}_{S}}.
%%%%%%%%%%%%%%%%%%%%%%%%%%%%%%%%%%%%%%%%%%%%%%%%%%%%%%%%%%%%%%%%%%%%%%%%
\begin{proof}[Proof of Lemma~\ref{lemma: Convergence from {P}_{N,S} to {P}_{S}}]
By the definitions of $\ell_{N,S}^{\rm DSM}$ and $\ell_{S}^{\rm DSM}$ in (\ref{loss: discrete-time-finite-sample loss}) and (\ref{loss: continuous-time-finite-sample loss}), we have 
\begin{align}
	&	| \ell_{N,S}^{\rm DSM}(\Theta) - 	\ell_{S}^{\rm DSM}(\Theta) | \nonumber \\
	%%%%%%%%%%%%%%%%%%%%%%%%%%%%%%%%%%%%%%%%%%
	= &  \Big| \frac{h_N}{TS} \sum_{k=0}^{N \!-\!1}\! \pmb{\lambda}(t_{N}^k)  \!\sum_{s=1}^{S} \!\mathbb{E}_{\bold{x}_{t_{N}^k} \! \sim   \boldsymbol{p}(t_{N}^k, \cdot|\mathrm{x}^{(s)}\!)} \! \left[\big\|\bold{s}_{\Theta}(t_{N}^k, {\bold{x}_{t_{N}^k}} ) \!- \! \nabla \ln \boldsymbol{p}(t_{N}^k \!, \bold{x}_{t_{N}^k}|\mathrm{x}^{(s)})\big\|_2^2\right]  \nonumber \\
	& \qquad \qquad  \quad - \frac{1}{TS} \int_{0}^{T} \pmb{\lambda}(t)  \sum_{s=1}^{S} \mathbb{E}_{\bold{x}_t \sim \boldsymbol{p}(t, \cdot|\mathrm{x}^{(s)})} \left[\big\|\bold{s}_{\Theta}(t, {\bold{x}}_t) -\nabla \ln \boldsymbol{p}(t, \bold{x}_t|\mathrm{x}^{(s)})\big\|_2^2\right] \mathrm{d} t \Big | \label{proof inequality: convergence of time in finite sample 1}
	\\
	%%%%%%%%%%%%%%%%%%%%%%%%%%%%%%%%%%%%%%%%%%%%%%
	%%%%%%%%%%%%%%%%%%%%%%%%%%%%%%%%%%%%%%%%%%%%%%
	\textcolor{red}{\leq} & \frac{1}{TS} \sum_{k=0}^{N-1} \sum_{s=1}^{S}  \int_{t_N^k}^{t_N^{k+1}}  \Big |  \pmb{\lambda}(t_{N}^k) \mathbb{E}_{\bold{x}_{t_{N}^k} \! \sim   \boldsymbol{p}(t_{N}^k, \cdot|\mathrm{x}^{(s)}\!)} \! \left[\big\|\bold{s}_{\Theta}(t_{N}^k, {\bold{x}_{t_{N}^k}}) \!- \! \nabla \ln \boldsymbol{p}(t_{N}^k \!, \bold{x}_{t_{N}^k}|\mathrm{x}^{(s)})\big\|_2^2\right] \nonumber \\
	& \qquad \qquad \qquad \qquad \quad -  \pmb{\lambda}(t) \mathbb{E}_{\bold{x}_t \sim \boldsymbol{p}(t, \cdot|\mathrm{x}^{(s)})} \left[\big\|\bold{s}_{\Theta}(t, {\bold{x}}_t) -\nabla \ln \boldsymbol{p}(t, \bold{x}_t|\mathrm{x}^{(s)})\big\|_2^2\right]  \Big | \mathrm{d}t \nonumber \\
	%%%%%%%%%%%%%%%%%%%%%%%%%%%%%%%%%%%%%%%%%%%%%%
	%%%%%%%%%%%%%%%%%%%%%%%%%%%%%%%%%%%%%%%%%%%%%%
	\leq & \frac{1}{TS}  \sum_{k=0}^{N-1} \! \sum_{s=1}^{S} \int_{t_N^k}^{t_N^{k+1}}  \! \Big | \pmb{\lambda}(t_{N}^k) \! - \! \pmb{\lambda}(t) \Big | 
	\mathbb{E} \left[\big\|\bold{s}_{\Theta}(t_{N}^k, {\bold{x}_{t_{N}^k}}) \!- \! \nabla \ln \boldsymbol{p}(t_{N}^k \!, \bold{x}_{t_{N}^k}|\mathrm{x}^{(s)})\big\|_2^2\right]  \mathrm{d} t \nonumber \\
	& + \frac{1}{TS}  \sum_{k=0}^{N-1} \! \sum_{s=1}^{S} \int_{t_N^k}^{t_N^{k+1}}  \! \pmb{\lambda}(t)  \Big | \mathbb{E} \Big[  
	\big\|\bold{s}_{\Theta}(t_{N}^k, {\bold{x}_{t_{N}^k}}) \!- \! \nabla \ln \boldsymbol{p}(t_{N}^k \!, \bold{x}_{t_{N}^k}|\mathrm{x}^{(s)})\big\|_2^2  \nonumber \\
	& \qquad \qquad \qquad \qquad \qquad \qquad \quad
	- \big\|\bold{s}_{\Theta}(t, {\bold{x}}_t) -\nabla \ln \boldsymbol{p}(t, \bold{x}_t|\mathrm{x}^{(s)})\big\|_2^2\Big] \Big | \mathrm{d}t. \nonumber
\end{align}
%where the last inequality replaces the expectation over $\boldsymbol{p}(t_{N}^i, \cdot|\mathrm{x}^{(s)}\!)$ and $\boldsymbol{p}(t, \cdot|\mathrm{x}^{(s)}\!)$ with an equivalent expectation over the Brownian motion \(\mathbf{B}\), as the distributions of \(\mathbf{x}_t\) are determined entirely by the SDE driven by \(\mathbf{B}\) starting from \(\mathrm{x}^{(s)}\).
%We next estimate the right-hand side of the above inequality term by term.

%%%%%%%%%%%%%%%%%%%%%%%%%%%%%%%%%%%%%%%%%%%%%%%%%%%%%%%%%%%%%%%%%%%%%%%%%%%%%%%%%%%%%%%%%%%%%%%%%%%%%%%

For the first term in the right hand side of (\ref{proof inequality: convergence of time in finite sample 1}), we have, by (\ref{proof inequality: estimation of difference of scores}) that
\begin{equation}
	\begin{aligned}
		&	\frac{1}{TS}  \sum_{k=0}^{N-1} \! \sum_{s=1}^{S} \int_{t_N^k}^{t_N^{k+1}}  \! \big | \pmb{\lambda}(t_{N}^k) \! - \! \pmb{\lambda}(t) \big | 
		\mathbb{E} \! \left[\big\|\bold{s}_{\Theta}(t_{N}^k, {\bold{x}_{t_{N}^k}}) \!- \! \nabla \ln \boldsymbol{p}(t_{N}^k \!, \bold{x}_{t_{N}^k}|\mathrm{x}^{(s)})\big\|_2^2\right] \! \mathrm{d} t \\
		\leq  & B_{\rm score\_err} \cdot \omega_{\pmb{\lambda}}(h_N),
	\end{aligned}
	\label{proof inequality: estimation of first part}
\end{equation}
where $\omega_{\pmb{\lambda}}$ is the modulus of continuity of $\pmb{\lambda}$.
For the second term in the right hand side of (\ref{proof inequality: convergence of time in finite sample 1}), we derive
\begin{align*}
	&\int_{t_N^k}^{t_N^{k+1}}  \! \pmb{\lambda}(t)  \Big | \mathbb{E}  \Big[  
	\big\|\bold{s}_{\Theta}(t_{N}^k, {\bold{x}_{t_{N}^k}} ) -  \nabla \ln \boldsymbol{p}(t_{N}^k \!, \bold{x}_{t_{N}^k}|\mathrm{x}^{(s)})\big\|_2^2 \\
	& \qquad \qquad \qquad \qquad \qquad \qquad \qquad \qquad \qquad \qquad	- \big\|\bold{s}_{\Theta}(t, {\bold{x}}_t)\! -\!\nabla \ln \boldsymbol{p}(t, \bold{x}_t|\mathrm{x}^{(s)})\big\|_2^2\Big] \Big | \mathrm{d}t \\
	%%%%%%%%%%%%%%%%%%%%%%%%%%%%%%%%%%%%%%%%%%%%%%%%%%%
	%%%%%%%%%%%%%%%%%%%%%%%%%%%%%%%%%%%%%%%%%%%%%%%%%%%
	\textcolor{red}{\leq} & B_{\pmb{\lambda}} \! \int_{t_N^k}^{t_N^{k+1}}  \mathbb{E} \! \Big[  \big \| \bold{s}_{\Theta}(t_{N}^k, {\bold{x}_{t_{N}^k}}) -  \nabla \ln \boldsymbol{p}(t_{N}^k \!, \bold{x}_{t_{N}^k}|\mathrm{x}^{(s)})   - 
	\bold{s}_{\Theta}(t, {\bold{x}}_t)\! + \!\nabla \ln \boldsymbol{p}(t, \bold{x}_t|\mathrm{x}^{(s)}) \big\| \\
	&  \qquad \qquad \quad \times  \big(  \big\| \bold{s}_{\Theta}(t_{N}^k,{\bold{x}_{t_{N}^k}}) -  \nabla \ln \boldsymbol{p}(t_{N}^k \!, \bold{x}_{t_{N}^k}|\mathrm{x}^{(s)}) \| + \| \bold{s}_{\Theta}(t, {\bold{x}}_t)\! -\!\nabla \ln \boldsymbol{p}(t, \bold{x}_t|\mathrm{x}^{(s)}) \big\| \big) \Big]  \mathrm{d}t\\
	%%%%%%%%%%%%%%%%%%%%%%%%%%%%%%%%%%%%%%%%%%%%%%%%%%
	%%%%%%%%%%%%%%%%%%%holder inequality%%%%%%%%%%%%%%
	\leq & B_{\pmb{\lambda}} \! \int_{t_N^k}^{t_N^{k+1}} \! \Big[ \mathbb{E}   \big( \| \bold{s}_{\Theta}(t_{N}^k \!, {\bold{x}_{t_{N}^k}}) \! - \! \bold{s}_{\Theta}(t, {\bold{x}}_t) \|  \!+ \!
	\| \nabla \! \ln \boldsymbol{p}(t_{N}^k \!, \bold{x}_{t_{N}^k}|\mathrm{x}^{(s)}) \! -\!\nabla \! \ln \boldsymbol{p}(t, \bold{x}_t|\mathrm{x}^{(s)})\|\big)^2 \Big]^{1/2} \\
	&  \qquad \quad \times \Big[\mathbb{E}  \big(  \| \bold{s}_{\Theta}(t_{N}^k,{\bold{x}_{t_{N}^k}}) \!-\!  \nabla \ln \boldsymbol{p}(t_{N}^k \!, \bold{x}_{t_{N}^k}|\mathrm{x}^{(s)}) \| \!+\! \| \bold{s}_{\Theta}(t, {\bold{x}}_t)\! -\!\nabla \ln \boldsymbol{p}(t, \bold{x}_t|\mathrm{x}^{(s)}) \big)^2 \Big]^{1/2}  \mathrm{d}t\\
	%%%%%%%%%%%%%%%%%%%%%%%%%%%%%%%%%%%%%%%%%%%%%%%%%%%%
	%%%%%%%%%%%%%%%%%%%%%%%%%%%%%%%%%%%%%%%%%%%%%%%%%%%%
	\leq &  B_{\pmb{\lambda}} \sqrt{2B_{\rm score\_err}}  \int_{t_N^k}^{t_N^{k+1}}  \Big[ 3 \mathbb{E}   \big[ \| \bold{s}_{\Theta}(t_{N}^k\!, {\bold{x}_{t_{N}^k}}) \! - \! \bold{s}_{\Theta}(t, {\bold{x}}_t) \|^2\big]
	+ 3\mathbb{E}  \big[ \| \nabla \ln \boldsymbol{p}(t_{N}^k \!, \bold{x}_{t_{N}^k}|\mathrm{x}^{(s)}) \! \\
	& \qquad \qquad \qquad  -\! \nabla \ln \boldsymbol{p}(t, \bold{x}_{t_{N}^k}|\mathrm{x}^{(s)})\|^2 \big] \!+3 
	\mathbb{E} \big[ \| \nabla \ln \boldsymbol{p}(t, \bold{x}_{t_{N}^k}|\mathrm{x}^{(s)}) \! -\!\nabla \ln \boldsymbol{p}(t, \bold{x}_t|\mathrm{x}^{(s)})\|^2 \big]  \Big]^{1/2}  \mathrm{d}t\\
	%%%%%%%%%%%%%%%%%%%%%%%%%%%%%%%%%%%%%%%%%%%%%%%%%%%%
	%%%%%%%%%%%%%%%%%%%%%%%%%%%%%%%%%%%%%%%%%%%%%%%%%%%%
	\leq &   B_{\pmb{\lambda}} \sqrt{6B_{\rm score\_err}}  \int_{t_N^k}^{t_N^{k+1}}  \Big[ 2 \mathbb{E}\big[ \| \bold{x}_{t_{N}^k} - \bold{x}_t\|_2^2 \!+\! L^2\mathrm{Lip}_{\psi}^2 \mathrm{Lip}^2_{\boldsymbol{e}}  B_{\Theta}^4 |t_N^k - {t}|^2 \big]  \exp(2L \mathrm{Lip}_{\psi}  B_{\Theta}^2) 	\\
	& \qquad \qquad \qquad \qquad \ +  3 C_{\boldsymbol{p}}^2 \mathbb{E}  \big[  \big( 
	1 \! + \! \|\bold{x}_{t_{N}^k}\|_2^2 \!+\! \|\mathrm{x}^{(s)}\|^2_2
	\big) \cdot \omega_{\pmb{\gamma}}^2(h_N) \big] 
	\!+ \!
	C_{\boldsymbol{p}}^2 \mathbb{E}  \big[ \| \bold{x}_{t_{N}^k} - \bold{x}_t\|_2^2 \big]  \Big]^{1/2}  \mathrm{d}t,  %\\
	%%%%%%%%%%%%%%%%%%%%%%%%%%%%%%%%%%%%%%%%%%
	%%%%%%%%%%%%%%%%%%%%%%%%%%%%%%%%%%%%%%%%%%
	%			\leq & \textcolor{red}{ B_{\pmb{\lambda}} \sqrt{6B_{\rm score\_err}}  \int_{t_N^k}^{t_N^{k+1}}  \Big[  \Big(\mathbb{E} \big[ C_{\bold{s}}^4  (\bold{x}_{t_{N}^k})\big] \Big)^{1/2} \Big( \mathbb{E} \big[ \big(\| \bold{x}_{t_{N}^k} - \bold{x}_t\|_2  + h_N\big)^4 \big] \Big)^{1/2}  }	\\
	%			& \qquad \qquad \qquad \qquad \ + \! C_{\boldsymbol{p}}^2 \mathbb{E}  \big[  \big( 
	%			1 \! + \! \|\bold{x}_{t_{N}^k}\|_2^2 \!+\! \|\mathrm{x}^{(s)}\|^2_2
	%			\big) \cdot \omega_{\pmb{\gamma}}^2(h_N) \big] 
	%			\!+ \!
	%			C_{\boldsymbol{p}}^2 \mathbb{E}  \big[ \| \bold{x}_{t_{N}^k} - \bold{x}_t\|_2^2 \big]  \Big]^{1/2}  \mathrm{d}t
\end{align*} 
where $k=0,1,\ldots, N-1, s=1,\ldots,S$, the second inequality uses H\"older's inequality, the third inequality uses (\ref{proof inequality: estimation of difference of scores}) and the last inequality uses assumptions (A4) and Lemma~\ref{lemma: property of DNN}. Using Lemma~\ref{lemma: property of SDE-1}, we obtain from the above inequality that, 
\begin{equation}
	\begin{aligned}
		& \frac{1}{TS}  \sum_{k=0}^{N-1} \! \sum_{s=1}^{S} \int_{t_N^k}^{t_N^{k+1}}  \! \pmb{\lambda}(t)  \Big | \mathbb{E}  \Big[  
		\big\|\bold{s}_{\Theta}(t_{N}^k, {\bold{x}_{t_{N}^k}} ) -  \nabla \ln \boldsymbol{p}(t_{N}^k \!, \bold{x}_{t_{N}^k}|\mathrm{x}^{(s)})\big\|_2^2 \\
		& \qquad \qquad \qquad \qquad \qquad \qquad \qquad \qquad \qquad	- \big\|\bold{s}_{\Theta}(t, {\bold{x}}_t)\! -\!\nabla \ln \boldsymbol{p}(t, \bold{x}_t|\mathrm{x}^{(s)})\big\|_2^2\Big] \Big | \mathrm{d}t \\
		%%%%%%%%%%%%%%%%%%%%%%%%%%%%%%%%%%%%%%%%%%%%%%%%%%%%%
		%%%%%%%%%%%%%%%%%%%%%%%%%%%%%%%%%%%%%%%%%%%%%%%%%%%%%
		\leq &   \frac{B_{\pmb{\lambda}} \sqrt{6B_{\rm score\_err}}}{TS}  \!  \sum_{k=0}^{N-1} \! \sum_{s=1}^{S} \int_{t_N^k}^{t_N^{k+1}} \!  \Big[2
		\big(	\mathrm{Lip}_{\bold{x},2} |t_N^k - t|  \!+\! L^2\mathrm{Lip}_{\psi}^2 \mathrm{Lip}^2_{\boldsymbol{e}}  B_{\Theta}^4 |t_N^k \! - \!  {t}|^2 \big)  \exp(2L \mathrm{Lip}_{\psi}  B_{\Theta}^2) \\
		& \qquad \qquad  \qquad \qquad  \qquad   \  + 3C_{\boldsymbol{p}}^2    \big( 1  + B_{\bold{x}, 2} + R^2 \big) \cdot \omega_{\pmb{\gamma}}^2(h_N)
		+ C_{\boldsymbol{p}}^2 \mathrm{Lip}_{\bold{x},2} |t_N^k- t| 	\Big]^{1/2} \mathrm{d}t \\
		\leq &  B_{\pmb{\lambda}} \sqrt{6B_{\rm score\_err}} \Big[ C_1  \omega_{\pmb{\gamma}}(h_N)  +  C_2 h_N^{1/2} + C_3 h_N \Big],
	\end{aligned}
	\nonumber
\end{equation}
where $C_1 = C_{\boldsymbol{p}}   \sqrt{3\big( 1  + B_{\bold{x}, 2} + R^2 \big)}$,  $C_2 = \sqrt{2\mathrm{Lip}_{\bold{x},2}\exp(2L \mathrm{Lip}_{\psi}  B_{\Theta}^2)   + C_{\boldsymbol{p}}^2 \mathrm{Lip}_{\bold{x},2}   }  $ and $C_3 = 2L \mathrm{Lip}_{\psi} \mathrm{Lip}_{\boldsymbol{e}}  B_{\Theta}^2\exp(L \mathrm{Lip}_{\psi}  B_{\Theta}^2) $ with $B_{\bold{x}, 2}$ and $\mathrm{Lip}_{\bold{x},2}$ defined in (\ref{coefficients: B_x2&Lip_x2}).
Combining the above inequality with (\ref{proof inequality: estimation of first part}) , we complete the proof. 
\end{proof}

%%%%%%%%%%%%%%%%%%%%%%%%%%%%%%%%%%%%%%%%%%%%%%%%%%%%%%%%%%%%%%%
%%%%%%%%%%%%%%%%%%%%%%%%%%%%%%%%%%%%%%%%%%%%%%%%%%%%%%%%%%%%%%%
Next, we give detailed proofs for Lemma~\ref{lemma: generalization error of score functions} . 

\begin{proof}[proof of Lemma~\ref{lemma: generalization error of score functions}]
The proof is divided into two steps.

\textbf{Step~1}. 
We use the Rademacher complexity (Definition~\ref{definition: RC} in Appendix) to estimate the upper bound of $\big| \mathfrak{R}_{p_{\rm data}} (\bold{s}_{\Theta}(t, {\bold{x}}_t)) - \hat{\mathfrak{R}}_{\mathcal{S}} (\bold{s}_{\Theta}(t, {\bold{x}}_t)) \big|$.
For this, we consider the generalization gap $\mathrm{GGap}(\mathcal{S}) := \sup_{\Theta \in \Omega} \Big ( \mathfrak{R}_{p_{\rm data}} (\bold{s}_{\Theta}(t, {\bold{x}}_t)) - \hat{\mathfrak{R}}_{\mathcal{S}} (\bold{s}_{\Theta}(t, {\bold{x}}_t)) \Big)$.
% by McDiarmid's inequality. For this, we need to validate the value of $\mathrm{GGap}(\cdot)$ can change by at most a $C^{\prime}>0$ under an arbitrary change of a single coordinate of $\mathrm{x} \in \mathbb{R}^{n_\mathrm{d}}$. Actually, 
Recall the dataset ${\mathcal{S}} = \{ {\mathrm{x}}^{(1)}, \ldots, {\mathrm{x}}^{(i)}, \ldots, {\mathrm{x}}^{(S)} \}$.
For any $i \in \{1, \ldots,S\}$, let ${\mathcal{S}}^{(i)} = \{ {\mathrm{x}}^{(1)}, \ldots, {\mathrm{x}}^{(i-1)}, $ $\tilde{\mathrm{x}}^{(i)}, {\mathrm{x}}^{(i+1)}, \ldots, {\mathrm{x}}^{(S)} \}$. Then we have 
\begin{equation}
	\begin{aligned}
		& \mathrm{GGap}(\mathcal{S}) - \mathrm{GGap}({\mathcal{S}}^{(i)}) \\
		%%%%%%%%%%%%%%%%%%%%%%%%%%%%%%%%%%%%%%%%%%%%
		%%%%%%%%%%%%%%%%%%%%%%%%%%%%%%%%%%%%%%%%%%%%
		=& \sup_{\Theta \in \Omega_{\Theta}} \big ( \mathfrak{R}_{p_{\rm data}} (\bold{s}_{\Theta}(t, {\bold{x}}_t)) \!-\! \hat{\mathfrak{R}}_{\mathcal{S}} (\bold{s}_{\Theta}(t, {\bold{x}}_t)) \big) \!- \!\sup_{\Theta \in \Omega_{\Theta}} \big ( \mathfrak{R}_{p_{\rm data}} (\bold{s}_{\Theta}(t, {\bold{x}}_t)) \!-\! \hat{\mathfrak{R}}_{\tilde{\mathcal{S}}^{(i)}} (\bold{s}_{\Theta}(t, {\bold{x}}_t)) \big) \\
		\leq & \sup_{\Theta \in \Omega_{\Theta}} \Big(\hat{\mathfrak{R}}_{{\mathcal{S}}^{(i)}} (\bold{s}_{\Theta}(t, {\bold{x}}_t)) - \hat{\mathfrak{R}}_{\mathcal{S}} (\bold{s}_{\Theta}(t, {\bold{x}}_t)) \Big) \\
		%	\leq & \sup_{\Theta \in \Omega_{\Theta}}  \Big( \frac{1}{S} \sum_{s=1}^{S} \mathbb{E}_{\bold{x}_t \sim \boldsymbol{p}(t, \cdot|\mathrm{x}^{(s)})} \left[\big\|\bold{s}_{\Theta}(t, {\bold{x}}_t) -\nabla \ln \boldsymbol{p}(t, \bold{x}_t|\mathrm{x}^{(s)})\big\|_2^2\right] \\
		%	 & \qquad \qquad \qquad \qquad - \frac{1}{S} \sum_{s=1}^{S} \mathbb{E}_{\bold{x}_t \sim \boldsymbol{p}(t, \cdot|\mathrm{x}^{(s)})} \left[\big\|\bold{s}_{\Theta}(t, {\bold{x}}_t) -\nabla \ln \boldsymbol{p}(t, \bold{x}_t|\mathrm{x}^{(s)})\big\|_2^2\right] \Big) \\
		%%%%%%%%%%%%%%%%%%%%%%%%%%%%%%%%%%%%%%%%%%%%
		%%%%%%%%%%%%%%%%%%%%%%%%%%%%%%%%%%%%%%%%%%%%
		\leq &	\frac{1}{S} \sup_{\Theta \in \Omega_{\Theta}} \Big(   \mathbb{E}_{\bold{x}_t \sim \boldsymbol{p}(t, \cdot|\tilde{\mathrm{x}}^{(i)})} \left[\big\|\bold{s}_{\Theta}(t, {\bold{x}}_t) -\nabla \ln \boldsymbol{p}(t, \bold{x}_t|\tilde{\mathrm{x}}^{(i)})\big\|_2^2\right] \\
		& \qquad \qquad \qquad \qquad \qquad \qquad - \mathbb{E}_{\bold{x}_t \sim \boldsymbol{p}(t, \cdot|\mathrm{x}^{(i)})} \left[\big\|\bold{s}_{\Theta}(t, {\bold{x}}_t) -\nabla \ln \boldsymbol{p}(t, \bold{x}_t|{\mathrm{x}}^{(i)})\big\|_2^2\right] \Big) \\
		\leq & \frac{2}{S}B_{\rm score\_err},
	\end{aligned}
	\label{proof inequality-2: property of change of a single coordinate}
\end{equation}
where the last inequality is owing to (\ref{proof inequality: estimation of difference of scores}). 
%%%%若交换GGap定义式的两项，这里仍然成立。
Hence, using McDiarmid's inequality \cite[Lemma~26.4]{shalev2014understanding}, we get, for all $\delta \in (0,1)$, with probability at least $1-\delta/2$, 
\begin{equation}
	\begin{aligned}
		\mathrm{GGap}(\mathcal{S}) \leq  	\mathbb{E}_{\mathcal{S} \sim (p_{\rm data})^S}[\mathrm{GGap}(\mathcal{S})  ] + B_{\rm score\_err} \sqrt{\frac{2 \ln (4/\delta)}{S}}.
	\end{aligned}
	\label{proof inequality-2: McDiarmid's inequality-1}
\end{equation}

To further establish the properties of $\mathrm{GGap}(\mathcal{S})$ in relation to the Rademacher complexity, we let 
$\tilde{\mathcal{S}} = \{ \tilde{\mathrm{x}}^{(1)}, \ldots, \tilde{\mathrm{x}}^{(S)} \}$ be another i.i.d. sample set from $p_{\rm data}$, which is independent of $\mathcal{S}$.
Using the fact that
$\mathfrak{R}_{p_{\rm data}} (\bold{s}_{\Theta}(t, {\bold{x}}_t)) \!=\! \mathbb{E}_{\tilde{\mathcal{S}}}\big[ \hat{\mathfrak{R}}_{\tilde{\mathcal{S}}} (\bold{s}_{\Theta}(t, {\bold{x}}_t)) \big], \forall \Theta \in \Omega_{{\Theta}}$, we  derive
$$
\begin{aligned}
	\mathfrak{R}_{p_{\rm data}} (\bold{s}_{\Theta}(t, {\bold{x}}_t)) - \hat{\mathfrak{R}}_{\mathcal{S}} (\bold{s}_{\Theta}(t, {\bold{x}}_t)) 
	= & \mathbb{E}_{\tilde{\mathcal{S}}}\big[ \hat{\mathfrak{R}}_{\tilde{\mathcal{S}}} (\bold{s}_{\Theta}(t, {\bold{x}}_t))  \big] - \hat{\mathfrak{R}}_{\mathcal{S}} (\bold{s}_{\Theta}(t, {\bold{x}}_t)) \\
	%%%%%%%%%%%%%%%%%%%%%%%%%%%%%%%%%%%%%%%%%%%%%%%%%%%%%%%
	= & \mathbb{E}_{\tilde{\mathcal{S}}} \big[ \hat{\mathfrak{R}}_{\tilde{\mathcal{S}}} (\bold{s}_{\Theta}(t, {\bold{x}}_t))  - \hat{\mathfrak{R}}_{\mathcal{S}} (\bold{s}_{\Theta}(t, {\bold{x}}_t)) \big]. 
\end{aligned}
$$
Taking supremum over $\Theta \in \Omega_{\Theta}$ of both sides and follows the supremum of expectation is smaller than expectation of the supremum,	
$$
\begin{aligned}
	\mathrm{GGap}(\mathcal{S}) \!=\! \sup_{\Theta \in \Omega_{\Theta}} \! \mathbb{E}_{\tilde{\mathcal{S}}}\! \big[ \! \hat{\mathfrak{R}}_{\tilde{\mathcal{S}}} (\bold{s}_{\Theta}(t, {\bold{x}}_t))  \!-\! \hat{\mathfrak{R}}_{\mathcal{S}} (\bold{s}_{\Theta}(t, {\bold{x}}_t)) \big] 
	%%%%%%%%%%%%%%%%%%%%%%%%%%%%%%%%%%%%%%%%%%%%%%%%%
	\!\leq\!  \mathbb{E}_{\tilde{\mathcal{S}}}\! \big[ \!\sup_{\Theta \in \Omega_{\Theta}}\! \big(\hat{\mathfrak{R}}_{\tilde{\mathcal{S}}} (\bold{s}_{\Theta}(t, {\bold{x}}_t)) \! -\! \hat{\mathfrak{R}}_{\mathcal{S}} (\bold{s}_{\Theta}(t, {\bold{x}}_t)) \big)  \big].
\end{aligned}
$$
By further taking expectation over $\mathcal{S}$ on both sides, we obtain
\begin{equation}
	\begin{aligned}
		\mathbb{E}_{\mathcal{S}} [\mathrm{GGap}(\mathcal{S})] 
		\leq &  \mathbb{E}_{\mathcal{S}} \mathbb{E}_{\tilde{\mathcal{S}}} \big[ \sup_{\Theta \in \Omega_{\Theta}}\! \big(\hat{\mathfrak{R}}_{\tilde{\mathcal{S}}} (\bold{s}_{\Theta}(t, {\bold{x}}_t)) \! -\! \hat{\mathfrak{R}}_{\mathcal{S}} (\bold{s}_{\Theta}(t, {\bold{x}}_t)) \big)  \big] \\
		%%%%%%%%%%%%%%%%%%%%%%%%%%%%%%%%%%%%%%%%%%%%%%%%%%%%%
		%%%%%%%%%%%%%%%%%%%%%%%%%%%%%%%%%%%%%%%%%%%%%%%%%%%%%
		= &  \frac{1}{S} \mathbb{E}_{(\mathcal{S}, \tilde{\mathcal{S}})} \! \Big[ \sup_{\Theta \in \Omega_{\Theta}}\!  \sum_{s=1}^{S}  \Big( \mathbb{E}_{\bold{x}_t \sim \boldsymbol{p}(t, \cdot|\tilde{\mathrm{x}}^{(s)})} \left[\big\|\bold{s}_{\Theta}(t, {\bold{x}}_t) -\nabla \ln \boldsymbol{p}(t, \bold{x}_t|\tilde{\mathrm{x}}^{(s)})\big\|_2^2\right]
		\\
		& \qquad \qquad \qquad \ - \mathbb{E}_{\bold{x}_t \sim \boldsymbol{p}(t, \cdot|{\mathrm{x}}^{(s)})} \left[\big\|\bold{s}_{\Theta}(t, {\bold{x}}_t) -\nabla \ln \boldsymbol{p}(t, \bold{x}_t|{\mathrm{x}}^{(s)})\big\|_2^2\right] \Big) \Big],
	\end{aligned}
	\label{proof inequality-2: properties of expection of GGap}
\end{equation}
due to the boundedness of expectation of score error.

We next estimate the right-hand side of (\ref{proof inequality-2: properties of expection of GGap}). For each index $k \in \{1, \ldots, S\}$, observe that
\begin{equation}
	\begin{aligned}
		& \mathbb{E}_{(\mathcal{S}, \tilde{\mathcal{S}})}  \Big[ \sup_{\Theta \in \Omega_{\Theta}}\! \Big\{  \sum_{s=1}^{S}  \Big( \mathbb{E}_{\bold{x}_t \sim \boldsymbol{p}(t, \cdot|\tilde{\mathrm{x}}^{(s)})} \left[\big\|\bold{s}_{\Theta}(t, {\bold{x}}_t) -\nabla \ln \boldsymbol{p}(t, \bold{x}_t|\tilde{\mathrm{x}}^{(s)})\big\|_2^2\right]
		\\
		& \qquad \qquad \qquad  \qquad \quad - \mathbb{E}_{\bold{x}_t \sim \boldsymbol{p}(t, \cdot|{\mathrm{x}}^{(s)})} \left[\big\|\bold{s}_{\Theta}(t, {\bold{x}}_t) -\nabla \ln \boldsymbol{p}(t, \bold{x}_t|{\mathrm{x}}^{(s)})\big\|_2^2\right]
		\Big) \Big\} \Big] \\
		%%%%%%%%%%%%%%%%%%%%%%%%%%%%%%%%%%%%%%%%%%%%%%%%%%%%%%%%
		%%%%%%%%%%%%%%%%%%%%%%%%%%%%%%%%%%%%%%%%%%%%%%%%%%%%%%%%
		= & \mathbb{E}_{(\mathcal{S}, \tilde{\mathcal{S}})}  \Big[ \sup_{\Theta \in \Omega_{\Theta}}\! \Big\{ \Big( \mathbb{E}_{\bold{x}_t \sim \boldsymbol{p}(t, \cdot|\tilde{\mathrm{x}}^{(k)})} \left[\big\|\bold{s}_{\Theta}(t, {\bold{x}}_t) -\nabla \ln \boldsymbol{p}(t, \bold{x}_t|\tilde{\mathrm{x}}^{(k)})\big\|_2^2\right]
		\\
		& \qquad \qquad \qquad  \qquad \quad - \mathbb{E}_{\bold{x}_t \sim \boldsymbol{p}(t, \cdot|{\mathrm{x}}^{(k)})} \left[\big\|\bold{s}_{\Theta}(t, {\bold{x}}_t) -\nabla \ln \boldsymbol{p}(t, \bold{x}_t|{\mathrm{x}}^{(k)})\big\|_2^2\right]
		\Big) \\
		& \qquad \qquad \quad \ + \sum_{s\not = k}\Big( \mathbb{E}_{\bold{x}_t \sim \boldsymbol{p}(t, \cdot|\tilde{\mathrm{x}}^{(s)})} \left[\big\|\bold{s}_{\Theta}(t, {\bold{x}}_t) -\nabla \ln \boldsymbol{p}(t, \bold{x}_t|\tilde{\mathrm{x}}^{(s)})\big\|_2^2\right]
		\\
		& \qquad \qquad \qquad  \qquad \quad - \mathbb{E}_{\bold{x}_t \sim \boldsymbol{p}(t, \cdot|{\mathrm{x}}^{(s)})} \left[\big\|\bold{s}_{\Theta}(t, {\bold{x}}_t) -\nabla \ln \boldsymbol{p}(t, \bold{x}_t|{\mathrm{x}}^{(s)})\big\|_2^2\right]
		\Big) \Big\} \Big] \\
		%%%%%%%%%%%%%%%%%%%%%%%%%%%%%%%%%%%%%%%%%%%%%%%%%%%%%%%%
		%%%%%%%%%%%%%%%%%%%%%%%%%%%%%%%%%%%%%%%%%%%%%%%%%%%%%%%%
		= & \mathbb{E}_{(\mathcal{S}, \tilde{\mathcal{S}})}  \Big[ \sup_{\Theta \in \Omega_{\Theta}}\! \Big\{ \Big( \mathbb{E}_{\bold{x}_t \sim \boldsymbol{p}(t, \cdot|{\mathrm{x}}^{(k)})} \left[\big\|\bold{s}_{\Theta}(t, {\bold{x}}_t) -\nabla \ln \boldsymbol{p}(t, \bold{x}_t|{\mathrm{x}}^{(k)})\big\|_2^2\right]
		\\
		& \qquad \qquad \qquad  \qquad \quad - \mathbb{E}_{\bold{x}_t \sim \boldsymbol{p}(t, \cdot|\tilde{\mathrm{x}}^{(k)})} \left[\big\|\bold{s}_{\Theta}(t, {\bold{x}}_t) -\nabla \ln \boldsymbol{p}(t, \bold{x}_t|\tilde{\mathrm{x}}^{(k)})\big\|_2^2\right]
		\Big) \\
		& \qquad \qquad \quad \ + \sum_{s\not = k}\Big( \mathbb{E}_{\bold{x}_t \sim \boldsymbol{p}(t, \cdot|\tilde{\mathrm{x}}^{(s)})} \left[\big\|\bold{s}_{\Theta}(t, {\bold{x}}_t) -\nabla \ln \boldsymbol{p}(t, \bold{x}_t|\tilde{\mathrm{x}}^{(s)})\big\|_2^2\right]
		\\
		& \qquad \qquad \qquad  \qquad \quad - \mathbb{E}_{\bold{x}_t \sim \boldsymbol{p}(t, \cdot|{\mathrm{x}}^{(s)})} \left[\big\|\bold{s}_{\Theta}(t, {\bold{x}}_t) -\nabla \ln \boldsymbol{p}(t, \bold{x}_t|{\mathrm{x}}^{(s)})\big\|_2^2\right]
		\Big) \Big\} \Big],
	\end{aligned}
\end{equation}
where the second equality is due to ${\mathrm{x}}^{(s)}$ and $\tilde{\mathrm{x}}^{(s)}$ are i.i.d. variables for each $s$. Therefore, 
\begin{equation}
	\begin{aligned}
		& \mathbb{E}_{(\mathcal{S}, \tilde{\mathcal{S}})}  \Big[ \sup_{\Theta \in \Omega_{\Theta}}\! \sum_{s=1}^{S}   \Big( \mathbb{E}_{\bold{x}_t \sim \boldsymbol{p}(t, \cdot|\tilde{\mathrm{x}}^{(s)})} \left[\big\|\bold{s}_{\Theta}(t, {\bold{x}}_t) -\nabla \ln \boldsymbol{p}(t, \bold{x}_t|\tilde{\mathrm{x}}^{(s)})\big\|_2^2\right]
		\\
		& \qquad \qquad \qquad  \qquad \quad - \mathbb{E}_{\bold{x}_t \sim \boldsymbol{p}(t, \cdot|{\mathrm{x}}^{(s)})} \left[\big\|\bold{s}_{\Theta}(t, {\bold{x}}_t) -\nabla \ln \boldsymbol{p}(t, \bold{x}_t|{\mathrm{x}}^{(s)})\big\|_2^2\right]
		\Big) \Big] \\
		%%%%%%%%%%%%%%%%%%%%%%%%%%%%%%%%%%%%%%%%%%%%%%%%%%%%%%%%
		%%%%%%%%%%%%%%%%%%%%%%%%%%%%%%%%%%%%%%%%%%%%%%%%%%%%%%%%
		= & \mathbb{E}_{(\mathcal{S}, \tilde{\mathcal{S}})} \mathbb{E}_{\epsilon_k} \Big[ \sup_{\Theta \in \Omega_{\Theta}}\! \Big\{ \epsilon_k  \Big( \mathbb{E}_{\bold{x}_t \sim \boldsymbol{p}(t, \cdot|\tilde{\mathrm{x}}^{(k)})} \left[\big\|\bold{s}_{\Theta}(t, {\bold{x}}_t) -\nabla \ln \boldsymbol{p}(t, \bold{x}_t|\tilde{\mathrm{x}}^{(k)})\big\|_2^2\right]
		\\
		& \qquad \qquad \qquad  \qquad \quad \ - \mathbb{E}_{\bold{x}_t \sim \boldsymbol{p}(t, \cdot|{\mathrm{x}}^{(k)})} \left[\big\|\bold{s}_{\Theta}(t, {\bold{x}}_t) -\nabla \ln \boldsymbol{p}(t, \bold{x}_t|{\mathrm{x}}^{(k)})\big\|_2^2\right]
		\Big) \\
		& \qquad \qquad \qquad \qquad \quad \ + \sum_{s\not = k}\Big( \mathbb{E}_{\bold{x}_t \sim \boldsymbol{p}(t, \cdot|\tilde{\mathrm{x}}^{(s)})} \left[\big\|\bold{s}_{\Theta}(t, {\bold{x}}_t) -\nabla \ln \boldsymbol{p}(t, \bold{x}_t|\tilde{\mathrm{x}}^{(s)})\big\|_2^2\right]
		\\
		& \qquad \qquad \qquad  \qquad \quad \ - \mathbb{E}_{\bold{x}_t \sim \boldsymbol{p}(t, \cdot|{\mathrm{x}}^{(s)})} \left[\big\|\bold{s}_{\Theta}(t, {\bold{x}}_t) -\nabla \ln \boldsymbol{p}(t, \bold{x}_t|{\mathrm{x}}^{(s)})\big\|_2^2\right]
		\Big) \Big\} \Big] ,
	\end{aligned}
\end{equation}
where $\epsilon_k$ is a random variable such that $\mathbb{P}[\epsilon_k=1] = \mathbb{P}[\epsilon_k=-1] = 1/2$. Denote $\boldsymbol{\epsilon}= \{\epsilon_1, \ldots, \epsilon_S \}$.
Repeating the above equality for all $k$ and combing (\ref{proof inequality-2: properties of expection of GGap}) we obtain that
\begin{equation}
	\begin{aligned}
		&\mathbb{E}_{\mathcal{S}} [\mathrm{GGap}(\mathcal{S})] \\
		\leq &  \frac{1}{S} \mathbb{E}_{(\mathcal{S}, \tilde{\mathcal{S}})} \mathbb{E}_{\boldsymbol{\epsilon}}  \Big[ \sup_{\Theta \in \Omega_{\Theta}}\! \Big\{ \sum_{s=1}^S \epsilon_s  \Big( \mathbb{E}_{\bold{x}_t \sim \boldsymbol{p}(t, \cdot|\tilde{\mathrm{x}}^{(s)})} \left[\big\|\bold{s}_{\Theta}(t, {\bold{x}}_t) -\nabla \ln \boldsymbol{p}(t, \bold{x}_t|\tilde{\mathrm{x}}^{(s)})\big\|_2^2\right]
		\\
		& \qquad \qquad \qquad  \qquad \qquad \ - \mathbb{E}_{\bold{x}_t \sim \boldsymbol{p}(t, \cdot|{\mathrm{x}}^{(s)})} \left[\big\|\bold{s}_{\Theta}(t, {\bold{x}}_t) -\nabla \ln \boldsymbol{p}(t, \bold{x}_t|{\mathrm{x}}^{(s)})\big\|_2^2\right]
		\Big) \Big\} \Big] \\
		%%%%%%%%%%%%%%%%%%%%%%%%%%%%%%%%%%%%%%%%%%%%%%%%%%%%%%%%%%%
		%%%%%%%%%%%%%%%%%%%%%%%%%%%%%%%%%%%%%%%%%%%%%%%%%%%%%%%%%%%
		\leq & \frac{1}{S} \mathbb{E}_{(\mathcal{S}, \tilde{\mathcal{S}})} \mathbb{E}_{\boldsymbol{\epsilon}} \Big[ \sup_{\Theta \in \Omega_{\Theta}}\! \Big\{ \sum_{s=1}^S \epsilon_s  \mathbb{E}_{\bold{x}_t \sim \boldsymbol{p}(t, \cdot|\tilde{\mathrm{x}}^{(s)})} \left[\big\|\bold{s}_{\Theta}(t, {\bold{x}}_t) -\nabla \ln \boldsymbol{p}(t, \bold{x}_t|\tilde{\mathrm{x}}^{(s)})\big\|_2^2\right] \Big\}
		\\
		& \qquad \qquad  + \sup_{\Theta \in \Omega_{\Theta}}\! \Big\{  \sum_{s=1}^S -\epsilon_s \mathbb{E}_{\bold{x}_t \sim \boldsymbol{p}(t, \cdot|{\mathrm{x}}^{(s)})} \left[\big\|\bold{s}_{\Theta}(t, {\bold{x}}_t) -\nabla \ln \boldsymbol{p}(t, \bold{x}_t|{\mathrm{x}}^{(s)})\big\|_2^2\right]
		\Big\} \Big] \\
		%%%%%%%%%%%%%%%%%%%%%%%%%%%%%%%%%%%%%%%%%%%%%%%%%%%%%%%%%%%
		%%%%%%%%%%%%%%%%%%%%%%%%%%%%%%%%%%%%%%%%%%%%%%%%%%%%%%%%%%%
		=  &  \frac{1}{S} \mathbb{E}_{(\mathcal{S}, \tilde{\mathcal{S}})} \mathbb{E}_{\boldsymbol{\epsilon}} \Big[ \sup_{\Theta \in \Omega_{\Theta}}\! \Big\{ \sum_{s=1}^S \epsilon_s  \mathbb{E}_{\bold{x}_t \sim \boldsymbol{p}(t, \cdot|\tilde{\mathrm{x}}^{(s)})} \left[\big\|\bold{s}_{\Theta}(t, {\bold{x}}_t) -\nabla \ln \boldsymbol{p}(t, \bold{x}_t|\tilde{\mathrm{x}}^{(s)})\big\|_2^2\right] \Big\}
		\\
		& \qquad \qquad  + \sup_{\Theta \in \Omega_{\Theta}}\! \Big\{  \sum_{s=1}^S \epsilon_s \mathbb{E}_{\bold{x}_t \sim \boldsymbol{p}(t, \cdot|{\mathrm{x}}^{(s)})} \left[\big\|\bold{s}_{\Theta}(t, {\bold{x}}_t) -\nabla \ln \boldsymbol{p}(t, \bold{x}_t|{\mathrm{x}}^{(s)})\big\|_2^2\right]
		\Big\} \Big] \\
		%%%%%%%%%%%%%%%%%%%%%%%%%%%%%%%%%%%%%%%%%%%%%%%%%%%%%%%%%%%
		%%%%%%%%%%%%%%%%%%%%%%%%%%%%%%%%%%%%%%%%%%%%%%%%%%%%%%%%%%%
		=  & 2 \mathbb{E}_{\mathcal{S}} [ \mathscr{R}_{\mathcal{S}}( \mathcal{F}_t)],
	\end{aligned}
	\label{proof inequality-2: estimation of expection of GGap}
\end{equation}
where $\mathscr{R}_{\mathcal{S}}(\mathcal{F}_t)=\frac{1}{S}\mathbb{E}_{\boldsymbol{\epsilon}} \Big[ \sup_{\Theta \in \Omega_{\Theta}} \Big\{  \sum_{s=1}^S \epsilon_s \mathbb{E}_{\bold{x}_t \sim \boldsymbol{p}(t, \cdot|{\mathrm{x}}^{(s)})} \left[\big\|\bold{s}_{\Theta}(t, {\bold{x}}_t) -\nabla \ln \boldsymbol{p}(t, \bold{x}_t|{\mathrm{x}}^{(s)})\big\|_2^2\right]
\Big\} \Big]$ is the Rademacher complexity of the hypothesis set 
\begin{equation}
	\mathcal{F}_t  = \Big\{  {f}_{\Theta, t}(\cdot) | \ {f}_{\Theta, t}(\mathrm{x}) = \mathbb{E}_{\bold{x}_t \sim \boldsymbol{p}(t, \cdot|\mathrm{x})} \left[\big\|\bold{s}_{\Theta}(t, {\bold{x}}_t) -\nabla \ln \boldsymbol{p}(t, \bold{x}_t|\mathrm{x})\big\|_2^2\right], \ \Theta \in \Omega_{\Theta}  \Big\}, 
	\label{equation: definition of hypothesis-set}
\end{equation} 
for samples in $\mathcal{S}$.

A similar discussion as in (\ref{proof inequality-2: property of change of a single coordinate}) shows that the random variable $\mathscr{R}_{\mathcal{S}}( \mathcal{F}_t)$ also satisfies the bounded difference condition of McDiarmid's inequality \cite[Lemma~26.4]{shalev2014understanding} with a constant $\frac{2}{S}B_{\rm score\_err}$, hence with probability of at least $1-\delta/2$, 
\begin{equation}
	\begin{aligned}
		\mathbb{E}_{\mathcal{S}} [ \mathscr{R}_{\mathcal{S}}( \mathcal{F}_t)] \leq  \mathscr{R}_{\mathcal{S}}( \mathcal{F}_t) + B_{\rm score\_err} \sqrt{\frac{2 \ln (4/\delta)}{S}}.
	\end{aligned}
	\label{proof inequality-2: McDiarmid's inequality-2}
\end{equation}
Combining (\ref{proof inequality-2: McDiarmid's inequality-1}), (\ref{proof inequality-2: estimation of expection of GGap}), (\ref{proof inequality-2: McDiarmid's inequality-2}), we get, with probability of at least $1-\delta$, 
\begin{equation}
	\begin{aligned}
		\mathrm{GGap}(\mathcal{S}) \leq & 2\mathscr{R}_{\mathcal{S}}( \mathcal{F}_t)	 + 3 B_{\rm score\_err} \sqrt{\frac{2 \ln (4/\delta)}{S}}.
	\end{aligned}
	\label{proof inequality-2: final estimate of GGap}
\end{equation}

Let $\overline{\mathrm{GGap}}(\mathcal{S}) := \sup_{\Theta \in \Omega_{\Theta}} \Big ( \hat{\mathfrak{R}}_{\mathcal{S}} (\bold{s}_{\Theta}(t, {\bold{x}}_t)) -  \mathfrak{R}_{p_{\rm data}} (\bold{s}_{\Theta}(t, {\bold{x}}_t)) \Big)$. A similar discussion as above yields that, with probability of at least $1-\delta$,
\begin{equation}
	\begin{aligned}
		\overline{\mathrm{GGap}}(\mathcal{S}) \leq & 2\mathscr{R}_{\mathcal{S}}( \mathcal{F}_t)	 + 3 B_{\rm score\_err} \sqrt{\frac{2 \ln (4/\delta)}{S}}.
	\end{aligned}
\end{equation}
Combining this equality with (\ref{proof inequality-2: final estimate of GGap}) gives that, with probability of at least $1-\delta$, 
\begin{equation}
	\begin{aligned}
		\big|\mathfrak{R}_{p_{\rm data}} (\bold{s}_{\Theta}(t, {\bold{x}}_t)) - \hat{\mathfrak{R}}_{\mathcal{S}} (\bold{s}_{\Theta}(t, {\bold{x}}_t)) \big| \leq & 2\mathscr{R}_{\mathcal{S}}(\mathcal{F}_t)	 + 3 B_{\rm score\_err} \sqrt{\frac{2 \ln (4/\delta)}{S}}, \forall  \Theta \in \Omega_{\Theta}.
	\end{aligned}
	\label{equation-proof: RC bound}
\end{equation}

%%%%%%%%%%%%%%%%%%%%%%%%%%%%%%%%%%%%%%%%%%%%%%%%%%%%%%%%%%%%%
\textbf{Step~2}.  We estimate $\mathscr{R}_{\mathcal{S}}(\mathcal{F}_t)$. For convenience, denote
$$ g_{\Theta} =  \sum_{s=1}^S \epsilon_s  \mathbb{E}_{\bold{x}_t \sim \boldsymbol{p}(t, \cdot|\mathrm{x}^{(s)})} \! \left[\big\|\bold{s}_{\Theta}(t, {\bold{x}}_t) \!- \! \nabla \ln \boldsymbol{p}\big(t, \bold{x}_t|\mathrm{x}^{(s)}\big)\big\|_2^2\right],
$$
where $(\epsilon_s, \ldots, \epsilon_S)$ is a Rademacher sequence.
We see that $\{g_{\Theta}\}_{\Theta}$ is a Rademacher process which is centered (i.e., $\mathbb{E}_{\boldsymbol{\epsilon}} [g_{\Theta}] = 0$) and is a subgaussian process \cite[Sec.~8.6]{foucart2013}. Since
$$
\begin{aligned}
	\mathbb{E}_{\boldsymbol{\epsilon}} \big[ | g_{\Theta} - g_{\tilde{\Theta}} |^2 \big] 
	= & \mathbb{E}_{\boldsymbol{\epsilon}} \Big[ \Big| 
	\sum_{s=1}^S \epsilon_s \Big( {f}_{\Theta, t}(\mathrm{x}^{(s)}) - {f}_{\tilde{\Theta}, t}(\mathrm{x}^{(s)}) \Big) \Big|^2 \Big]  \\
	%%%%%%%%%%%%%%%%%%%%%%%%%%%%%%%%%%%%%%%%%%%%%%%%%
	= &  \sum_{s=1}^S   \Big( {f}_{\Theta, t}(\mathrm{x}^{(s)}) - {f}_{\tilde{\Theta}, t}(\mathrm{x}^{(s)}) \Big) ^2, \ \forall \Theta, \tilde{\Theta} \in \Omega_{\Theta},
\end{aligned}
$$
we consider a pseudo-metric associated to process $\{g_{\Theta}\}_{\Theta}$ as 
$$
d(\Theta, \tilde{\Theta}) = \big( \mathbb{E}_{\boldsymbol{\epsilon}}  | g_{\Theta} - g_{\tilde{\Theta}} |^2 \big)^{1/2} = \Big( \sum_{s=1}^S   \Big( {f}_{\Theta, t}(\mathrm{x}^{(s)}) - {f}_{\tilde{\Theta}, t}(\mathrm{x}^{(s)}) \Big) ^2  \Big) ^{1/2}.
$$
The diameter of $\Omega_{\Theta}$ with respect to $d(\cdot, \cdot)$ is
$$
\begin{aligned}
	& \Delta (\Omega_{\Theta})  =  \sup_{\Theta, \tilde{\Theta} \in \Omega_{\Theta} } d(\Theta, \tilde{\Theta}) \\
	%= & \sup_{\Theta, \tilde{\Theta} \in \Omega_{\Theta} } \Big\{  \sum_{s=1}^S \Big( \mathbb{E}_{\bold{x}_t \sim \boldsymbol{p}(t, \cdot|\mathrm{x}^{(s)})} \left[\big\|\bold{s}_{\Theta}(t, {\bold{x}}_t) -\nabla \ln \boldsymbol{p}(t, \bold{x}_t|\mathrm{x}^{(s)})\big\|_2^2\right] \\
	% & \qquad \qquad \qquad \quad - \mathbb{E}_{\bold{x}_t \sim \boldsymbol{p}(t, \cdot|\mathrm{x}^{(s)})} \left[\big\|\bold{s}_{\tilde{\Theta}}(t, {\bold{x}}_t) -\nabla \ln \boldsymbol{p}(t, \bold{x}_t|\mathrm{x}^{(s)})\big\|_2^2\right] \Big)^2 \Big\} \\
	\leq &  \sup_{\Theta \in \Omega_{\Theta} } \Big[  \sum_{s=1}^S \Big( \mathbb{E}_{\bold{x}_t \sim \boldsymbol{p}(t, \cdot|\mathrm{x}^{(s)})} \left[\big\|\bold{s}_{\Theta}(t, {\bold{x}}_t) -\nabla \ln \boldsymbol{p}(t, \bold{x}_t|\mathrm{x}^{(s)})\big\|_2^2\right] \Big)^2
	\Big]^{1/2} 
	\leq \sqrt{S} B_{\rm score\_err}
\end{aligned}
$$
owing to (\ref{proof inequality: estimation of difference of scores}).
Hence, we can apply the Dudley’s inequality \cite[Theorem~8.23]{foucart2013} for this subgaussian processes, which gives
\begin{equation}
	\begin{aligned}
		\mathscr{R}_{\mathcal{S}}(\mathcal{F}_t) = & \frac{1}{S} \mathbb{E}_{\boldsymbol{\epsilon}} \left[ \sup_{\Theta \in \Omega} \sum_{s=1}^S \epsilon_s  \mathbb{E}_{\bold{x}_t \sim \boldsymbol{p}(t, \cdot|\mathrm{x}^{(s)})} \left[\big\|\bold{s}_{\Theta}(t, {\bold{x}}_t) -\nabla \ln \boldsymbol{p}(t, \bold{x}_t|\mathrm{x}^{(s)})\big\|_2^2\right] \right] \\
		\leq &   \frac{4\sqrt{2}}{S} \int_{0}^{\Delta (\Omega_{\Theta}) /2} \sqrt{\ln \big(  N(\Omega_{{\Theta}}, d, u) \big)} \mathrm{d}u \\
		\leq &   \frac{4\sqrt{2}}{S} \int_{0}^{\sqrt{S} B_{\rm score\_err} /2 } \sqrt{\ln \big(  N(\Omega_{{\Theta}}, d, u) \big)} \mathrm{d}u,
	\end{aligned}
	\label{proof-ineq: Dudley's inequality}
\end{equation}
where $N(\Omega_{{\Theta}}, d, u)$ is the covering number \cite[Sec.~C.2]{foucart2013} of set $\Omega_{{\Theta}}$ with respect to $d(\cdot, \cdot)$.

To derive the bound for $\mathscr{R}_{\mathcal{S}}(\mathcal{F}_t)$, it suffices to estimate $N(\Omega_{{\Theta}}, d, u)$. Noting that 
$$
\begin{aligned}
	d ^2 (\Theta, \tilde{\Theta}) 
	= &  \sum_{s=1}^S \Big(   \mathbb{E}_{\bold{x}_t \sim \boldsymbol{p}(t, \cdot|\mathrm{x}^{(s)})} \left[\big\|\bold{s}_{\Theta}(t, {\bold{x}}_t) -\nabla \ln \boldsymbol{p}(t, \bold{x}_t|\mathrm{x}^{(s)})\big\|_2^2\right] \\
	& \qquad \qquad \qquad \qquad   - \mathbb{E}_{\bold{x}_t \sim \boldsymbol{p}(t, \cdot|\mathrm{x}^{(s)})} \left[\big\|\bold{s}_{\tilde{\Theta}}(t, {\bold{x}}_t) -\nabla \ln \boldsymbol{p}(t, \bold{x}_t|\mathrm{x}^{(s)})\big\|_2^2\right] \Big)^2  \\
	\leq &  \sum_{s=1}^S \Big( \mathbb{E} \big[\big\|\bold{s}_{\Theta}(t, {\bold{x}}_t) - \bold{s}_{\tilde{\Theta}}(t, {\bold{x}}_t) \big\|_2 \times
	\big(\big\|\bold{s}_{\Theta}(t, {\bold{x}}_t) -\nabla \ln \boldsymbol{p}(t, \bold{x}_t|\mathrm{x}^{(s)})\big\|_2\\
	&  \qquad \qquad   \qquad \qquad   \qquad \qquad  \qquad \quad   + \big\|\bold{s}_{\tilde{\Theta}}(t, {\bold{x}}_t) -\nabla \ln \boldsymbol{p}(t, \bold{x}_t|\mathrm{x}^{(s)})\big\|_2 \big)
	\big] \Big)^2\\
	\leq &  \sum_{s=1}^S \mathbb{E} \big[\big\|\bold{s}_{\Theta}(t, {\bold{x}}_t) \! - \! \bold{s}_{\tilde{\Theta}}(t, {\bold{x}}_t) \big\|_2^2 \big] \! \times \! 2 \big(\mathbb{E}\big[\big\|\bold{s}_{\Theta}(t, {\bold{x}}_t) -\nabla \ln \boldsymbol{p}(t, \bold{x}_t|\mathrm{x}^{(s)})\big\|_2^2\big] \\
	& \qquad \qquad \qquad \qquad  \qquad \qquad \qquad  \quad   + \mathbb{E} \big[ \big\|\bold{s}_{\tilde{\Theta}}(t, {\bold{x}}_t) -\nabla \ln \boldsymbol{p}(t, \bold{x}_t|\mathrm{x}^{(s)})\big\|_2^2 \big] \big)  \\
	\leq & \sum_{s=1}^S 4 B_{\rm score\_err} C_{\bold{s},2}  \|\Theta - \tilde{\Theta} \|_{\mathcal{E}^L}^2 
	=  4 S B_{\rm score\_err} C_{\bold{s},2} \|\Theta - \tilde{\Theta} \|_{\mathcal{E}^L}^2, \ \forall \Theta, \tilde{\Theta} \in \Omega_{\Theta},
\end{aligned}
$$
where the third inequality is due to H\"older inequality, the forth inequality is due to Lemma~\ref{lemma: property of DNN} and Lemma~\ref{lemma: proeprty of dnn state}. 
%$\Theta \in \Omega_{\Theta} \subset (\mathbb{R}^{n \times n_{\boldsymbol{e}}} \times \mathbb{R}^{n \times n_{\rm d}} \times \mathbb{R}^{n_{\rm d} \times n})^L$,
Applying the above estimate together with \cite[Proposition~C.3]{foucart2013}, and noting that the dimension of $\Theta$ is $\mathrm{dim}(\Theta) = Ln(2n_{\rm d}+n_{\boldsymbol{e}})$, we obtain
$$
\begin{aligned}
	N(\Omega_{{\Theta}}, d, u) \leq &  N \Big(\Omega_{{\Theta}}, \|\cdot\|_{\mathcal{E}^L}, {u}/{\sqrt{4SB_{\rm score\_err} C_{\bold{s},2} }} \Big) \\
	\leq &  \big(1+4B_{\Theta}\sqrt{SB_{\rm score\_err} C_{\bold{s},2} }/u \big)^{Ln(2n_{\rm d}+n_{\boldsymbol{e}})}.
\end{aligned}
$$

Therefore, invoking (\ref{proof-ineq: Dudley's inequality}) yields
\begin{equation}
	\begin{aligned}
		\mathscr{R}_{\mathcal{S}}(\mathcal{F}_t) 
		%	\leq &   \frac{4\sqrt{2}}{S} \int_{0}^{\sqrt{S} B_{\rm score\_err} /2} \sqrt{\ln \big(  N\big(\Omega_{{\Theta}}, \|\cdot\|_{\mathcal{E}^L}, {u}/{ \sqrt{4SB_{\rm score\_err} C_{\bold{s},2} }} \big) \big)} \mathrm{d}u \\
		\leq & \frac{4\sqrt{2}}{S} \int_{0}^{\sqrt{S} B_{\rm score\_err} /2 } \sqrt{Ln(2n_{\rm d}+n_{\boldsymbol{e}}) \ln(1+4B_{\Theta}\sqrt{SB_{\rm score\_err} C_{\bold{s},2} }/u)} \mathrm{d}u \\
		%	\leq & \frac{2\sqrt{2Ln(2n_{\rm d}+n_{\boldsymbol{e}})}}{S} \cdot \sqrt{S} B_{\rm score\_err} \Psi(8B_{\Theta}\sqrt{SB_{\rm score\_err} C_{\bold{s},2} } / \sqrt{S} B_{\rm score\_err} ) \\
		\leq & \frac{2\sqrt{2Ln(2n_{\rm d}+n_{\boldsymbol{e}})}}{\sqrt{S}} \cdot B_{\rm score\_err} \Psi(8B_{\Theta}\sqrt{ C_{\bold{s},2} /B_{\rm score\_err}}),
	\end{aligned}
	\nonumber
\end{equation}
where the second inequality is due to \cite[Lemma~5.4]{schnoor2023generalization} and function $\Psi$ is defined by $\Psi(t)= \sqrt{\ln(1+t) \! + \!  t \big(\ln(1+t) \! - \!  \ln(t)\big)}$.
Substituting this estimate into (\ref{equation-proof: RC bound}) completes the proof.
\end{proof}

We are ready to prove Theorem~\ref{theorem: Convergence of learning problems}.
\begin{proof}[proof of Theorem~\ref{theorem: Convergence of learning problems}]

(i)  Recall the definitions of $\ell^{\rm DSM}, \ell_{N,S}^{\rm DSM}, \ell_{S}^{\rm DSM}$ in (\ref{loss: continuous-infinite-score-loss-function}), (\ref{loss: discrete-time-finite-sample loss}), (\ref{loss: continuous-time-finite-sample loss}), and of $\hat{\mathfrak{R}}_{\mathcal{S}}, \mathfrak{R}_{p_{\rm data}}$ in (\ref{equation: definition of training loss}). We decompose the error
\begin{equation}
	\begin{aligned}
		& \big| \ell^{\rm DSM}_{N,S}(\Theta) - \ell^{\rm DSM} (\Theta) \big| \\
		\leq &   \big| \ell^{\rm DSM}_{N,S}(\Theta) - \ell_S^{\rm DSM} (\Theta) \big| + \big| \ell^{\rm DSM}_{S}(\Theta) - \ell^{\rm DSM} (\Theta) \big| \\
		\leq & \big| \ell^{\rm DSM}_{N,S}(\Theta) - \ell_S^{\rm DSM} (\Theta) \big| +  \frac{1}{T} \int_{0}^{T} | \pmb{\lambda}(t) | \cdot |	\hat{\mathfrak{R}}_{\mathcal{S}} (\bold{s}_{\Theta}(t, {\bold{x}}_t)) -  \mathfrak{R}_{p_{\rm data}} (\bold{s}_{\Theta}(t, {\bold{x}}_t))|  \mathrm{d} t.
	\end{aligned}
	\nonumber
\end{equation}

By Lemma~\ref{lemma: Convergence from {P}_{N,S} to {P}_{S}},
there exist constants $C_1, C_2, C_3$ defined in (\ref{coefficients: C_1_2_3}) such that
\begin{equation}
	\big|\ell^{\rm DSM}_{N,S}(\Theta) - \ell_{S}^{\rm DSM}(\Theta) \big| \leq B_{\rm score\_err}\omega_{\pmb{\lambda}}(h_N) +   B_{\pmb{\lambda}} \sqrt{6B_{\rm score\_err}} \Big[ C_1  \omega_{\pmb{\gamma}}(h_N)  +  C_2 h_N^{1/2} + C_3 h_N \Big],
	\nonumber
\end{equation}
Furthermore, by Lemma~\ref{lemma: generalization error of score functions}, we have for any $\delta \in (0,1)$, with probability at least $1-\delta$, that
\begin{equation}
	\begin{aligned}
		\big | \hat{\mathfrak{R}}_{\mathcal{S}} (\bold{s}_{\Theta}(t, {\bold{x}}_t)) - \mathfrak{R}_{p_{\rm data}} (\bold{s}_{\Theta}(t, {\bold{x}}_t)) \big| \leq &  \frac{C_4}{\sqrt{S}} + 3B_{\rm score\_err}\sqrt{\frac{2 \ln (4/\delta)}{S}}, \ \forall \Theta \in \Omega_{\Theta}.
	\end{aligned}
	\nonumber
\end{equation}	
It then follows by combining the above three inequalities that (\ref{equation: difference of losses}) holds.

%%%%%%%%%%%%%%%%%%%%%%%%%%%%%%%%%%%%%%%%%%%%%%%%%%%%%%%%%%%%%%%%
%%%%%%%%%%%%%%%%%%%%%%%%%%%%%%%%%%%%%%%%%%%%%%%%%%%%%%%%%%%%%%%%
(ii) We now prove the convergence of optimal values.
By (\ref{equation: difference of losses}), for every $\epsilon > 0$,
\begin{equation}
	\mathbb{P} \left( \sup_{\Theta \in \Omega_{\Theta}} |\ell^{\rm DSM}_{N,S}(\Theta) - \ell^{\rm DSM}(\Theta)| > \epsilon \right) \to 0 \quad \text{as } (N,S) \to (\infty,\infty).
	\label{proof inequality-4: probability inequality}
\end{equation}
Let $e_{N,S}:=\sup _{\Theta}\left|\ell^{\rm DSM}_{N,S}(\Theta)-\ell^{\rm DSM}(\Theta)\right|$. Then for all $\Theta \in \Omega_{\Theta}$, 
$$
\ell^{\rm DSM}_{N,S}(\Theta)-e_{N,S} \leq \ell^{\rm DSM}(\Theta) \leq \ell^{\rm DSM}_{N,S}(\Theta) + e_{N,S}.
$$
Taking the infimum over $\Theta$ gives
$
\big|\ell^*_{N,S} - \ell^*\big| \leq e_{N,S}.
$
Together with (\ref{proof inequality-4: probability inequality}), this shows $\ell^*_{N,S}$ converges in probability to $\ell^*$  as $(N,S)\to(\infty,\infty)$.	

Next we prove the convergence of optimal solutions. Let $\Theta^*_{N,S} \in \mathcal{O}^*_{N,S}$. For any $ \Theta^* \in \mathcal{O}^*$, we have 
$$
\ell^{\rm DSM}(\Theta^*_{N,S}) - e_{N,S} \leq  \ell^{\rm DSM}_{N,S}(\Theta^*_{N,S}) 
\leq \ell^{\rm DSM}_{N,S}(\Theta^*) \leq \ell^{*} + e_{N,S},	
$$
which implies 
\begin{equation}
	\ell^{\rm DSM}(\Theta^*_{N,S}) - \ell^{*} \leq 2 e_{N,S} .
	\label{proof inequality-4-1: probability inequality}
\end{equation}

Consider the event $A_{\epsilon}:=\{\mathrm{dist}(\Theta^*_{N,S}, \mathcal{O}^*) \ge \epsilon \}$. 
If $A_{\epsilon}$ occurs, then $\Theta^*_{N,S} \in \mathcal{K}_{\epsilon} := \{ \Theta \in \Omega_{{\Theta}}: \mathrm{dist}(\Theta, \mathcal{O}^*) \ge \epsilon \}$. By Lemma~\ref{lemma: Existence of solutions for learning problems} and the continuity of $\ell^{\rm DSM}$, $\mathcal{O}^*$ is nonempty and closed, and since $\mathcal{K}_{\epsilon}$ is a compact set and $\mathcal{K}_{\epsilon} \cap \mathcal{O}^* = \emptyset$, the gap $\delta_{\epsilon}:= \inf_{\Theta \in \mathcal{K}_{\epsilon}} \ell^{\rm DSM}(\Theta) - \ell^* >0$.
Thus, on event $A_{\epsilon}$,  $\ell^{\rm DSM}(\Theta^*_{N,S}) - \ell^* \ge \delta_{\epsilon}$.
Combing this with (\ref{proof inequality-4-1: probability inequality}) gives 
$$A_{\epsilon} \subset \{ 2 e_{N,S} \ge \delta_{\epsilon} \}.$$
Hence, $P\big(A_{\epsilon} \big) \leq P\big( 2 e_{N,S} \ge \delta_{\epsilon} \big) \to 0$ as $(N,S) \to (\infty, \infty)$, and $\mathrm{dist}(\Theta^*_{N,S}, \mathcal{O}^*)$ converges in probability to $0$. 

%%%%%%%%%%%%%%%%%%%%%%%%%%%%%%%%%%%%%%%%%%%%%%%%%%%%%%%%%%%%%%%%
%%%%%%%%%%%%%%%%%%%%%%%%%%%%%%%%%%%%%%%%%%%%%%%%%%%%%%%%%%%%%%%%
(iii) 
Set $\delta_k := 1/k^2$. Since $\sum_k \delta_k < \infty$, we can choose a subsequence $(N_k, S_k) \to (\infty, \infty)$ such that
\[
\mathbb{P} \left( \sup_{\Theta \in \Omega_{\Theta} } |\ell^{\rm DSM}_{N_k, S_k}(\Theta) - \ell^{\rm DSM}(\Theta)| > 1/k \right) < \frac{1}{k^2}.
\]
Hence,
$
\sum_{k=1}^\infty \mathbb{P} \left( \sup_{\Theta \in \Omega_{\Theta}} |\ell^{\rm DSM}_{N_k, S_k}(\Theta) - \ell^{\rm DSM}(\Theta)| > 1/k \right) < \infty.
$
Applying the Borel–Cantelli lemma \cite[Theorem~4.2.1]{chung2000course} then gives
\begin{equation}
	\sup_{\Theta \in \Omega_{\Theta}} |\ell^{\rm DSM}_{N_k, S_k}(\Theta) - \ell^{\rm DSM}(\Theta)| \to 0, \text{as } k \to \infty, \ \text{almost surely}.
	\label{proof inequality-5: a.s. convergence}
\end{equation}

By assumption (A5), $\Omega_{\Theta}$ is compact. Hence, the sequence $\{\Theta^*_{N_k, S_k}\}_k$ has a convergent subsequence (which we do not relabel). Let $\Theta^* \in \Omega_{\Theta}$ be a cluster point.
From (\ref{proof inequality-5: a.s. convergence}), for any $\epsilon > 0$, there exists a constant $K_1$ such that for all $k > K_1$,
\[
\sup_{\Theta \in \Omega_{\Theta}} |\ell_{N_k, S_k}^{\rm DSM}(\Theta) - \ell^{\rm DSM}(\Theta)| < \epsilon/3, \ \text{almost surely}.
\]
Moreover, using Lemma~\ref{lemma: property of DNN} and Lemma~\ref{lemma: proeprty of dnn state}, along with inequality \eqref{proof inequality: estimation of difference of scores}, we derive that for any $\Theta, \tilde{\Theta} \in \Omega_{\Theta}$ and $t \in [0,T]$,
\begin{equation}
	\begin{aligned}
		& \Big| \hat{\mathfrak{R}}_{\mathcal{S}} (\bold{s}_{\tilde{\Theta}}(t, {\bold{x}}_t)) - \hat{\mathfrak{R}}_{\mathcal{S}} (\bold{s}_{\Theta}(t, {\bold{x}}_t))\Big| \\
		%%%%%%%%%%%%%%%%%%%%%%%%%%%%%%%%%%%%%%%%%%%%%%%%%%%%%%%%
		%%%%%%%%%%%%%%%%%%%%%%%%%%%%%%%%%%%%%%%%%%%%%%%%%%%%%%%%
		% \leq &  \Big| \frac{1}{S} \sum_{s=1}^{S} \mathbb{E}_{\bold{x}_t \sim \boldsymbol{p}(t, \cdot|\mathrm{x}^{(s)})} \! \left[\big\|\bold{s}_{\hat{\Theta}}(t, {\bold{x}}_t) \!-\! \nabla \ln \boldsymbol{p}(t, \bold{x}_t|\mathrm{x}^{(s)})\big\|_2^2 \!-\! \big\|\bold{s}_{\Theta}(t, {\bold{x}}_t) -\nabla \ln \boldsymbol{p}(t, \bold{x}_t|\mathrm{x}^{(s)})\big\|_2^2\right]  \Big| \\
		%%%%%%%%%%%%%%%%%%%%%%%%%%%%%%%%%%%%%%%%%%%%%%%%%%%%%%%%
		%%%%%%%%%%%%%%%%%%%%%%%%%%%%%%%%%%%%%%%%%%%%%%%%%%%%%%%%
		\leq &  \frac{1}{S} \sum_{s=1}^{S} \mathbb{E}_{\bold{x}_t \sim \boldsymbol{p}(t, \cdot|\mathrm{x}^{(s)})} \! \Big[\big\|\bold{s}_{\tilde{\Theta}}(t, {\bold{x}}_t) \!-\! \bold{s}_{\Theta}(t, {\bold{x}}_t)\big\|_2 \cdot   \Big(\big\|\bold{s}_{\tilde{\Theta}}(t, {\bold{x}}_t) -\nabla \ln \boldsymbol{p}(t, \bold{x}_t|\mathrm{x}^{(s)})\big\|_2   \\ 
		& \qquad \qquad \qquad \qquad \qquad \qquad \qquad \qquad \qquad \qquad  + \big\|\bold{s}_{\Theta}(t, {\bold{x}}_t) -\nabla \ln \boldsymbol{p}(t, \bold{x}_t|\mathrm{x}^{(s)})\big\|_2 \Big) \Big]   \\
		%%%%%%%%%%%%%%%%%%%%%%%%%%%%%%%%%%%%%%%%%%%%%%%%%%%%%%%%
		%%%%%%%%%%%%%%%%%%%%%%%%%%%%%%%%%%%%%%%%%%%%%%%%%%%%%%%%
		\leq &    \frac{\sqrt{2}}{S} \sum_{s=1}^{S} \Big(\mathbb{E}_{\bold{x}_t \sim \boldsymbol{p}(t, \cdot|\mathrm{x}^{(s)})} \! \big[C_{\bold{s}}^2(\bold{x}_t) \big\|\tilde{\Theta} - \Theta\big\|^2_{\mathcal{E}^L} \big] \Big)^{1/2}  \\ 
		& \qquad  \times \! \Big( \mathbb{E}_{\bold{x}_t \sim \boldsymbol{p}(t, \cdot|\mathrm{x}^{(s)})} \! \Big[ \big\|\bold{s}_{\tilde{\Theta}}(t, {\bold{x}}_t) \! - \! \nabla \ln \boldsymbol{p}(t, \bold{x}_t|\mathrm{x}^{(s)})\big\|^2_2 \! + \! \big\|\bold{s}_{\Theta}(t, {\bold{x}}_t) \!- \! \nabla \ln \boldsymbol{p}(t, \bold{x}_t|\mathrm{x}^{(s)})\big\|^2_2  \Big]  \Big)^{1/2} \\
		%%%%%%%%%%%%%%%%%%%%%%%%%%%%%%%%%%%%%%%%%%%%%%%%%%%%%%%%
		%%%%%%%%%%%%%%%%%%%%%%%%%%%%%%%%%%%%%%%%%%%%%%%%%%%%%%%%
		\leq &   \sqrt{2C_{\bold{s},2} B_{\rm score\_err}} \big\| \tilde{\Theta} - \Theta\big\|_{\mathcal{E}^L}.
	\end{aligned}
	\nonumber
\end{equation}
% where the second inequality uses Lemma~\ref{lemma: property of DNN}, the third inequality uses Lemma~\ref{lemma: proeprty of dnn state} and	(\ref{proof inequality: estimation of difference of scores}). 
This implies that
\begin{equation}
	\begin{aligned}
		\big| \ell^{\rm DSM}_{N,S}(\tilde{\Theta}) - \ell^{\rm DSM}_{N,S}(\Theta) \big| 
		\leq B_{\pmb{\lambda}} \sqrt{2C_{\bold{s},2} B_{\rm score\_err}} \big\|\tilde{\Theta} - \Theta\big\|_{\mathcal{E}^L}.
	\end{aligned}
	\label{equation: lip of l_score_N_S}
\end{equation}
Therefore, for sufficiently large $k > K_2$,
\[
\big|\ell^{\rm DSM}_{N_k, S_k}(\Theta^*) - \ell^{\rm DSM}_{N_k, S_k}(\Theta^*_{N_k, S_k}) \big| < \epsilon/3.
\]
Thus, for $k> \max\{K_1, K_2\}$ and for all $\Theta \in \Omega_{\Theta}$, almost surely,
\begin{equation}
	\begin{aligned}
		&\ell^{\rm DSM}(\Theta^*) - \ell^{\rm DSM}(\Theta) \\
		\leq & \ell^{\rm DSM}(\Theta^*) - \ell^{\rm DSM}_{N_k, S_k}(\Theta^*_{N_k, S_k}) + \ell^{\rm DSM}_{N_k, S_k}(\Theta) - \ell^{\rm DSM}(\Theta) \\
		\leq & |\ell^{\rm DSM}(\Theta^*) - \ell^{\rm DSM}_{N_k, S_k}(\Theta^*)| + |\ell^{\rm DSM}_{N_k, S_k}(\Theta^*) - \ell^{\rm DSM}_{N_k, S_k}(\Theta^*_{N_k, S_k})| + |\ell^{\rm DSM}_{N_k, S_k}(\Theta) - \ell^{\rm DSM}(\Theta)| \\
		< &\epsilon/3 + \epsilon/3 + \epsilon/3 = \epsilon. 
	\end{aligned}
	\label{proof-ineq: final convergence}
\end{equation}
Since $\epsilon>0$ is arbitrary, we conclude that almost surely,
$
\ell^{\rm DSM}(\Theta^*) \leq \ell^{\rm DSM}(\Theta), \ \forall \Theta \in \Omega_{\Theta}$, 
i.e., $\Theta^*$ is an optimal solution of $(\mathcal{P})$ almost surely. Furthermore, 
\[
\lim_{k \to \infty} \ell_{N_k, S_k}^* = \lim_{k \to \infty} \ell^{\rm DSM}_{N_k, S_k}(\Theta^*_{N_k, S_k}) = \ell^{\rm DSM}(\Theta^*) = \ell^* \quad \text{almost surely},
\]
owing to (\ref{proof-ineq: final convergence}), which completes the proof.
\end{proof}

%%%%%%%%%%%%%%%%%%%%%%%%%%%%%%%%%%%%%%%%%%%%%%%%%%%%%%%%%%%%%%%%%%%%%%%%%%%%%

\subsection{Proof of results in Section~\ref{Sec: 3.3}}
In this subsection, we prove Lemma~\ref{lemma: Discretization error of sampling process} and Theorem~\ref{theorem: sampling error}. 
\begin{proof}[Proof of Lemma~\ref{lemma: Discretization error of sampling process}]
Our analysis for Lemma~\ref{lemma: Discretization error of sampling process} builds upon the framework developed by Chen et al. \cite{chen2023sampling}. 
The proof is organized in three steps. Noting the relationship of (\ref{equation: reverse-time generative SDE}) and (\ref{equation: Euler-Maruyama scheme}), and the continuity and positiveness of $\boldsymbol{g}$ through Assumption (A3), 
we define the process
$$
\bold{b}(t) :=
\frac{ \bold{f}( T \!- \! t, \hat{\bold{x}}_t)  \!- \! \bold{f}( T \! - \! t_N^k, \hat{\bold{x}}_{t_N^k}) }{{\boldsymbol{g}(T-t)}} \! +  \!  \Big( \frac{\boldsymbol{g}^2(T \!- \! t_N^k)\bold{s}_{\Theta}(T \! - \! t_N^k, {\hat{\bold{x}}}_{t_N^k}) }{\boldsymbol{g}(T \!- \! t)}\!- \! \boldsymbol{g}(T \!- \! t) \bold{s}_{\Theta}(T \! - \! t, {\hat{\bold{x}}}_{t})\Big),
$$
for every interval $[t_N^k,t_N^{k+1}]$ ($k=1,\ldots,N-1$) and $t\in[t_N^k,t_N^{k+1}]$.
This is the normalized drift difference that appears in the Girsanov transformation below.

\textbf{{Step~1}}. 	We estimate $\mathbb{E}_{\nu} \int_{0}^{T} [\|\bold{b}(s)\|_2^2] \mathrm{d}s$.

Fix $k$ and $t\in[t_N^k,t_N^{k+1}]$. By Assumption (A2) 
and Lemma~\ref{lemma: property of SDE-1}, we obtain
\begin{equation}
	\begin{aligned}
		& \mathbb{E}_{\nu} \Big[\|  \bold{f}( T \!- \! t, \hat{\bold{x}}_t)  \!- \! \bold{f}( T \! - \! t_N^k, \hat{\bold{x}}_{t_N^k}) \|_2^2 \Big]  \\
		\leq & 2 \mathbb{E}_{\nu} \Big[\|  \bold{f}( T \!- \! t, \hat{\bold{x}}_t)  \!- \bold{f}( T \! - \! t, \hat{\bold{x}}_{t_N^k}) \|_2^2+ \| \bold{f}( T \! - \! t, \hat{\bold{x}}_{t_N^k})- \! \bold{f}( T \! - \! t_N^k, \hat{\bold{x}}_{t_N^k}) \|_2^2 \Big] \\
		\leq & 2  C_{\bold{f}}^2 \mathbb{E}_{\nu} [ \|\hat{\bold{x}}_t -  \hat{\bold{x}}_{t_N^k} \|_2^2 ] + 2  C_{\bold{f}}^2 \mathbb{E}_{\nu} [ \|\hat{\bold{x}}_{t_N^k}\|_2^2] \omega_{\pmb{\alpha}}^2 (h_N) \\
		\leq & 2 C_{\bold{f}}^2 \mathrm{Lip}_{\hat{\bold{x}}, 2} h_N +   2  C_{\bold{f}}^2 B_{\hat{\bold{x}},2} \omega_{\pmb{\alpha}}^2 (h_N),
	\end{aligned}	
	\label{proof-inq: b1}
\end{equation}
where the second inequality is due to assumption $(A_2)$, the last inequality uses Lemma~\ref{lemma: property of SDE-1} in Appendix. 

Next, for the score-related term, we have
\begin{equation}
	\begin{aligned}
		& \mathbb{E}_{\nu} \Big[ \| \boldsymbol{g}^2(T \!- \! t_N^k)\bold{s}_{\Theta}(T \! - \! t_N^k, {\hat{\bold{x}}}_{t_N^k}) -  \boldsymbol{g}^2(T \!- \! t) \bold{s}_{\Theta}(T \! - \! t, {\hat{\bold{x}}}_{t}) \|_2^2 \Big]  \\
		\leq & 2 \mathbb{E}_{\nu} \Big[ \big\| \big( \boldsymbol{g}^2(T \!- \! t_N^k) -  \boldsymbol{g}^2(T \!- \! t) \big) \bold{s}_{\Theta}(T \! - \! t_N^k, {\hat{\bold{x}}}_{t_N^k}) \big\|_2^2 \\
		& \qquad \qquad \qquad \qquad \qquad \qquad +  \big\|  \boldsymbol{g}^2(T \!- \! t) \big( \bold{s}_{\Theta}(T \! - \! t_N^k, {\hat{\bold{x}}}_{t_N^k}) - \bold{s}_{\Theta}(T \! - \! t, {\hat{\bold{x}}}_{t}) \big) \big\|_2^2 \Big]   \\
		\leq &  8B_{\boldsymbol{g}}^2 \omega_{\boldsymbol{g}}^2 (h_N) \mathbb{E}_{\nu} \Big[ \|\bold{s}_{\Theta}(T \! - \! t_N^k, {\hat{\bold{x}}}_{t_N^k}) \|_2^2 \Big] \! + \! 2 B_{\boldsymbol{g}}^2 \mathbb{E}_{\nu} \Big[  \big\| \bold{s}_{\Theta}(T-t_N^k, {\hat{\bold{x}}}_{t_N^k}) \! -\!  \bold{s}_{\Theta}(T \! -\!  t, {\hat{\bold{x}}}_{t}) \big\|^2 \Big], 
	\end{aligned}	
	\nonumber
\end{equation}
owing to Assumption (A3) together with Remark~\ref{remark: proeprty of assumptions}.
Applying Lemma~\ref{lemma: property of DNN} and Lemma~\ref{lemma: property of SDE-1} yields
\begin{equation}
	\begin{aligned}
		& \mathbb{E}_{\nu} \Big[ \| \boldsymbol{g}^2(T \!- \! t_N^k)\bold{s}_{\Theta}(T \! - \! t_N^k, {\hat{\bold{x}}}_{t_N^k}) -  \boldsymbol{g}^2(T \!- \! t) \bold{s}_{\Theta}(T \! - \! t, {\hat{\bold{x}}}_{t}) \|_2^2 \Big]  \\
		%	\leq & 2 \mathbb{E}_{\nu} \Big[ \big\| \big( \boldsymbol{g}^2(T \!- \! t_N^k) -  \boldsymbol{g}^2(T \!- \! t) \big) \bold{s}_{\Theta}(T \! - \! t_N^k, {\hat{\bold{x}}}_{t_N^k}) \big\|_2^2 \\
		%	& \qquad \qquad \qquad \qquad \qquad \qquad +  \big\|  \boldsymbol{g}^2(T \!- \! t) \big( \bold{s}_{\Theta}(T \! - \! t_N^k, {\hat{\bold{x}}}_{t_N^k}) - \bold{s}_{\Theta}(T \! - \! t, {\hat{\bold{x}}}_{t}) \big) \big\|_2^2 \Big]   \\
		%	\leq &  8B_{\boldsymbol{g}}^2 \omega_{\boldsymbol{g}}^2 (h_N) \mathbb{E}_{\nu} \Big[ \|\bold{s}_{\Theta}(T \! - \! t_N^k, {\hat{\bold{x}}}_{t_N^k}) \|_2^2 \Big] \! + \! 2 B_{\boldsymbol{g}}^2 \mathbb{E}_{\nu} \Big[  \big\| \bold{s}_{\Theta}(T-t_N^k, {\hat{\bold{x}}}_{t_N^k}) \! -\!  \bold{s}_{\Theta}(T \! -\!  t, {\hat{\bold{x}}}_{t}) \big\|^2 \Big] \\ 
		\leq &  16B_{\boldsymbol{g}}^2 \omega_{\boldsymbol{g}}^2 (h_N)	\big( B_{\hat{\bold{x}},2} + L^2 \mathrm{Lip}_{\psi}^2 B_{\boldsymbol{e}}^2  B_{\Theta}^{4} \big) \exp(2L \mathrm{Lip}_{\psi}  B_{\Theta}^2) \\
		& + 2 B_{\boldsymbol{g}}^2 \mathbb{E}_{\nu}\big[ \big(1+L^2\mathrm{Lip}^2_{\psi} \mathrm{Lip}^2_{\boldsymbol{e}}  B_{\Theta}^4 \big) \exp(2L \mathrm{Lip}_{\psi} B_{\Theta}^2 ) \big(\|  {\hat{\bold{x}}}_{t_N^k} - {\hat{\bold{x}}}_{t} \|^2_2 + |t-t_N^k|^2\big) \big] \\
		%	\leq & 16B_{\boldsymbol{g}}^2 \omega_{\boldsymbol{g}}^2 (h_N)	\big( B_{\hat{\bold{x}},2} + L^2 \mathrm{Lip}_{\psi}^2 B_{\boldsymbol{e}}^2  B_{\Theta}^{4} \big) \exp(2L \mathrm{Lip}_{\psi}  B_{\Theta}^2) \\
		%	& + 4 B_{\boldsymbol{g}}^2 \big(1+\mathrm{Lip}_{\psi} \mathrm{Lip}_{\boldsymbol{e}}  B_{\Theta}^2 \big)^2 \exp(2L \mathrm{Lip}_{\psi} B_{\Theta}^2 )  \cdot  \mathbb{E}_{\nu}\big[ \|  {\hat{\bold{x}}}_{t_N^k} - {\hat{\bold{x}}}_{t} \|^2_2 + |t-t_N^k|^2 \big]  \\
		\leq &  16B_{\boldsymbol{g}}^2 \omega_{\boldsymbol{g}}^2 (h_N)	\big( B_{\hat{\bold{x}},2} + L^2 \mathrm{Lip}_{\psi}^2 B_{\boldsymbol{e}}^2  B_{\Theta}^{4} \big) \exp(2L \mathrm{Lip}_{\psi}  B_{\Theta}^2) \\
		& + 2 B_{\boldsymbol{g}}^2 \big(1+L^2\mathrm{Lip}^2_{\psi} \mathrm{Lip}^2_{\boldsymbol{e}}  B_{\Theta}^4 \big)  \exp(2L \mathrm{Lip}_{\psi} B_{\Theta}^2 )  \big( \mathrm{Lip}_{\hat{\bold{x}}, 2} h_N + h_N^2 \big). 
	\end{aligned}	
	\label{proof-inq: b2}
\end{equation}

Using Tonelli's theorem and combining the two bounds in (\ref{proof-inq: b1}) and (\ref{proof-inq: b2}) together with the notation $b_{\boldsymbol{g}} := \min_{ t \in [0, T]} |\boldsymbol{g}(t)|$ in Remark~\ref{remark: proeprty of assumptions}, we obtain
\begin{equation}
	\begin{aligned}
		&	\mathbb{E}_{\nu}\!\left[\int_0^T \|\bold{b}(s)\|_2^2 \mathrm{d}s\right]
		= \int_0^T \mathbb{E}_{\nu} \big[\|\bold{b}(s)\|_2^2 \big]\mathrm{d}s \\
		\leq 	&  \frac{2T}{b_{\boldsymbol{g}}^2}\Big[ 2 C_{\bold{f}}^2 \mathrm{Lip}_{\hat{\bold{x}}, 2} h_N +   2  C_{\bold{f}}^2 B_{\hat{\bold{x}},2} \omega_{\pmb{\alpha}}^2 (h_N) \\
		& \qquad \quad + 16B_{\boldsymbol{g}}^2 \omega_{\boldsymbol{g}}^2 (h_N)	\big( B_{\hat{\bold{x}},2} + L^2 \mathrm{Lip}_{\psi}^2 B_{\boldsymbol{e}}^2  B_{\Theta}^{4} \big) \exp(2L \mathrm{Lip}_{\psi}  B_{\Theta}^2) \\
		& \qquad \quad  + 2 B_{\boldsymbol{g}}^2 \big(1+L^2\mathrm{Lip}^2_{\psi} \mathrm{Lip}^2_{\boldsymbol{e}}  B_{\Theta}^4 \big)  \exp(2L \mathrm{Lip}_{\psi} B_{\Theta}^2 )  \big( \mathrm{Lip}_{\hat{\bold{x}}, 2} h_N + h_N^2 \big)	\Big] \\
		:= & C_5 \, h_N + C_6 \omega_{\boldsymbol{\alpha}}^2 (h_N)  + C_7 \omega_{\boldsymbol{g}}^2 (h_N)	< \infty,
	\end{aligned}
	\label{equation: bounded of b}
\end{equation}
where $C_5, C_6, C_7$ are defined in (\ref{coefficients: C_5_6}).
%$C_5 = \frac{2T}{b_{\boldsymbol{g}}^2} \big(2 C_{\bold{f}}^2 \mathrm{Lip}_{\hat{\bold{x}}, 2} + 4 B_{\boldsymbol{g}}^2 \big(1+\mathrm{Lip}_{\psi} \mathrm{Lip}_{\boldsymbol{e}}  B_{\Theta}^2 \big) ^2 \exp(2L \mathrm{Lip}_{\psi} B_{\Theta}^2 ) \big) \big( \mathrm{Lip}_{\hat{\bold{x}}, 2} + h_N  \big) $;  $C_6 =  \frac{2T}{b_{\boldsymbol{g}}^2} 2  C_{\bold{f}}^2 B_{\hat{\bold{x}},2}  $;  $C_7 =  \frac{2T}{b_{\boldsymbol{g}}^2} 16B_{\boldsymbol{g}}^2 	\big( B_{\hat{\bold{x}},2} + L^2 \mathrm{Lip}_{\psi}^2 B_{\boldsymbol{e}}^2  B_{\Theta}^{4} \big) \exp(2L \mathrm{Lip}_{\psi}  B_{\Theta}^2)$.

\textbf{{Step~2}}. 
Set $${\mathscr{L}}_t := \int_{0}^{t} \bold{b}(s) \mathrm{d}\bold{B}_s, \quad \mathscr{E}(\mathscr{L})_t = \exp \Big(\int_{0}^{t} \bold{b}(s) \mathrm{d}\bold{B}_s - \frac{1}{2}\int_{0}^{t} \|\bold{b}(s)\|^2 \mathrm{d}s \Big), 
$$ 
where $t \in [0,T]$, $\bold{B}$ is $\nu$-Brownian motion. By the integrability shown in \textbf{Step~1}, ${\mathscr{L}}_t$ is a well-defined continuous local martingale.
%%%%%%%%%%%%%%%%%%%%%%%%%%%%
Therefore $\mathscr{E}(\mathscr{L})_t$ is a local martingale due to \cite[Proposition~5.11]{le2016brownian}, and there exits a non-decreasing sequence of stopping times $T_m \uparrow T$ such that for any $m$, $\Big(\mathscr{E}(\mathscr{L})_{t \land T_m}\Big)_{t\in[0,T]}$ is a martingale.

Define the truncated exponential and the corresponding probability measure
\[
\mathscr{E}(\mathscr{L}^m)_t := \mathscr{E}(\mathscr{L})_{t\wedge T_m},\qquad
\frac{d\nu_{N,S}^m}{d\nu}:=\mathscr{E}(\mathscr{L}^m)_{T}=\mathscr{E}(\mathscr{L})_{T_m}.
\]
Since $\mathbb{E}_{\nu} \int_{0}^{T}\|\bold{b}(s) \bold{1}_{[0,T_m]}\|^2\mathrm{d}s \leq \mathbb{E}_{\nu} \int_{0}^{T}\|\bold{b}(s)\|^2\mathrm{d}s < \infty$ and $\mathbb{E}_{\nu} \mathscr{E}(\mathscr{L})_{T_m}  = \mathbb{E}_{\nu} \mathscr{E}(\mathscr{L}^m)_{0} = e^0 = 1$, we can get that under $\nu_{N,S}^m$, the process
$$
\boldsymbol{\beta}_t^m := \bold{B}_t - \int_{0}^{t} \bold{b}(s) \bold{1}_{[0, T_m]}(s) \mathrm{d}s
$$
is a Brownian motion through a consequence of Girsanov’s theorem (a combining of Pages 136–139, Theorem~5.22, and Theorem~4.13 in \cite{le2016brownian}, which is restated in \cite[Theorem~8]{chen2023sampling}). Combining the above equality with the definition of $\bold{b}$, we get
\begin{equation}
	\begin{aligned}
		\mathrm{d} \bold{B}_t = & 	\mathrm{d}\boldsymbol{\beta}_t^m + \bold{b}(t)\bold{1}_{[0, T_m]}(t) \mathrm{d}t \\
		= & \mathrm{d}\boldsymbol{\beta}_t^m + \frac{1}{\boldsymbol{g}(T-t)} \big( \bold{f}( T-t, \hat{\bold{x}}_t)  \!- \! \bold{f}( T-t_N^k, \hat{\bold{x}}_{t_N^k})  \big)\bold{1}_{[0, T_m]}(t) \mathrm{d}t \\
		& + \Big( \frac{\boldsymbol{g}^2(T \!- \! t_N^k)\bold{s}_{\Theta}(T \! - \! t_N^k, {\hat{\bold{x}}}_{t_N^k}) }{\boldsymbol{g}(T \!- \! t)}\!- \! \boldsymbol{g}(T \!- \! t) \bold{s}_{\Theta}(T \! - \! t, {\hat{\bold{x}}}_{t})\Big)  \bold{1}_{[0, T_m]}(t) \mathrm{d}t .
	\end{aligned}
	\label{equation: BM-2}
\end{equation}

Recall that under measure $\nu$, we have
\begin{equation}
	\begin{aligned}
		\mathrm{d} {\hat{\bold{x}}}_t  & = \left[- \bold{f}(T-t, {\hat{\bold{x}}}_t) + \boldsymbol{g}^2(T-t) \bold{s}_{\Theta}(T-t, {\hat{\bold{x}}}_t)\right] \mathrm{d} t+\boldsymbol{g}(T-t) \mathrm{d} {\bold{B}}_t ,\\
		\hat{\bold{x}}_0 & \sim \boldsymbol{p}(T, \cdot).
	\end{aligned}
	\nonumber
\end{equation}
Note that the above equation still holds $\nu_{N,S}^m$-a.s. since $\nu_{N,S}^m$ is absolutely continuous w.r.t. $\nu$ (i.e., $\nu_{N,S}^m \ll \nu$). Substituting (\ref{equation: BM-2}) into the above equation we obtain $\nu_{N,S}^m$-a.s.,
\begin{equation}
	\begin{aligned}
		\mathrm{d} \hat{\bold{x}}_t  = &  
		\left[- \bold{f}(T-t, {\hat{\bold{x}}}_t) + \boldsymbol{g}^2(T-t) \bold{s}_{\Theta}(T-t, {\hat{\bold{x}}}_t)\right] \mathrm{d} t+\boldsymbol{g}(T-t) \mathrm{d} \boldsymbol{\beta}_t^m \\
		& + \big( \bold{f}( T-t, \hat{\bold{x}}_t)  \!- \! \bold{f}( T-t_N^k, \hat{\bold{x}}_{t_N^k})  \big)\bold{1}_{[0, T_m]}(t) \mathrm{d}t \\
		& + \boldsymbol{g}(T-t) \Big( \frac{\boldsymbol{g}^2(T \!- \! t_N^k)\bold{s}_{\Theta}(T \! - \! t_N^k, {\hat{\bold{x}}}_{t_N^k}) }{\boldsymbol{g}(T \!- \! t)}\!- \! \boldsymbol{g}(T \!- \! t) \bold{s}_{\Theta}(T \! - \! t, {\hat{\bold{x}}}_{t})\Big)  \bold{1}_{[0, T_m]}(t) \mathrm{d}t \\
		= &  \left[- \bold{f}(T-t_N^k, \hat{\bold{x}}_{t_N^k})  + \boldsymbol{g}^2(T-t_N^k) \bold{s}_{\Theta}(T-t_N^k, {\hat{\bold{x}}}_{t_N^k})  \right] \bold{1}_{[0, T_m]}(t) \mathrm{d} t \\ 
		& + \left[-\bold{f}( T-t, \hat{\bold{x}}_t) + \boldsymbol{g}^2(T-t) \bold{s}_{\Theta}(T-t, \hat{\bold{x}}_t) \right]\bold{1}_{[T_m, T]}(t) \mathrm{d} t   + \boldsymbol{g}(T-t)   \mathrm{d}\boldsymbol{\beta}_t^m,\\
		\hat{\bold{x}}_0 \sim & \boldsymbol{p}(T, \cdot),
	\end{aligned}
	\label{equation: SDE-2}
\end{equation}
i.e. $\nu_{N,S}^m$ is the law of the process which follows the Euler-Maruyama drift up to $T_m$
and the true drift afterwards (driven by the Brownian motion $\boldsymbol{\beta}^m$). 
Therefore, we have the bound
\begin{equation}
	\begin{aligned}
		{KL}( \nu || \nu_{N,S}^m) = & \mathbb{E}_{\nu} \Big[\ln (\frac{d \nu}{d \nu_{N,S}^m}) \Big] 
		=  \mathbb{E}_{\nu} \Big[ - \ln (\mathscr{E}(\mathscr{L})_{T_m}) \Big] \\
		= & \mathbb{E}_{\nu} \Big[-\int_{0}^{T_m} \bold{b}(s) \mathrm{d}\bold{B}_s + \frac{1}{2}\int_{0}^{T_m} \|\bold{b}(s)\|^2 \mathrm{d}s  \Big] \\
		= & \frac{1}{2} \mathbb{E}_{\nu} \Big[\int_{0}^{T_m} \|\bold{b}(s)\|^2 \mathrm{d}s \Big] 
		\leq \frac{1}{2} \mathbb{E}_{\nu} \Big[\int_{0}^{T} \|\bold{b}(s)\|^2 \mathrm{d}s  \Big] \\
		\leq &  C_5 \, h_N + C_6 \, \omega_{\boldsymbol{\alpha}}^2 (h_N)  + C_7 \, \omega_{\boldsymbol{g}}^2 (h_N)	 .
	\end{aligned}
	\label{proof-inq: bound_nu&nu_NS}
\end{equation}

\textbf{Step~3}. We show the final result by an approximation argument. 

As in \cite{chen2023sampling}, we consider a coupling of
$\{\nu_{N,S}^m\}_{m \in \mathbb{N}}$, $\nu$, and a sequence of stochastic processes
$\{\check{\bold{x}}^m_t\}_{m \in \mathbb{N}}$ with laws $\nu_{N,S}^m$, a stochastic process
$\{\check{\bold{x}}_t\}$ with law $\nu_{N,S}$, and a single Brownian motion
$\{\bold{B}_t\}$, all defined on the same probability space $(\Omega,\mathcal{F},P)$.
The processes $\{\check{\bold{x}}^m_t\}$ and $\{\check{\bold{x}}_t\}$ satisfy the following SDEs.
\begin{equation}
	\begin{aligned}
		\mathrm{d} \check{\bold{x}}^m_t  
		= &  \left[- \bold{f}( T-t_N^k, \check{\bold{x}}^m_{t_N^k})  + \boldsymbol{g}^2(T-t_N^k) \bold{s}_{\Theta}(T-t_N^k, {\check{\bold{x}}}^m_{t_N^k})  \right] \bold{1}_{[0, T_m]}(t) \mathrm{d} t \\
		& + \left[-\bold{f}( T-t, \check{\bold{x}}^m_t) + \boldsymbol{g}^2(T-t) \bold{s}_{\Theta}(T-t, \check{\bold{x}}_t) \right]\bold{1}_{[T_m, T]}(t)\mathrm{d} t   + \boldsymbol{g}(T-t)   \mathrm{d} {\bold{B}}_t,
	\end{aligned}
	\nonumber
\end{equation}
and
\begin{equation}
	\begin{aligned}
		\mathrm{d} \check{\bold{x}}_t  
		=   - \bold{f}( T-t_N^k, \check{\bold{x}}_{t_N^k})  + \boldsymbol{g}^2(T-t_N^k) \bold{s}_{\Theta}(T-t_N^k, {\check{\bold{x}}}_{t_N^k})  \mathrm{d} t   + \boldsymbol{g}(T-t)   \mathrm{d} {\bold{B}}_t,
	\end{aligned}
	\nonumber
\end{equation}
respectively, with $\check{\bold{x}}^m_0 = \check{\bold{x}}_0$ a.s. and $\check{\bold{x}}_0\sim \boldsymbol{p}(T, \cdot)$.

%	There exists an integer $N>0$ such that $T-\epsilon < T_m$ if $m>N$, and for all $t \in [0, T]$,
%	$$
%	\begin{aligned}
	%		\|\pi_{\varepsilon}\left(\check{\bold{x}}^m\right)(t) - \pi_{\varepsilon}\left(\check{\bold{x}}\right)(t)\|_2 = & \| \check{\bold{x}}^m_{t \wedge (T-\varepsilon)} - \check{\bold{x}}_{t \wedge (T-\varepsilon)}\|_2 \\
	%		= & \left\{ \begin{array}{ll}
		%			\| \check{\bold{x}}^m_t - \check{\bold{x}}_t\|_2 = 0,  \text{ a.s., if } 0 \leq t \leq T-\varepsilon \\
		%			\| \check{\bold{x}}^m_{T-\varepsilon} - \check{\bold{x}}_{T-\varepsilon}\|_2=0 , \text{ a.s., if } T-\varepsilon < t \leq T
		%		\end{array} \right. \\
	%		< & \epsilon.
	%	\end{aligned}
%	$$
%	
Fix $\varepsilon>0$ and define the truncation map
$\pi_\varepsilon:\mathcal C([0,T];\mathbb{R}^d)\to\mathcal C([0,T];\mathbb{R}^d)$ by
$\pi_\varepsilon(\eta)(t)=\eta(t\wedge (T-\varepsilon))$.
Noting that $\check{\bold{x}}_t^m=\check{\bold{x}}_t$ $a.s.$ for every $t \in\left[0, T_m\right]$. 
Since $T_m \uparrow T$ a.s., for every fixed $\varepsilon>0$ and for a.e.\ $\omega$, there exists	an integer $m_0(\omega,\varepsilon)$ such that $T-\varepsilon < T_m(\omega)$ for all
$m \ge m_0(\omega,\varepsilon)$.
Therefore, for a.e.\ $\omega$, we have
$\pi_\varepsilon(\check{\bold{x}}^m(\omega)) = \pi_\varepsilon(\check{\bold{x}}(\omega))$
for all sufficiently large $m$. 
Hence, we have $\pi_{\varepsilon}\left(\check{\bold{x}}^m\right) \to  \pi_{\varepsilon}\left(\check{\bold{x}}\right)$ a.s., uniformly
over $[0, T]$, and $h(\pi_{\varepsilon}\left(\check{\bold{x}}^m\right)) \to h(\pi_{\varepsilon}\left(\check{\bold{x}}\right))$ a.s. as $m \to \infty$ for any bounded, continuous function $h: \mathcal{C}\left([0, T] ; \mathbb{R}^d\right) \to \mathbb{R}$.
Then 
$$
\begin{aligned}
	\int\! h \mathrm{d} (\left(\pi_{\varepsilon}\right)_{\#} \nu_{N,S}^m) \! = & \!  \int_{\mathcal{C}\left([0, T] ; \mathbb{R}^d\right)} h (\pi_{\varepsilon}(\omega)) \nu_{N,S}^m (\mathrm{d}\omega) \\
	= & \! \int_{\Omega} h (\pi_{\varepsilon}(\check{\bold{x}}^m(\omega))) P (\mathrm{d}\omega) 
	\! \to \!  \int_{\Omega} h (\pi_{\varepsilon}(\check{\bold{x}}(\omega))) P (\mathrm{d}\omega) \\
	& \qquad \qquad \qquad \qquad \ = \!  \int_{\mathcal{C}\left([0, T] ; \mathbb{R}^d\right)} h (\pi_{\varepsilon}(\omega)) \nu_{N,S} (\mathrm{d}\omega)\!  = \!	\int \! h \mathrm{d} (\left(\pi_{\varepsilon}\right)_{\#} \nu_{N,S}),
\end{aligned}
$$
and therefore, $\pi_{\varepsilon \#} \nu_{N,S}^m $ converges narrowly to  $ \pi_{\varepsilon \#} \nu_{N,S}$  \cite[Sec.~5.1]{ambrosio2005gradient}. Using the joint lower semicontinuity of the KL divergence, the data-processing inequality \cite[Lemma 9.4.3 and Lemma 9.4.5]{ambrosio2005gradient} and (\ref{proof-inq: bound_nu&nu_NS}), we obtain
$$
\begin{aligned}
	{KL}\left(\left(\pi_{\varepsilon}\right)_{\#} \nu \| \left(\pi_{\varepsilon}\right)_{\#} \nu_{N,S}\right) \! &  \leq \! \liminf _{m \rightarrow \infty} \mathrm{KL}\left(\left(\pi_{\varepsilon}\right)_{\#}  \nu \|\left(\pi_{\varepsilon}\right)_{\#}  \nu_{N,S}^m\right) \\
	&\!  \leq \! \liminf _{m \rightarrow \infty} \mathrm{KL}\left(\nu \| \nu_{N,S}^m\right) \\
	%	& \lesssim\left(\varepsilon_{\text {score }}^2+L^2 d h+L^2 \mathfrak{m}_2^2 h^2\right) T
	& \leq  C_5 \, h_N + C_6 \omega_{\boldsymbol{\alpha}}^2 (h_N)  + C_7 \omega_{\boldsymbol{g}}^2 (h_N)	 .
\end{aligned}
$$

Now let $\varepsilon_k = 1/k$. Since for every
$\eta\in\mathcal C([0,T];\mathbb R^d)$,
$\pi_{\varepsilon_k}(\eta) \to \eta$ uniformly over $[0,T]$, it follows from
\cite[Corollary 9.4.6]{ambrosio2005gradient} that
\[
\lim_{k\to\infty}
\mathrm{KL}\big((\pi_{\varepsilon_k})_{\#}\nu \,\|\, (\pi_{\varepsilon_k})_{\#}\nu_{N,S}\big)
=
\mathrm{KL}(\nu \,\|\, \nu_{N,S}),
\]
which completes the proof.

%Noting that $\pi_{\varepsilon}(\omega) \to \omega$ as $\varepsilon \to 0$ uniformly over $[0,T]$, we have $$\lim\limits_{\varepsilon \to 0} {KL}\big(\left(\pi_{\varepsilon}\right)_{\#} \nu \| \left(\pi_{\varepsilon}\right)_{\#} \nu_{N,S}\big) = {KL}\left(\nu \| \nu_{N,S}\right)$$by \cite[Corollary 9.4.6]{ambrosio2005gradient}, which completes the proof.
\end{proof}

We are ready to prove Theorem~\ref{theorem: sampling error}.
\begin{proof}[Proof of Theorem~\ref{theorem: sampling error}]
\textbf{Step~1}. We prove ${KL}\left(\mu \| \nu\right) \leq  \frac{T}{2}\,\ell^{\rm ESM}(\Theta)$.

Let $\bar{\bold{x}}_t$ be the solution of \eqref{equation: reverse-time SDE} with initialization ${\boldsymbol{p}}(T,\cdot)$. Define the process
$$
\bold{b}(t) :=
\boldsymbol{g}(T-t)\big(\bold{s}_{\Theta}(T-t, \bar{\bold{x}}_t) - \nabla \ln \boldsymbol{p}(T-t, \bar{\bold{x}}_t)\big), t\in[0, T].
$$
By Assumption (A2) (A4), Lemma~\ref{lemma: proeprty of dnn state} and Lemma~\ref{lemma: property of SDE-1}, we have
\begin{equation}
	\begin{aligned}
		\mathbb{E}_{\mu} \Big[\| \bold{b}(t) \|_2^2 \Big]  
		\! \leq \! 2B_{\boldsymbol{g}}^2 \Big( B_{\bold{s},2} + 2 C_{\boldsymbol{p}}^2 \big(1 + B_{\bold{x},2} \big)\Big).
	\end{aligned}	
	\label{proof-inq: b1-2}
\end{equation}
Using Tonelli's theorem we obtain
\begin{equation}
	\begin{aligned}
		&	\mathbb{E}_{\mu}\!\left[\int_0^T \|\bold{b}(s)\|_2^2 \mathrm{d}s\right]
		= \int_0^T \mathbb{E}_{\mu} \big[\|\bold{b}(s)\|_2^2 \big]\mathrm{d}s 
		\leq 2T B_{\boldsymbol{g}}^2 \Big( B_{\bold{s},2} + 2 C_{\boldsymbol{p}}^2 \big(1 + B_{\bold{x},2} \big)\Big) < +\infty.
	\end{aligned}
	\label{equation: bounded of b-2}
\end{equation}

Set ${\mathscr{L}}_t := \int_{0}^{t} \bold{b}(s) \mathrm{d}\bold{B}_s, \quad \mathscr{E}(\mathscr{L})_t = \exp \Big(\int_{0}^{t} \bold{b}(s) \mathrm{d}\bold{B}_s - \frac{1}{2}\int_{0}^{t} \|\bold{b}(s)\|^2 \mathrm{d}s \Big), 
$
where $t \in [0,T]$, $\bold{B}$ is $\mu$-Brownian motion. By (\ref{equation: bounded of b-2}), ${\mathscr{L}}_t$ is a well-defined continuous martingale and $\mathscr{E}(\mathscr{L})_t$ is a local martingale due to \cite[Proposition~5.11]{le2016brownian}. Therefore, there exits a non-decreasing sequence of stopping times $T_m \uparrow T$ such that for any $m$, $\Big(\mathscr{E}(\mathscr{L})_{t \land T_m}\Big)_{t\in[0,T]}$ is a martingale.  
Define the truncated exponential and the corresponding probability measure
\[
\mathscr{E}(\mathscr{L}^m)_t := \mathscr{E}(\mathscr{L})_{t\wedge T_m},\qquad
\frac{d\nu^m}{d\mu}:=\mathscr{E}(\mathscr{L}^m)_{T}=\mathscr{E}(\mathscr{L})_{T_m}.
\]
Since $\mathbb{E}_{\mu} \int_{0}^{T}\|\bold{b}(s) \bold{1}_{[0,T_m]}\|^2\mathrm{d}s \leq \mathbb{E}_{\mu} \int_{0}^{T}\|\bold{b}(s)\|^2\mathrm{d}s < \infty$ and $\mathbb{E}_{\mu} \mathscr{E}(\mathscr{L})_{T_m}  = \mathbb{E}_{\mu} \mathscr{E}(\mathscr{L}^m)_{0} = e^0 = 1$, we can get that under $\nu^m$, the process
$$
\boldsymbol{\beta}_t^m := \bold{B}_t - \int_{0}^{t} \bold{b}(s) \bold{1}_{[0, T_m]}(s) \mathrm{d}s
$$
is a Brownian motion through a consequence of Girsanov’s theorem (a combining of Pages 136–139, Theorem~5.22, and Theorem~4.13 in \cite{le2016brownian}, which is restated in \cite[Theorem~8]{chen2023sampling}). 

Recall that under measure $\mu$, we have
\begin{equation}
	\begin{aligned}
		\mathrm{d} \bar{\bold{x}}_t \!  =  \! \left[- \bold{f}(T-t, \bar{\bold{x}}_t) \! + \!  \boldsymbol{g}^2(T-t) \nabla \ln \boldsymbol{p}(T-t, \bar{\bold{x}}_t)\right] \mathrm{d} t \! +\! \boldsymbol{g}(T-t) \mathrm{d} {\bold{B}}_t , \
		\bar{\bold{x}}_0  \sim \boldsymbol{p}(T, \cdot).
	\end{aligned}
	\nonumber
\end{equation}
Note that the above equation still holds $\nu^m$-a.s. since $\nu^m$ is absolutely continuous w.r.t. $\mu$ (i.e., $\nu^m \ll \mu$). Then we obtain $\nu^m$-a.s.,
\begin{equation}
	\begin{aligned}
		\mathrm{d} \bar{\bold{x}}_t \!  = & \! \left[- \bold{f}(T-t, \bar{\bold{x}}_t) \! + \!  \boldsymbol{g}^2(T-t) \nabla \ln \boldsymbol{p}(T-t, \bar{\bold{x}}_t)\right] \mathrm{d} t \! +\! \boldsymbol{g}(T-t) \mathrm{d}\boldsymbol{\beta}_t^m \\
		& + \! \boldsymbol{g}^2(T-t)\big(\bold{s}_{\Theta}(T-t, \bar{\bold{x}}_t) - \nabla \ln \boldsymbol{p}(T-t, \bar{\bold{x}}_t)\big) \bold{1}_{[0, T_m]}(t) \mathrm{d} t \\
		%%%%%%%%%%%%%%%%%%%%%%%%%%%%%%%%%%%%%%%%%%%%%%%%%%%%%%%
		= & \! \left[- \bold{f}(T-t, \bar{\bold{x}}_t) \! + \!  \boldsymbol{g}^2(T-t) \bold{s}_{\Theta}(T-t, \bar{\bold{x}}_t) \right] \bold{1}_{[0, T_m]}(t) \mathrm{d} t \\
		& + \left[- \bold{f}(T-t, \bar{\bold{x}}_t) \! + \!  \boldsymbol{g}^2(T-t) \nabla \ln \boldsymbol{p}(T-t, \bar{\bold{x}}_t)\right]  \bold{1}_{[T_m, T]}(t) \mathrm{d} t \! +\! \boldsymbol{g}(T-t) \mathrm{d}\boldsymbol{\beta}_t^m . 
	\end{aligned}
	\label{equation: SDE-2-2}
\end{equation}
Therefore, we have the bound
\begin{equation}
	\begin{aligned}
		{KL}( \mu || \nu^m) = & \mathbb{E}_{\mu} \Big[\ln \Big(\frac{d \mu}{d \nu^m}\Big) \Big] 
		=  \mathbb{E}_{\mu} \Big[ - \ln (\mathscr{E}(\mathscr{L})_{T_m}) \Big] \\
		= & \mathbb{E}_{\mu} \Big[-\int_{0}^{T_m} \bold{b}(s) \mathrm{d}\bold{B}_s + \frac{1}{2}\int_{0}^{T_m} \|\bold{b}(s)\|^2 \mathrm{d}s  \Big] \\
		= & \frac{1}{2} \mathbb{E}_{\mu} \Big[\int_{0}^{T_m} \|\bold{b}(s)\|^2 \mathrm{d}s \Big] 
		\leq \frac{1}{2} \mathbb{E}_{\mu} \Big[\int_{0}^{T} \|\bold{b}(s)\|^2 \mathrm{d}s  \Big] \\
		= &   \frac{1}{2} \int_{0}^{T} \boldsymbol{g}^2(T-t)\,
		\mathbb{E}_{\bar{\bold{x}}_t \sim \boldsymbol{p}(T-t,\cdot)}
		\!\left[\|\nabla \ln \boldsymbol{p}(T-t, \bar{\bold{x}}_t) - \bold{s}_{\Theta}(T-t, \bar{\bold{x}}_t)\|_2^2\right]\mathrm{d}t \\
		= & \frac{T}{2}\,\ell^{\rm ESM}(\Theta), 
	\end{aligned}
	\label{proof-inq: bound_nu&nu_NS-2}
\end{equation}
where we use the definition of $\ell^{\rm ESM}(\Theta)$ in (\ref{loss: ESM}) and $\pmb{\lambda}=\boldsymbol{g}^2$ in last equality.

We next considering a coupling of $\{\nu^m\}_{m \in \mathcal{N}}, \nu$ and a sequence of stochastic processes $\{ \hat{\bold{x}}^m_t\} $ with laws $\nu^m$, $m \in \mathcal{N}$, a stochastic process $\{ \hat{\bold{x}}_t \}$ with law $\nu$, and a single Brownian motion $\{\bold{B}_t\}$, all defined on the same probability space $(\Omega, \mathcal{F}, P)$. The process $\{\hat{\bold{x}}^m_t\}$, $\{\hat{\bold{x}}_t\}$ satisfied the following SDEs
\begin{equation}
	\begin{aligned}
		\mathrm{d} \hat{\bold{x}}^m_t  
		= &  \left[- \bold{f}(T-t, \hat{\bold{x}}^m_t) \! + \!  \boldsymbol{g}^2(T-t) \bold{s}_{\Theta}(T-t, \hat{\bold{x}}^m_t) \right] \bold{1}_{[0, T_m]}(t) \mathrm{d} t \\
		& + \left[- \bold{f}(T-t, \hat{\bold{x}}^m_t) \! + \!  \boldsymbol{g}^2(T-t) \nabla \ln \boldsymbol{p}(T-t, \hat{\bold{x}}^m_t)\right]  \bold{1}_{[T_m, T]}(t) \mathrm{d} t \! +\! \boldsymbol{g}(T-t) \mathrm{d} {\bold{B}}_t,\\
		%%%%%%%%%%%%%%%%%%%%%%%%%%%%%%%%%%%%%%%%%%%%
		%%%%%%%%%%%%%%%%%%%%%%%%%%%%%%%%%%%%%%%%%%%%
		\mathrm{d} \hat{\bold{x}}_t  
		=  & \left[- \bold{f}(T-t, \hat{\bold{x}}_t) \! + \!  \boldsymbol{g}^2(T-t) \bold{s}_{\Theta}(T-t, \hat{\bold{x}}_t) \right]  \mathrm{d} t   + \boldsymbol{g}(T-t)   \mathrm{d} {\bold{B}}_t,
	\end{aligned}
	\nonumber
\end{equation}
respectively, with $\hat{\bold{x}}^m_0 = \hat{\bold{x}}_0$ a.s. and $\hat{\bold{x}}_0\sim \boldsymbol{p}(T, \cdot)$.

Fix $\varepsilon>0$ and define the truncation map
$\pi_\varepsilon:\mathcal C([0,T];\mathbb{R}^d)\to\mathcal C([0,T];\mathbb{R}^d)$ by
$\pi_\varepsilon(\omega)(t)=\omega(t\wedge (T-\varepsilon))$.
Noting that $\hat{\bold{x}}_t^m=\hat{\bold{x}}_t$ $a.s.$ for every $t \in\left[0, T_m\right]$. 
Since $T_m \uparrow T$ a.s., for every fixed $\varepsilon>0$ and for a.e.\ $\omega$, there exists	an integer $m_0(\omega,\varepsilon)$ such that $T-\varepsilon < T_m(\omega)$ for all
$m \ge m_0(\omega,\varepsilon)$.
Therefore, for a.e.\ $\omega$, we have
$\pi_\varepsilon(\hat{\bold{x}}^m(\omega)) = \pi_\varepsilon(\hat{\bold{x}}(\omega))$
for all sufficiently large $m$. 
Hence, we have $\pi_{\varepsilon}\left(\hat{\bold{x}}^m\right) \to  \pi_{\varepsilon}\left(\hat{\bold{x}}\right)$ a.s., uniformly
on $[0, T]$, and $h(\pi_{\varepsilon}\left(\hat{\bold{x}}^m\right)) \to h(\pi_{\varepsilon}\left(\hat{\bold{x}}\right))$ a.s. as $m \to \infty$ for any bounded, continuous function $h: \mathcal{C}\left([0, T] ; \mathbb{R}^d\right) \to \mathbb{R}$.
Then 
$$
\begin{aligned}
	\int\! h \mathrm{d} (\left(\pi_{\varepsilon}\right)_{\#} \nu^m) \! = & \!  \int_{\mathcal{C}\left([0, T] ; \mathbb{R}^d\right)} h (\pi_{\varepsilon}(\omega)) \nu^m (\mathrm{d}\omega) \\
	= & \! \int_{\Omega} h (\pi_{\varepsilon}(\hat{\bold{x}}_t^m(\omega))) P (\mathrm{d}\omega) 
	\! \to \!  \int_{\Omega} h (\pi_{\varepsilon}(\hat{\bold{x}_t}(\omega))) P (\mathrm{d}\omega) \\
	& \qquad \qquad \qquad \qquad \qquad = \!  \int_{\mathcal{C}\left([0, T] ; \mathbb{R}^d\right)} h (\pi_{\varepsilon}(\omega)) \nu (\mathrm{d}\omega)\!  = \!	\int \! h \mathrm{d} (\left(\pi_{\varepsilon}\right)_{\#} \nu),
\end{aligned}
$$
and therefore, $\pi_{\varepsilon \#} \nu^m $ narrowly convergence to $ \pi_{\varepsilon \#} \nu$  \cite[Sec.~5.1]{ambrosio2005gradient}. Using the joint lower semicontinuity of the KL divergence, the data-processing inequality \cite[Lemma 9.4.3 and Lemma 9.4.5]{ambrosio2005gradient} and (\ref{proof-inq: bound_nu&nu_NS}), we obtain
$$
\begin{aligned}
	{KL}\left(\left(\pi_{\varepsilon}\right)_{\#} \mu \| \left(\pi_{\varepsilon}\right)_{\#} \nu\right) \! &  \leq \! \liminf _{m \rightarrow \infty} \mathrm{KL}\left(\left(\pi_{\varepsilon}\right)_{\#}  \mu \|\left(\pi_{\varepsilon}\right)_{\#}  \nu^m\right) \\
	&\!  \leq \! \liminf _{m \rightarrow \infty} \mathrm{KL}\left(\mu \| \nu^m\right) 
	\leq  \frac{T}{2}\,\ell^{\rm ESM}(\Theta) .
\end{aligned}
$$

Now let $\varepsilon_k = 1/k$. Since for every
$\eta\in\mathcal C([0,T];\mathbb R^d)$,
$\pi_{\varepsilon_k}(\eta) \to \eta$ uniformly on $[0,T]$, it follows from
\cite[Corollary 9.4.6]{ambrosio2005gradient} and the above inequality that
\begin{equation}
	\begin{aligned}
		{KL}\left(\mu \| \nu\right) = \lim\limits_{\varepsilon \to 0} {KL}\big(\left(\pi_{\varepsilon}\right)_{\#} \mu \| \left(\pi_{\varepsilon}\right)_{\#} \nu\big)  \leq  \frac{T}{2}\,\ell^{\rm ESM}(\Theta).
		\label{equation: KL&ESM}
	\end{aligned}
\end{equation}

\textbf{Step~2}. We estimate $\operatorname{TV}(\check{\boldsymbol{p}}(T,\cdot)\|p_{\rm data})$.

A direct calculation gives
$$
\begin{aligned}
	& \mathbb{E}_{\mathrm{x} \sim p_{\rm data}} \!\left[\mathbb{E}_{\bold{x}_t \sim \boldsymbol{p}(t,\cdot|\mathrm{x})} \!\big[\|\bold{s}_{\Theta}(t,\bold{x}_t) - \nabla \ln \boldsymbol{p}(t,\bold{x}_t|\mathrm{x})\|_2^2\big]\right] \\ 
	%%%%%%%%%%%%%%%%%%%%%%%%%%%%%%%%%%%%%%%%%%%%%%%%%%%%%%%%%%%%%%%%%%%%
	= & \mathbb{E}_{\mathrm{x} \sim p_{\rm data}} \!\left[\mathbb{E}_{\bold{x}_t \sim \boldsymbol{p}(t,\cdot|\mathrm{x})} \!\left[   \|\bold{s}_{\Theta}(t,\bold{x}_t)\|_2^2 + \|\nabla \ln \boldsymbol{p}(t,\bold{x}_t|\mathrm{x})\|_2^2 - 2 \bold{s}_{\Theta}(t,\bold{x}_t)^{\top} \nabla \ln \boldsymbol{p}(t,\bold{x}_t|\mathrm{x}) \right] \right]\\
	%%%%%%%%%%%%%%%%%%%%%%%%%%%%%%%%%%%%%%%%%%%%%%%%%%%%%%%%%%%%%%%%%%%%
	= &	\mathbb{E}_{\bold{x}_t \sim \boldsymbol{p}(t,\cdot)} \!\left[\|\bold{s}_{\Theta}(t,\bold{x}_t)\|_2^2\right]
	- 2 \mathbb{E}_{\mathrm{x} \sim p_{\rm data}} \!\left[\mathbb{E}_{\bold{x}_t \sim \boldsymbol{p}(t,\cdot|\mathrm{x})} \!\left[  \bold{s}_{\Theta}(t,\bold{x}_t)^{\top} \nabla \ln \boldsymbol{p}(t,\bold{x}_t|\mathrm{x}) \right] \right] \\
	& + \mathbb{E}_{\mathrm{x} \sim p_{\rm data}} \!\left[\mathbb{E}_{\bold{x}_t \sim \boldsymbol{p}(t,\cdot|\mathrm{x})} \!\left[  \|\nabla \ln \boldsymbol{p}(t,\bold{x}_t|\mathrm{x})\|_2^2  \right] \right].
\end{aligned}
$$
By Lemma~\ref{lemma: property of DNN}, Lemma~\ref{lemma: property of SDE-1}, bounded assumption of data in (A1) and (\ref{assumption: changable condition}), we also have 
$$
\begin{aligned}
	& \mathbb{E}_{\mathrm{x} \sim p_{\rm data}} \!\left[\mathbb{E}_{\bold{x}_t \sim \boldsymbol{p}(t,\cdot|\mathrm{x})} \!\left[  \bold{s}_{\Theta}(t,\bold{x}_t)^{\top} \nabla \ln \boldsymbol{p}(t,\bold{x}_t|\mathrm{x}) \right] \right] \\
	= & \int_{\mathrm{x}} \boldsymbol{p}(0,\mathrm{x}) \int_{\bold{x}_t} \bold{s}_{\Theta}(t,\bold{x}_t)^{\top} \boldsymbol{p}(t,\bold{x}_t|\mathrm{x})  \nabla \ln \boldsymbol{p}(t,\bold{x}_t|\mathrm{x}) \mathrm{d} \bold{x}_t \mathrm{d}\mathrm{x} \\
	= & \int_{\bold{x}_t} \bold{s}_{\Theta}(t,\bold{x}_t)^{\top} \,   \nabla \! \int_{\mathrm{x}} \boldsymbol{p}(0,\mathrm{x}) \boldsymbol{p}(t,\bold{x}_t|\mathrm{x}) \mathrm{d} \mathrm{x} \mathrm{d}\bold{x}_t \\
	= & \int_{\bold{x}_t} \bold{s}_{\Theta}(t,\bold{x}_t)^{\top}   \nabla  \boldsymbol{p}(t,\bold{x}_t) \mathrm{d}\bold{x}_t  = \mathbb{E}_{\bold{x}_t \sim \boldsymbol{p}(t,\cdot)} \left[\bold{s}_{\Theta}(t, \bold{x}_t )^{\top}   \nabla  \ln\boldsymbol{p}(t, \bold{x}_t )\right].
\end{aligned}
$$
Combing the above two equations, we can decompose
\begin{equation}
	\begin{aligned}
		& \mathbb{E}_{\mathrm{x} \sim p_{\rm data}}\!\left[\mathbb{E}_{\bold{x}_t \sim \boldsymbol{p}(t,\cdot|\mathrm{x})} \!\big[\|\bold{s}_{\Theta}(t,\bold{x}_t) - \nabla \ln \boldsymbol{p}(t,\bold{x}_t|\mathrm{x})\|_2^2\big]\right] \\
		= &	\mathbb{E}_{\bold{x}_t \sim \boldsymbol{p}(t,\cdot)} \!\left[\|\bold{s}_{\Theta}(t,\bold{x}_t)\|_2^2\right]
		- 2\mathbb{E}_{\bold{x}_t \sim \boldsymbol{p}(t,\cdot)} \left[\bold{s}_{\Theta}(t, \bold{x}_t )^{\top}   \nabla \ln \boldsymbol{p}(t, \bold{x}_t )\right] \\
		& + \mathbb{E}_{\mathrm{x} \sim p_{\rm data}} \!\left[\mathbb{E}_{\bold{x}_t \sim \boldsymbol{p}(t,\cdot|\mathrm{x})} \!\left[  \|\nabla \ln \boldsymbol{p}(t,\bold{x}_t|\mathrm{x})\|_2^2  \right] \right]  \\
		= & 	\mathbb{E}_{\bold{x}_t \sim \boldsymbol{p}(t,\cdot)} \!\left[\|\bold{s}_{\Theta}(t,\bold{x}_t) - \nabla \ln \boldsymbol{p}(t,\bold{x}_t)\|_2^2\right]
		+  \boldsymbol{I}(t),	
	\end{aligned}
\end{equation}
where $\boldsymbol{I}(t) = \mathbb{E}_{\mathrm{x} \sim p_{\rm data}}\mathbb{E}_{\bold{x}_t \sim \boldsymbol{p}(t,\cdot|\mathrm{x})} \!\big[ \| \nabla \ln \boldsymbol{p}(t,\bold{x}_t|\mathrm{x})\|_2^2\big] - \mathbb{E}_{\bold{x}_t \sim \boldsymbol{p}(t,\cdot)}\!\left[\|\nabla \ln \boldsymbol{p}(t,\bold{x}_t)\|_2^2\right]$. By Jensen's inequality, $\boldsymbol{I}(t) \ge 0 $, which yields
\begin{equation}
	\ell^{\rm DSM}(\Theta) = \ell^{\rm ESM}(\Theta) + \frac{1}{T}\int_0^T \pmb{\lambda}(t) \boldsymbol{I}(t) \mathrm{d}t \ge \ell^{\rm ESM}(\Theta).
	\label{equation: DSM&ESM}
\end{equation}
Combining (\ref{equation: KL&ESM}) and (\ref{equation: DSM&ESM}), we get
\begin{equation}
	\begin{aligned}
		\operatorname{KL}(\mu\|\nu) \leq  \frac{T}{2}\ell^{\rm DSM}(\Theta) \leq \frac{T}{2} |\ell^{\rm DSM}(\Theta) - \ell_{N, S}^{\rm DSM}(\Theta) | +  \frac{T}{2} \ell_{N, S}^{\rm DSM}(\Theta), \Theta \in \Omega_{{\Theta}}.
	\end{aligned}
	\nonumber
\end{equation}
Applying (\ref{equation: difference of losses}) into the above inequality, we get, 
for any $\delta \in (0,1)$, with probability at least $1-\delta$, that
\begin{equation}
	\begin{aligned}
		\operatorname{KL}(\mu\|\nu) \leq
		&   \frac{TB_{\pmb{\lambda}} \sqrt{6B_{\rm score\_err}}}{2} \!  \Big[ C_1  \omega_{\pmb{\gamma}}(h_N)  \!+\!  C_2 h_N^{1/2} \!+\! C_3 h_N \Big] \! + \! \frac{TB_{\pmb{\lambda}}B_{\rm score\_err}}{2}\omega_{\pmb{\lambda}}(h_N)  \\
		&  + \frac{TB_{\pmb{\lambda}}C_4}{2\sqrt{S}}  + \frac{3TB_{\pmb{\lambda}}B_{\rm score\_err}}{2}\sqrt{\frac{2 \ln (4/\delta)}{S}} +  \frac{T}{2} \ell_{N, S}^{\rm DSM}(\Theta).
	\end{aligned}
	\label{equation: difference of path measure}
\end{equation}
Hence, for any $\delta \in (0,1)$, with probability at least $1-\delta$, we have 
\begin{equation}
	\begin{aligned}
		& \operatorname{TV}(\check{\boldsymbol{p}}(T,\cdot)\|p_{\rm data}) \leq  TV (\nu_{N,S}^{\pi}\| \mu ) \\
		\leq & \operatorname{TV}(\nu_{N,S}^{\pi}\|\nu_{N,S}) + \operatorname{TV}(\nu_{N,S}\|\nu) + \operatorname{TV}(\nu\|\mu) \\
		= &  \operatorname{TV}(\nu_{N,S}^{\pi}\|\nu_{N,S}) + \operatorname{TV}(\nu\|\nu_{N,S}) + \operatorname{TV}(\mu\|\nu) \\
		\leq & \operatorname{TV}(\pi \| \boldsymbol{p}(T, \cdot) ) + \sqrt{\frac{1}{2}\operatorname{KL}(\nu\|\nu_{N,S})} + \sqrt{\frac{1}{2}\operatorname{KL}(\mu\|\nu)} \\
		%%%%%%%%%%%%%%%%%%%%%%%%%%%%%%%%%%%%%%%%%%%%%%%%%%%
		%%%%%%%%%%%%%%%%%%%%%%%%%%%%%%%%%%%%%%%%%%%%%%%%%%
		\leq & \operatorname{TV}(\pi \| \boldsymbol{p}(T, \cdot) ) +  \sqrt{\frac{C_5}{2}}  h_N^{1/2} + \sqrt{\frac{C_6}{2}} \omega_{\boldsymbol{\alpha}} (h_N)  + \sqrt{\frac{C_7}{2}} \omega_{\boldsymbol{g}} (h_N)   \\
		& \ + \!  \frac{(TB_{\pmb{\lambda}})^{1/2} \sqrt[4]{6B_{\rm score\_err}}}{2} \!  \Big[ \sqrt{C_1}  \sqrt{\omega_{\pmb{\gamma}}(h_N) } \!+\!   \sqrt{C_2} h_N^{1/4} \!+\! \sqrt{C_3} h_N^{1/2} \Big]  \\
		& \  + \frac{(TB_{\pmb{\lambda}}B_{\rm score\_err})^{1/2}}{2} \sqrt{\omega_{\pmb{\lambda}}(h_N)} \! +  \frac{\sqrt{TB_{\pmb{\lambda}}C_4}}{2\sqrt[4]{S}} +  \frac{(3TB_{\pmb{\lambda}}B_{\rm score\_err})^{1/2}}{2}\sqrt[4]{\frac{2 \ln (4/\delta)}{S}} \\
		& \ + \sqrt{ \frac{T}{2}(\ell_{N, S}^{\rm DSM}(\Theta) - \ell_{N, S}^*)} +  \sqrt{ \frac{T\ell_{N, S}^*}{2}} \\
		%%%%%%%%%%%%%%%%%%%%%%%%%%%%%%%%%%%%%%%%%%%%%%%%%%%
		%%%%%%%%%%%%%%%%%%%%%%%%%%%%%%%%%%%%%%%%%%%%%%%%%%
		:= &  \operatorname{TV}\big(\pi \| \boldsymbol{p}(T, \cdot) \big) + C_8 \Big(
		\omega_{\boldsymbol{\alpha}} (h_N) + \omega_{\boldsymbol{g}} (h_N) + \sqrt{\omega_{\pmb{\lambda}}(h_N)} +  \sqrt{\omega_{\pmb{\gamma}}(h_N)}
		\Big) + C_9 h_N^{1/4} \\
		& \ +  \frac{\sqrt{TB_{\pmb{\lambda}}C_4}}{4\sqrt[4]{S}} +  \frac{(3TB_{\pmb{\lambda}}B_{\rm score\_err})^{1/2}}{4}\sqrt[4]{\frac{2 \ln (4/\delta)}{S}} + \sqrt{ \frac{T}{2}(\ell_{N, S}^{\rm DSM}(\Theta) - \ell_{N, S}^*)} +  \sqrt{ \frac{T\ell_{N, S}^*}{2}} ,
	\end{aligned}
	\nonumber
\end{equation}
where the first inequality uses data processing inequality \cite[Lemma 9.4.5]{ambrosio2005gradient}, the third inequality uses data processing inequality and Pinsker’s inequality \cite[Lemma~ 2.5]{tsybakov2008nonparametric}, and the fourth inequality uses Lemma~\ref{lemma: Discretization error of sampling process} and (\ref{equation: difference of path measure}). 
The $$
\begin{aligned}
	C_8 \!= & \max \left\lbrace  \sqrt{\frac{C_6}{2}}, \sqrt{\frac{C_7}{2}}, \frac{(TB_{\pmb{\lambda}}C_1)^{1/2} \sqrt[4]{6B_{\rm score\_err}}}{2} , \frac{(TB_{\pmb{\lambda}}B_{\rm score\_err})^{1/2}}{2} \right\rbrace , \\
	C_9 \! = & \frac{(TB_{\pmb{\lambda}})^{1/2} \sqrt[4]{6B_{\rm score\_err}}}{4} \Big[\sqrt{C_2}  \!+\! \sqrt{C_3} h_N^{1/4} \Big] + \sqrt{\frac{C_5}{2}}  h_N^{1/4},
\end{aligned}
$$ 
and 
$\ell_{N, S}^*$ is the optimal value of (${P}_{N,S}$).
This completes the proof.
\end{proof}

\section{Conclusion}
\label{sec: conclusion}

In this work, we presented an end-to-end theoretical framework bridging the statistical learning of score functions with the numerical sampling process of diffusion models. By focusing on practical ResNet-type architectures, we first analyzed the generalization and convergence properties of the learning problem of score function, explicitly connecting the practical finite-sample, discrete-time learning problem to the ideal continuous-time, population-level objective. Building upon this foundation, we derived an total variation bound for the generated terminal distribution. 
This framework yields a precise error decomposition that partitions the total generative discrepancy into four fundamental sources: forward process truncation, time discretization of forward and reverse process, generalization gap, and the optimization error. Consequently, it provides a quantitative characterization of how statistical sample sizes and numerical grid resolutions jointly govern the total approximation error.

While our unified bound successfully isolates the sources of sampling errors, our current analysis treats the training optimization error as a decoupled term rather than fully resolving its internal dynamics. In practical diffusion model training, optimizing deep score networks involves navigating a highly non-convex empirical risk landscape. Consequently, providing further estimations for this optimization gap remains a mathematical challenge. We leave the rigorous characterization of these non-convex training dynamics and the explicit bounding of the optimization error as an important direction for future research.

\section*{Acknowledgments}
This work was partially supported by the National Natural Science Foundation of China (grant 12271273) and the Key Program (21JCZDJC00220) of the Natural Science Foundation of Tianjin, China.

\appendix

\section{Definitions and lemmas} 
\begin{definition}
\label{definition: RC}
The (empirical) Rademacher complexity of the function class $\mathcal{F}$ for i.i.d. sample $\mathcal{Z}=\left\{\mathrm{z}_{(1)}, \mathrm{z}_{(2)}, \cdots, \mathrm{z}_{(S)}\right\}$ from $\mathfrak{B}$ is:
$$
\mathscr{R}_{\mathcal{Z}}(\mathcal{F})=\frac{1}{S} \mathbb{E}_{\pmb{\varepsilon}} \left[\sup _{{f} \in \mathcal{F}} \sum_{s=1}^S \varepsilon_{(s)} {f} \left(\mathrm{z}_{(s)}\right)\right],
$$
where the expectation is taken over $\boldsymbol{\varepsilon}=\left\{\varepsilon_{(1)}, \varepsilon_{(2)}, \cdots, \varepsilon_{(S)} \right\}$ and $\varepsilon_{(s)}$ $(1 \leq s \leq S)$ are independently and uniformly distributed over $\{-1, 1\}$.
%and they are independent random variables following the Rademacher distribution, i.e., $P\left(\varepsilon_{(s)}=1\right)=P\left(\varepsilon_{(s)}=-1\right)=1 / 2$. 
\end{definition}

%\begin{lemma} \cite{emmrich1999discrete}
%	Let $\left\lbrace  {a}^{{k}} \right\rbrace  _{{k}=0}^{{n}},\left\{{b}^{k}\right\}_{{ k}=0}^{{ n}},\left\{{ c}^{{ k}}\right\}_{{ k}=0}^{{ n}}$ be sequences of non-negative numbers such that
%	$
%	{ a}^{{ k}} \leq \left(1+{ b}^{{ k}}\right) { a}^{{ k}-1}+{ c}^{{ k}}, { k}=1,2, \cdots, { n} 
%	$.
%	Then 
%	\begin{equation}
%		{ a}^{{ k}} \leq \Big({ a}^{0}+\sum_{{ j}=1}^{{ k}} { c}^{ j}\Big) \exp {\Big( \sum_{{ j}=1}^{{ k}} { b}^{ j}\Big)}, \quad { k}=1,2, \cdots, { n }.
%		\label{dis: gronwall}
%	\end{equation} 
%	\label{lemma: dis-gronwall-inequality}
%\end{lemma}

\begin{lemma}[Properties of OU process]
\label{lemma: OU process}
Assume that (A1) hold. Consider the Ornstein–Uhlenbeck (OU) type process 
\begin{equation}
	\begin{aligned}
		\mathrm{d} \bold{x}_t &= \boldsymbol{f}(t)\bold{x}_t \mathrm{d} t+ \boldsymbol{g}(t) \mathrm{d} \bold{B}_t, \\
		\bold{x}_0 &= \mathrm{x} \sim p_{\rm data},
	\end{aligned}
\end{equation}
where $\boldsymbol{f}(t), \boldsymbol{g}(t)$  are continuous functions in $[0,T]$ and $|\boldsymbol{g}(t)|>0$ for all $t \in[0, T]$. Then
\begin{itemize}
	\item [(i)]	the solution $\bold{x}_t$ satisfies
	\begin{equation}
		\begin{aligned}
			\bold{x}_t = & \mathrm{x} \exp\big(\int_{0}^{t}\boldsymbol{f}(s) \mathrm{d}s\big) + \int_{0}^{t} \boldsymbol{g}(s)  \exp\big(\int_{s}^{t}\boldsymbol{f}(r) \mathrm{d}r\big) \bold{B}_s; \\
			\bold{x}_t | \mathrm{x} \sim & \mathcal{N} \Big(\mathrm{x} \exp\Big(\int_{0}^{t}\boldsymbol{f}(s) \mathrm{d}s\Big), \int_{0}^{t}\boldsymbol{g}^2(s)  \exp\Big(2\int_{s}^{t}\boldsymbol{f}(r) \mathrm{d}r\Big) \mathrm{d}_s \Big).
		\end{aligned}
	\end{equation}
	
	\item [(ii)] the closed-form of \( \boldsymbol{p}(t, \mathbf{x}_t | \mathrm{x}) \) is 
	\begin{equation}
		\boldsymbol{p}(t, \mathbf{x}_t | \mathrm{x}) = \mathcal{N}\left( \mathbf{x}_t ; \boldsymbol{r}(t) \mathrm{x}, \boldsymbol{r}^2(t) \boldsymbol{v}^2(t) \mathbf{I}_d \right),
	\end{equation}
	and therefore
	\begin{equation}
		\nabla \ln \boldsymbol{p}(t, \mathbf{x}_t | \mathrm{x})
		= -\frac{1}{\boldsymbol{r}^2(t) \boldsymbol{v}^2(t)} \bigl(\mathbf{x}_t - \boldsymbol{r}(t) \mathrm{x}\bigr),
	\end{equation}
	where \( \boldsymbol{r}(t) = e^{\int_0^t \boldsymbol{f}(\zeta) \, \mathrm{d}\zeta} \),  \( \boldsymbol{v}(t) = \sqrt{\int_0^t \frac{\boldsymbol{g}^2(\zeta)}{\boldsymbol{r}^2(\zeta)} \, \mathrm{d} \zeta} \).
	\item [(iii)] there exists a function $\boldsymbol{h}: [0,T]\times \mathbb{R}^{d}$ such that $\forall (t, \bold{x}_t ) \in  [0,T]\times \mathbb{R}^{d}$ 
	\begin{equation}
		\begin{aligned}
			\| \boldsymbol{p}(0,\mathrm{x}) \boldsymbol{p}(t,\bold{x}_t|\mathrm{x})\|_2  \leq \boldsymbol{h}(t, \mathrm{x}), \ 
			\text{ and }	 \int_{\mathrm{x}}  \boldsymbol{h}(t, \mathrm{x})  \mathrm{d} \mathrm{x}< \infty.
		\end{aligned}
		\label{equation: condition of exchangnable}
	\end{equation}
	%		the following equality hold:
	%		\begin{equation}
		%			\begin{aligned}
			%				& \int_{\mathrm{x}\sim p_{\rm data}} \boldsymbol{p}(0,\mathrm{x}) \int_{\mathbf{x}_t} \bold{s}_{\Theta}(t,\mathbf{x}_t)^{\top} \boldsymbol{p}(t,\mathbf{x}_t|\mathrm{x})  \nabla \ln \boldsymbol{p}(t,\mathbf{x}_t|\mathrm{x}) \mathrm{d} \mathbf{x}_t \mathrm{d}\mathrm{x} \\
			%				= & \int_{\mathbf{x}_t} \bold{s}_{\Theta}(t,\mathbf{x}_t) \,   \nabla \! \int_{\mathrm{x}\sim p_{\rm data}} \boldsymbol{p}(0,\mathrm{x}) \boldsymbol{p}(t,\mathbf{x}_t|\mathrm{x}) \mathrm{d} \mathrm{x} \mathrm{d}\mathbf{x}_t 
			%			\end{aligned}
		%			\label{equation: condition of exchangnable}
		%		\end{equation}
\end{itemize}
\end{lemma}
\begin{proof}
(i)	Let $\bold{z}_t=\phi(\bold{x}_t, t) = \bold{x}_t \exp\big(-\int_{0}^{t}\boldsymbol{f}(s) \mathrm{d}s\big) $. By Ito's formula,
\begin{equation}
	\begin{aligned}
		\mathrm{d} \bold{z}_t &= \frac{\partial \phi}{\partial t}(\bold{x}_t, t) \mathrm{d}t + \frac{\partial \phi}{\partial \bold{x}_t}(\bold{x}_t, t) \mathrm{d}\bold{x}_t + \frac{1}{2} \frac{\partial^2 \phi}{\partial^2 \bold{x}_t}(\bold{x}_t, t) (\mathrm{d}\bold{x}_t)^2 \\
		&= -f(t)\bold{x}_t \exp\Big(-\int_{0}^{t}\boldsymbol{f}(s) \mathrm{d}s\Big) \mathrm{d}t +
		\exp\Big(-\int_{0}^{t}\boldsymbol{f}(s) \mathrm{d}s \Big)  \big( \boldsymbol{f}(t)\bold{x}_t \mathrm{d} t+ {g}(t) \mathrm{d} \bold{B}_t\big) \\
		&= \boldsymbol{g}(t)  \exp\Big(\int_{0}^{t}\boldsymbol{f}(s) \mathrm{d}s\Big) \mathrm{d} \bold{B}_t.
	\end{aligned}
\end{equation}
Therefore, 
$
\bold{z}_t = \bold{z}_0 + \int_{0}^{t}\boldsymbol{g}(s)  \exp\big(-\int_{0}^{s}\boldsymbol{f}(r) \mathrm{d}r\big) \mathrm{d}\bold{B}_s,
$
and 
\begin{equation}
	\bold{x}_t = \bold{x}_0\exp\Big(\int_{0}^{t}\boldsymbol{f}(s) \mathrm{d}s\Big) + \int_{0}^{t}\boldsymbol{g}(s)  \exp\Big(\int_{s}^{t}\boldsymbol{f}(r) \mathrm{d}r\Big) \mathrm{d}\bold{B}_s.
\end{equation}
Taking expectation for the above equality on both sides, we obtain
$$
\begin{aligned}
	\mathbb{E}[\bold{x}_t] & =  \mathbb{E}\Big[\bold{x}_0\exp\Big(\int_{0}^{t}\boldsymbol{f}(s) \mathrm{d}s\Big) + \int_{0}^{t}\boldsymbol{g}(s)  \exp\Big(\int_{s}^{t}\boldsymbol{f}(r) \mathrm{d}r\Big) \mathrm{d}\bold{B}_s \Big] \\
	& =  \mathbb{E} \Big[\bold{x}_0\exp\Big(\int_{0}^{t}\boldsymbol{f}(s) \mathrm{d}s\Big) \Big] = \bold{x}_0\exp\Big(\int_{0}^{t}\boldsymbol{f}(s) \mathrm{d}s\Big).
\end{aligned}
$$
Furthermore,
$$
\begin{aligned}
	\mathbb{E}[\bold{x}_t \otimes \bold{x}_t]_{i,j}  = & \bold{x}_0[i]\bold{x}_0[j] \exp\Big(2\int_{0}^{t}\boldsymbol{f}(s) \mathrm{d}s\Big)  \\
	& \ + \!  \mathbb{E} \Big[ \Big(\int_{0}^{t}\boldsymbol{g}(s)  \exp\Big( \! \int_{s}^{t}\boldsymbol{f}(r) \mathrm{d}r \! \Big)\mathrm{d} \bold{B}_s \Big) \! \otimes \! \Big(\int_{0}^{t}\boldsymbol{g}(s)  \exp\Big( \! \int_{s}^{t}\boldsymbol{f}(r) \mathrm{d}r \! \Big) \mathrm{d}\bold{B}_s  \! \Big) \Big]_{i,j} \\
	= & \bold{x}_0[i]\bold{x}_0[j] \exp\Big(2\int_{0}^{t}\boldsymbol{f}(s) \mathrm{d}s\Big) 
	+  \delta_{i,j} \cdot \int_{0}^{t}\boldsymbol{g}^2(s)  \exp\Big(2\int_{s}^{t}\boldsymbol{f}(r) \mathrm{d}r\Big) \mathrm{d}s.
\end{aligned}
$$
Here $\otimes$ is the outer product of vectors. Combining the above equations, we obtain
$$
\bold{x}_t | \mathrm{x} \sim \mathcal{N} \Big(\mathrm{x} \exp\Big(\int_{0}^{t}\boldsymbol{f}(s) \mathrm{d}s\Big), \int_{0}^{t}\boldsymbol{g}^2(s)  \exp\Big(2\int_{s}^{t}\boldsymbol{f}(r) \mathrm{d}r\Big) \mathrm{d}s \Big).
$$

(ii) Let \( \boldsymbol{r}(t) = e^{\int_0^t \boldsymbol{f}(\zeta) \, d\zeta} \),  \( \boldsymbol{v}(t) = \sqrt{\int_0^t \frac{\boldsymbol{g}^2(\zeta)}{\boldsymbol{r}^2(\zeta)} \, d\zeta} \). By Step~1, the closed-form expression of \( \boldsymbol{p}(t, \mathbf{x}_t | \mathrm{x}) \) is 
$
\boldsymbol{p}(t, \mathbf{x}_t | \mathrm{x}) = \mathcal{N}\left( \mathbf{x}_t ; \boldsymbol{r}(t) \mathrm{x}, \boldsymbol{r}^2(t) \boldsymbol{v}^2(t) \mathbf{I}_d \right)
$, i.e.,
%Note that the explicit density function for a multivariate Gaussian \(\mathcal{N}(\mathrm{x}; \sigma, \Sigma)\) with mean \(\sigma\) and covariance \(\Sigma\) is:
%\[
%p(\mathrm{x}) = \frac{1}{\sqrt{(2\pi)^d \det(\Sigma)}} \exp\left(-\frac{1}{2} (\mathrm{x} - \sigma)^\top \Sigma^{-1} (\mathrm{x} - \sigma)\right).
%\]
%For $\boldsymbol{p}(t, \mathbf{x}_t | \mathrm{x})$, we have
%$	\sigma = \boldsymbol{r}(t) \mathrm{x}, \quad \Sigma = \boldsymbol{r}^2(t) \boldsymbol{v}^2(t) \mathbf{I}_d,$
%and \(\Sigma^{-1} = \frac{1}{\boldsymbol{r}^2(t) \boldsymbol{v}^2(t)} \mathbf{I}_d\).
%Substituting into the formula for the Gaussian density gives
\[
\boldsymbol{p}(t, \mathbf{x}_t\mid \mathrm{x}) =
\frac{1}{\sqrt{(2\pi)^d \bigl(\boldsymbol{r}^2(t) \boldsymbol{v}^2(t)\bigr)^d}}
\exp\left(-\frac{1}{2} \frac{\|\mathbf{x}_t - \boldsymbol{r}(t) \mathrm{x}\|_2^2}{\boldsymbol{r}^2(t) \boldsymbol{v}^2(t)}\right).
\]
Therefore,
%$
%\ln \boldsymbol{p}(t, \mathbf{x}_t \mid \mathrm{x}) =
%-\frac{d}{2} \ln\bigl(2\pi \boldsymbol{r}^2(t) \boldsymbol{v}^2(t)\bigr)
%-\frac{\|\mathbf{x}_t - \boldsymbol{r}(t) \mathrm{x}\|_2^2}{2 \boldsymbol{r}^2(t) \boldsymbol{v}^2(t)},
%$
%and 
$
\nabla \ln \boldsymbol{p}(t, \mathbf{x}_t | \mathrm{x})
= -\frac{1}{\boldsymbol{r}^2(t) \boldsymbol{v}^2(t)} \bigl(\mathbf{x}_t - \boldsymbol{r}(t) \mathrm{x}_0\bigr).
$

(iii) 
Since \(\boldsymbol{r}(t)>0, \boldsymbol{v}(t)>0\), the conditional density \(p(t,\mathbf{x}_t\mid \mathrm{x})\) is a non-degenerate Gaussian for each \(\mathrm{x}\), and we have
\begin{equation}
	\begin{aligned}
		& \int_{\mathrm{x}} \boldsymbol{p}(0,\mathrm{x}) \int_{\mathbf{x}_t} \bold{s}_{\Theta}(t,\mathbf{x}_t)^{\top} \boldsymbol{p}(t,\mathbf{x}_t|\mathrm{x})  \nabla \ln \boldsymbol{p}(t,\mathbf{x}_t|\mathrm{x}) \mathrm{d} \mathbf{x}_t \mathrm{d}\mathrm{x} \\
		= & \int_{\mathbf{x}_t} \bold{s}_{\Theta}(t,\mathbf{x}_t)^{\top} \,     \int_{\mathrm{x}}  \boldsymbol{p}(0,\mathrm{x}) \nabla \boldsymbol{p}(t,\mathbf{x}_t|\mathrm{x}) \mathrm{d} \mathrm{x} \mathrm{d}\mathbf{x}_t 
	\end{aligned}
	\nonumber
\end{equation}
Therefore, to prove (\ref{equation: condition of exchangnable}), we only need to justify
\begin{equation}
	\int_{\mathrm{x}}  \boldsymbol{p}(0,\mathrm{x}) \nabla \boldsymbol{p}(t,\mathbf{x}_t|\mathrm{x}) \mathrm{d} \mathrm{x}
	= \nabla \int_{\mathrm{x}}  \boldsymbol{p}(0,\mathrm{x})  \boldsymbol{p}(t,\mathbf{x}_t|\mathrm{x}) \mathrm{d} \mathrm{x}, \ \forall \mathbf{x}_t \in \mathbb{R}^d.
	\label{equation: proof-ex-con-1}
\end{equation}

From (ii), we get
\begin{equation}
	\|\nabla \boldsymbol{p}(t,\mathbf{x}_t|\mathrm{x})\|_2
	= \frac{1}{\boldsymbol{r}(t)^2 \boldsymbol{v}(t)^2}\,\|\mathbf{x}_t - \boldsymbol{r}(t) \mathrm{x}\|\,(2\pi \boldsymbol{r}(t)^2 \boldsymbol{v}(t)^2)^{-d/2}\exp\!\Big(-\frac{\|\mathbf{x}_t-\boldsymbol{r}(t)\mathrm{x}\|^2}{2 \boldsymbol{r}(t)^2 \boldsymbol{v}(t)^2}\Big).
	\label{equation: grad_p}
\end{equation}
Fix $t \in [0,T]$ and consider the scalar function \(q(u)=u\exp(-\alpha u^2)\) on \(u\ge0\) with \(\alpha=\frac{1}{2\boldsymbol{r}(t)^2 \boldsymbol{v}(t)^2}>0\). Its maximum is attained at \(u=(2\alpha)^{-1/2}=\boldsymbol{r}(t)\boldsymbol{v}(t)\) and equals \((\boldsymbol{r}(t)\boldsymbol{v}(t))e^{-1/2}\). 
Taking \(y=\|\mathbf{x}_t - \boldsymbol{r}(t) \mathrm{x}\|\) in (\ref{equation: grad_p}) and applying the above discussion yield 
\[
\|\nabla \boldsymbol{p}(t,\mathbf{x}_t|\mathrm{x})\|_2 \leq  (2\pi \boldsymbol{r}^2(t)\boldsymbol{v}^2(t))^{-d/2}(\boldsymbol{r}(t)\boldsymbol{v}(t))^{-1}e^{-1/2}=:B_{\rm grad }(t).
\]
Take \(\boldsymbol{h}(t, \mathrm{x}):=B_{\rm grad }(t) \boldsymbol{p}(0,\mathrm{x})\) as a dominating function. Therefore, 
$$
|\boldsymbol{p}(0,\mathrm{x})  \boldsymbol{p}(t,\mathbf{x}_t|\mathrm{x})| \leq \boldsymbol{h}(t, \mathrm{x}), \ \forall (t, \mathbf{x}_t) \in [0,T]\times\mathbb{R}^d, \text{ and } \int_{\mathrm{x}} \boldsymbol{h}(t, \mathrm{x}) \mathrm{d} \mathrm{x} = B_{\rm grad }(t) < \infty,
$$
which completes the proof. %Applying the dominated convergence theorem then gives (\ref{equation: proof-ex-con-1}),
\end{proof}

%%%%%%%%%%%%%%%%%%%%%%%%%%%%%%%%%%%%%%%%%%%%%%%%%%%%%%%%%%%%%%%%%%%%%%%%%%%%%
\section{Proofs}
\label{appendix: proof of lemmas}
\begin{proof}[Proof of Lemma~\ref{lemma: property of SDE-1}]
By assumptions (A1)-(A3), the results of SDE (\ref{equation: forward-time SDE}) can be obtained by directly using \cite[Theorem~4.5.4]{peter1992numerical}. 

Next we consider the results about SDE (\ref{equation: reverse-time generative SDE}). Let $\hat{\bold{f}}(t, \mathrm{x})  =  - \bold{f}(T-t,  \mathrm{x}) + \boldsymbol{g}^2(T-t)  \bold{s}_{\Theta}(T-t, \mathrm{x})$. Using assumptions (A1)-(A3), (A5)(A6) and the property of $\bold{s}_{\Theta}$ in Lemma~\ref{lemma: property of DNN}, we have 
\begin{equation}
	\begin{aligned}
		&	\|\hat{\bold{f}}(t, \mathrm{x}) \|_2 \leq \|\bold{f}(T-t,  \mathrm{x})\|_2 + \boldsymbol{g}^2(T-t) \| \bold{s}_{\Theta}(T-t, \mathrm{x})\|_2 \\
		\leq & C_{\bold{f}} (1+ \|\mathrm{x}\|_2) + B_{\boldsymbol{g}} ^2 \big( \|\mathrm{x}\|_2 + L \mathrm{Lip}_{\psi} B_{\boldsymbol{e}}  B_{\Theta}^2  \big) \exp(L \mathrm{Lip}_{\psi}  B_{\Theta}^2) \\
		\leq & C_{\hat{\bold{f}}}^{\prime} ( 1 + \|\mathrm{x}\|_2),
	\end{aligned}
\end{equation}
where $C_{\hat{\bold{f}}}^{\prime} = C_{\bold{f}} +B_{\boldsymbol{g}} ^2 \exp(L \mathrm{Lip}_{\psi}  B_{\Theta}^2) \cdot \max \{1, L \mathrm{Lip}_{\psi} B_{\boldsymbol{e}}  B_{\Theta}^2 \} $. Similarly, we can obtain
\begin{equation}
	\begin{aligned}
		&	\|\hat{\bold{f}}(t, \mathrm{x}) - \hat{\bold{f}}(t, \tilde{\mathrm{x}}) \|_2 \\
		\leq&  \|\bold{f}(T-t,  \mathrm{x}) - \bold{f}(T-t,  \tilde{\mathrm{x}})\|_2 + \boldsymbol{g}^2(T-t) \| \bold{s}_{\Theta}(T-t, \mathrm{x}) - \bold{s}_{\Theta}(T-t, \tilde{\mathrm{x}}) \|_2 \\
		\leq & C_{\bold{f}} \|\mathrm{x} -\tilde{\mathrm{x}}  \|_2 + B_{\boldsymbol{g}} ^2 \exp(L \mathrm{Lip}_{\psi}  B_{\Theta}^2) \|\mathrm{x} -\tilde{\mathrm{x}}  \|_2 := C_{\hat{\bold{f}}}^{\prime \prime}  \|\mathrm{x} -\tilde{\mathrm{x}}  \|_2,
	\end{aligned}
\end{equation}
where $C_{\hat{\bold{f}}}^{\prime \prime} =  C_{\bold{f}} + B_{\boldsymbol{g}} ^2 \exp(L \mathrm{Lip}_{\psi}  B_{\Theta}^2)$. Let 
$$
C_{\hat{\bold{f}}} = \max\{C_{\hat{\bold{f}}}^{\prime} , C_{\hat{\bold{f}}}^{\prime \prime} \} =  C_{\bold{f}} +B_{\boldsymbol{g}} ^2 \exp(L \mathrm{Lip}_{\psi}  B_{\Theta}^2) \cdot \max \{1, L \mathrm{Lip}_{\psi} B_{\boldsymbol{e}}  B_{\Theta}^2 \}.
$$
We can apply \cite[Theorem~4.5.4]{peter1992numerical} again to complete the proof.  
\end{proof}

\begin{proof}[Proof of Lemma~\ref{lemma: property of DNN}]
Let $\bold{z}^{l} $ denote the state variable of (\ref{equation: ResNets}) at $l$-layer with parameter $\Theta$, input $(t, \mathrm{x})$, and $\tilde{\bold{z}}^{l} $ the state variable of (\ref{equation: ResNets}) at $l$-layer with parameter $\tilde{\Theta}$, input $(\tilde{t}, \tilde{\mathrm{x}})$, respectively. 

We estimate $\bold{z}^{l} $. By Eq.~(\ref{equation: ResNets}) and assumptions $(A5),(A6)$, we have
\begin{equation}
	\begin{aligned}
		\|\bold{z}^{l} \|_2 \leq \|\bold{z}^{l-1}\|_2 + \mathrm{Lip}_{\psi} \|{\Theta}\|_{\mathcal{E}^{L}}^2 \|\bold{z}^{l-1}\|_2 + \mathrm{Lip}_{\psi} B_{\boldsymbol{e}} \|{\Theta}\|_{\mathcal{E}^{L}}^2, \ 1\leq l\leq L.
	\end{aligned}
	\nonumber
\end{equation}
Therefore,
\begin{equation}
	\begin{aligned}
		\|\bold{z}^{l} \|_2 \leq & \Big( \|\bold{z}^{0}\|_2 + \sum_{k=1}^{L} \mathrm{Lip}_{\psi} B_{\boldsymbol{e}} \|{\Theta}\|_{\mathcal{E}^{L}}^2  \Big)  \exp\Big(\sum_{k=1}^{L} \mathrm{Lip}_{\psi} \|{\Theta}\|_{\mathcal{E}^{L}}^2 \Big) \\
		\leq &  \Big( \|\mathrm{x}\|_2 + L \mathrm{Lip}_{\psi} B_{\boldsymbol{e}}  B_{\Theta}^2  \Big)  \exp(L \mathrm{Lip}_{\psi}  B_{\Theta}^2) =: B_{\bold{s}} (\mathrm{x}),
	\end{aligned}
	\nonumber
\end{equation}
due to the discrete-time Gr\"onwall's inequality \cite{emmrich1999discrete}.

Next, we estimate $\|\bold{z}^{l}  - \tilde{\bold{z}}^{l}\|_2$. By (\ref{equation: ResNets}), assumptions $(A5),(A6)$ and the triangle inequality, we derive
\begin{equation}
	\begin{aligned}
		& \|\bold{z}^{l}  - \tilde{\bold{z}}^{l} \|_2 \\ \textcolor{red}{\leq} & \|\bold{z}^{l-1}  - \tilde{\bold{z}}^{l-1}\|_2 
		+ \|\mathrm{W}^{l}\psi \circ \big(\mathrm{V}^{l}\bold{z}^{l-1}  + \mathrm{U}^{l} \boldsymbol{e}(t)\big) -  \tilde{\mathrm{W}}^{l}\psi \circ \big(\tilde{\mathrm{V}}^{l}\tilde{\bold{z}}^{l-1} + \tilde{\mathrm{U}}^{l} \boldsymbol{e}(\tilde{t})\big)\|_2  \\
		\leq & \|\bold{z}^{l-1}  - \tilde{\bold{z}}^{l-1}\|_2  + \| \mathrm{W}^{l} - \tilde{\mathrm{W}}^{l} \|_2 \cdot \| \psi \circ (\mathrm{V}^{l}\bold{z}^{l-1} + \mathrm{U}^{l} \boldsymbol{e}(t))\|_2  \\
		& + \|\tilde{\mathrm{W}}^{l} \|_2 \cdot
		\| \psi \circ \big(\mathrm{V}^{l}\bold{z}^{l-1} + \mathrm{U}^{l} \boldsymbol{e}(t) \big) - \psi \circ \big(\tilde{\mathrm{V}}^{l}\tilde{\bold{z}}^{l-1}+ \tilde{\mathrm{U}}^{l} \boldsymbol{e}(\tilde{t}) \big) \|_2 \\
		%%%%%%%%%%%%%%%%%%%%%%%%%%%%%%%%%%%%%%%%%%%%%%%%%%
		%%%%%%%%%%%%%%%%%%%%%%%%%%%%%%%%%%%%%%%%%%%%%%%%%%
		\leq & \|\bold{z}^{l-1}  - \tilde{\bold{z}}^{l-1}\|_2 +   \mathrm{Lip}_{\psi} \|{\Theta}\|_{\mathcal{E}^{L}} \| \Theta - \tilde{\Theta} \|_{\mathcal{E}^{L}} (B_{\bold{s}} (\mathrm{x}) + B_{\boldsymbol{e}}) \\
		& +   \mathrm{Lip}_{\psi} \|\tilde{\Theta}\|_{\mathcal{E}^{L}} \Big(\|\mathrm{V}^{l} - \tilde{\mathrm{V}}^{l} \|_2 \|\bold{z}^{l-1}\|_2 + \|\tilde{\mathrm{V}}^{l}\|_2 \|\bold{z}^{l-1}  - \tilde{\bold{z}}^{l-1}\|_2 \\
		& \qquad \qquad \qquad \quad  + \|\mathrm{U}^{l} - \tilde{\mathrm{U}}^{l} \|_2 \|\boldsymbol{e}(t)\|_2 + \|\tilde{\mathrm{U}}^{l}\|_2 \|\boldsymbol{e}(t) - \boldsymbol{e}(\tilde{t})\|_2 \Big) \\
		%%%%%%%%%%%%%%%%%%%%%%%%%%%%%%%%%%%%%%%%%%%%%%%%%%
		%%%%%%%%%%%%%%%%%%%%%%%%%%%%%%%%%%%%%%%%%%%%%%%%%%
		\leq & \|\bold{z}^{l-1}  - \tilde{\bold{z}}^{l-1}\|_2+ \mathrm{Lip}_{\psi} B_{\Theta}^2 \|\bold{z}^{l-1}  - \tilde{\bold{z}}^{l-1}\|_2 \\
		& \qquad \qquad  \qquad+  2\mathrm{Lip}_{\psi} B_{\Theta} (B_{\bold{s}} (\mathrm{x}) + B_{\boldsymbol{e}}) \| \Theta - \tilde{\Theta} \|_{\mathcal{E}^{L}} + \mathrm{Lip}_{\psi} \mathrm{Lip}_{\boldsymbol{e}} B_{\Theta}^2 |t - \tilde{t}|.
	\end{aligned}
	\nonumber
\end{equation}
Applying the discrete-time Gr\"onwall's inequality again and using $\|{\Theta}\|_{\mathcal{E}^{L}}  \leq B_{\Theta}$, we obtain
\begin{equation}
	\begin{aligned}
		& \|\bold{z}^{l}  - \tilde{\bold{z}}^{l} \|_2 \\
		%	\leq & \big( \|\mathrm{x} \!-\! \tilde{\mathrm{x}}\|_2 \!+\! L\big(2\mathrm{Lip}_{\psi} (B_{\bold{s}} (\mathrm{x}) \!+\! B_{\boldsymbol{e}}) \|{\Theta}\|_{\mathcal{E}^{L}} \| \Theta \!  - \!  \tilde{\Theta} \|_{\mathcal{E}^{L}} \! + \!  \mathrm{Lip}_{\psi} \mathrm{Lip}_{\boldsymbol{e}} \|{\Theta}\|_{\mathcal{E}^{L}}^2 |t \! -\!  \tilde{t}|\big)  \big) \!  \exp(L \mathrm{Lip}_{\psi} \|{\Theta}\|_{\mathcal{E}^{L}}^2) \\
		%%%%%%%%%%%%%%%%%%%%%%%%%%%%%%%%%%%%%
		\leq & \Big( \|\mathrm{x} \!-\! \tilde{\mathrm{x}}\|_2 \!+\! L\big[2\mathrm{Lip}_{\psi}  B_{\Theta} (B_{\bold{s}} (\mathrm{x}) \!+\! B_{\boldsymbol{e}})  \| \Theta - \tilde{\Theta} \|_{\mathcal{E}^{L}} + \mathrm{Lip}_{\psi} \mathrm{Lip}_{\boldsymbol{e}}  B_{\Theta}^2 |t - \tilde{t}|\big] \Big)  \exp(L \mathrm{Lip}_{\psi}  B_{\Theta}^2), 
		%%%%%%%%%%%%%%%%%%%%%%%%%%%%%%%%%%%%%%%%%%%%%%%%%%
		%	\leq & C_{\bold{s}}(\mathrm{x}) \Big( \|\mathrm{x} \!-\! \tilde{\mathrm{x}}\|_2 + \| \Theta - \tilde{\Theta} \|_{\mathcal{E}^{L}} + |t - \tilde{t}| \Big), 
	\end{aligned}
	\nonumber
\end{equation}
%where $C_{\bold{s}}(\mathrm{x}) =  \max \{1, 2L\mathrm{Lip}_{\psi} (B_{\bold{s}} (\mathrm{x}) \!+\! B_{\boldsymbol{e}})  B_{\Theta}, L\mathrm{Lip}_{\psi} \mathrm{Lip}_{\boldsymbol{e}} \} \cdot \exp(L \mathrm{Lip}_{\psi}  B_{\Theta}^2) $. This
which completes the proof.
\end{proof}

\begin{proof}[proof of Lemma~\ref{lemma: proeprty of dnn state}]
Since $\bold{x}_t$ is the solution of (\ref{equation: forward-time SDE}) with initialization $\mathrm{x}\sim p_{\rm data}$, we have, by assumptions (A1-A6) and Lemma~\ref{lemma: property of SDE-1} that 
$$
\begin{aligned}
	\mathbb{E}[B_{\bold{s}}^2(\bold{x}_t)] = & \mathbb{E} [\big( \|\bold{x}_t\|_2 + L \mathrm{Lip}_{\psi} B_{\boldsymbol{e}}  B_{\Theta}^2  \big)^2  \exp(2L \mathrm{Lip}_{\psi}  B_{\Theta}^2)] \\
	\leq & 2\big( \mathbb{E}[\|\bold{x}_t\|_2^2] + L^2 \mathrm{Lip}_{\psi}^2 B_{\boldsymbol{e}}^2  B_{\Theta}^{4} \big) \exp(2L \mathrm{Lip}_{\psi}  B_{\Theta}^2)  \\
	%	\leq & 2 \Big( \big( 1 + \mathbb{E}[\|\mathrm{x}\|_2^2]  \big) \exp(12C_{\bold{f}}^2 t) + L^2 \mathrm{Lip}_{\psi}^2 B_{\boldsymbol{e}}^2  B_{\Theta}^{4}  \Big) \exp(2L \mathrm{Lip}_{\psi}  B_{\Theta}^2) \\
	\leq & 2 \Big( \big( 1 + R^2  \big) \exp(12C_{\bold{f}}^2 T) + L^2 \mathrm{Lip}_{\psi}^2 B_{\boldsymbol{e}}^2  B_{\Theta}^{4}  \Big) \exp(2L \mathrm{Lip}_{\psi}  B_{\Theta}^2) =: B_{\bold{s},2}, \\
	\mathbb{E}[C_{\bold{s}}^2(\bold{x}_t)] 
	=&  \mathbb{E} \big[ 4 L^2\mathrm{Lip}_{\psi}^2 2 (B_{\bold{s}}^2 (\bold{x}_t) \!+\! B^2_{\boldsymbol{e}})  B_{\Theta}^2 \cdot \exp(2 L \mathrm{Lip}_{\psi}  B_{\Theta}^2) \big] \\
	\leq & 8 L^2\mathrm{Lip}_{\psi}^2 ( B_{\bold{s},2} + B^2_{\boldsymbol{e}})  B_{\Theta}^2  \exp(2 L \mathrm{Lip}_{\psi}  B_{\Theta}^2) =: C_{\bold{s},2}, 
\end{aligned}
$$
%	$$
%	\begin{aligned}
	%		&	\mathbb{E}[B_{\bold{s}}^i(\bold{x}_t)] = \mathbb{E} [\big( \|\bold{x}_t\|_2 + L \mathrm{Lip}_{\psi} B_{\boldsymbol{e}}  B_{\Theta}^2  \big)^i  \exp(iL \mathrm{Lip}_{\psi}  B_{\Theta}^2)] \\
	%		\leq & 2^{i-1}\big( \mathbb{E}[\|\bold{x}_t\|_2^i] + L^i \mathrm{Lip}_{\psi}^i B_{\boldsymbol{e}}^i  B_{\Theta}^{2i} \big) \exp(iL \mathrm{Lip}_{\psi}  B_{\Theta}^2)  \\
	%		\leq & 2^{i-1} \Big( \big( 1 + \mathbb{E}[\|\mathrm{x}\|_2^i]  \big) \exp(2i(i+1)C_{\bold{f}}^2 t) + L^i \mathrm{Lip}_{\psi}^i B_{\boldsymbol{e}}^i  B_{\Theta}^{2i}  \Big) \exp(iL \mathrm{Lip}_{\psi}  B_{\Theta}^2) \\
	%		\leq & 2^{i-1} \Big( \big( 1 + R^i  \big) \exp(2i(i+1)C_{\bold{f}}^2 T) + L^i \mathrm{Lip}_{\psi}^i B_{\boldsymbol{e}}^i  B_{\Theta}^{2i}  \Big) \exp(iL \mathrm{Lip}_{\psi}  B_{\Theta}^2), \\
	%		&	\mathbb{E}[C_{\bold{s}}^i(\bold{x}_t)] \\
	%		=&  \mathbb{E} \big[\max \{1, 2^i L^i\mathrm{Lip}_{\psi}^i 2^{i-1} (B_{\bold{s}}^i (\bold{x}_t) \!+\! B^i_{\boldsymbol{e}})  B_{\Theta}^i, L^i\mathrm{Lip}^i_{\psi} \mathrm{Lip}^i_{\boldsymbol{e}} \} \cdot \exp(i L \mathrm{Lip}_{\psi}  B_{\Theta}^2) \big] \\
	%		\leq & \big( 1 + 2^i L^i\mathrm{Lip}_{\psi}^i 2^{i-1} ( B_{\bold{s},i} + B^i_{\boldsymbol{e}})  B_{\Theta}^i +  L^i\mathrm{Lip}^i_{\psi} \mathrm{Lip}^i_{\boldsymbol{e}}  \big) \exp(i L \mathrm{Lip}_{\psi}  B_{\Theta}^2), 
	%	\end{aligned}
%	$$
for $t \in [0,T]$. 
In addition, by assumptions (A1-A6),  (\ref{proof inequality: estimation of expection of states}) together with Lemma~\ref{lemma: property of SDE-1}, that
\begin{equation}
	\begin{aligned}
		\mathbb{E} \left[\big\|\bold{s}_{\Theta}(t, \bold{x}_{t}) -  \nabla \ln \boldsymbol{p}(t, \bold{x}_{t}|\mathrm{x})\big\|_2^2\right] 
		\leq &  2 \mathbb{E}  \Big[
		\big\|\bold{s}_{\Theta}(t, {\bold{x}_{t}}) \|^2_2 + \| \nabla \ln \boldsymbol{p}(t , \bold{x}_{t}|\mathrm{x})\|^2_2 \Big] \\
		\leq & 2 B_{\bold{s},2} + 6 C_{\boldsymbol{p}} \mathbb{E}\big[1+\|\bold{x}_t\|_2^2 + \|\mathrm{x}\|_2^2 \big] \\
		\leq & 2 B_{\bold{s},2} + 6C_{\boldsymbol{p}} \Big(1+(1+R^2) e^{6C_{\bold{f}}^2T} + R^2 \Big),	
	\end{aligned}
	\nonumber
\end{equation}
which completes th proof.
\end{proof}

\bibliography{references}

@article{huang2024on,
year = {2024},
month = {jun},
publisher = {IOP Publishing},
volume = {40},
number = {7},
pages = {075006},
author = {Jinshu Huang and Yiming Gao and Chunlin Wu},
title = {On dynamical system modeling of learned primal-dual with a linear operator $\mathcal{K}$: stability and convergence properties},
journal = {Inverse Problems}
}

@article{thorpe2018deep,
  title={Deep limits of residual neural networks},
  author={Thorpe, Matthew and van Gennip, Yves},
  journal={Research in the Mathematical Sciences},
  volume={10},
  number={1},
  pages={6},
  year={2023},
  publisher={Springer}
}

@book{shalev2014understanding,
  title={Understanding machine learning: From theory to algorithms},
  author={Shalev-Shwartz, Shai and Ben-David, Shai},
  year={2014},
  publisher={Cambridge university press}
}

@inproceedings{song2021scorebased,
  title={Score-based generative modeling through stochastic differential equations},
  author={Yang Song and Jascha Sohl-Dickstein and Diederik P Kingma and Abhishek Kumar and Stefano Ermon and Ben Poole},
  booktitle={International Conference on Learning Representations},
  year={2021}
}

@article{schnoor2023generalization, 
  title={Generalization error bounds for iterative recovery algorithms unfolded as neural networks},
  author={Schnoor, Ekkehard and Behboodi, Arash and Rauhut, Holger},
  journal={Information and Inference: A Journal of the IMA},
  volume={12},
  number={3},
  pages={2267--2299},
  year={2023},
  publisher={Oxford University Press}
}

@book{emmrich1999discrete,
  title={Discrete versions of Gronwall's lemma and their application to the numerical analysis of parabolic problems},
  author={Emmrich, Etienne},
  year={1999},
  publisher={Techn. Univ.}
}

@article{vincent2011connection,
  title={A connection between score matching and denoising autoencoders},
  author={Vincent, Pascal},
  journal={Neural computation},
  volume={23},
  number={7},
  pages={1661--1674},
  year={2011},
  publisher={MIT Press}
}

@book{peter1992numerical,
  title={Numerical solution of stochastic differential equations},
  author={Peter E. Kloeden, Eckhard Platen},
  year={1992},
  publisher={Springer Berlin, Heidelberg}
}

@book{le2016brownian,
  title={Brownian motion, martingales, and stochastic calculus},
  author={Le Gall, Jean-Fran{\c{c}}ois},
  year={2016},
  publisher={Springer}
}

@book{foucart2013,
  author    = {Foucart, Simon and Rauhut, Holger},
  title     = {A mathematical introduction to compressive sensing},
  year      = {2013},
  series    = {Applied and Numerical Harmonic Analysis},
  publisher = {Springer},
  address   = {New York, NY}
}

@book{chung2000course,
  title={A course in probability theory},
  author={Chung, Kai Lai},
  year={2000},
  publisher={Elsevier}
}

@book{van2000asymptotic,
  title={Asymptotic statistics},
  author={Van der Vaart, Aad W},
  volume={3},
  year={2000},
  publisher={Cambridge University Press}
}

@inproceedings{chen2023sampling,
  title     = {Sampling is as easy as learning the score: Theory for diffusion models with minimal data assumptions},
  author    = {Chen, Sitan and Chewi, Sinho and Li, Jerry and Li, Yuanzhi and Salim, Adil and Zhang, Anru R.},
  booktitle = {International Conference on Learning Representations (ICLR)},
  year      = {2023},
  url       = {https://openreview.net/forum?id=zyLVMgsZ0U_},
  eprint    = {2209.11215}
}

@book{ambrosio2005gradient,
  title={Gradient flows: in metric spaces and in the space of probability measures},
  author={Ambrosio, Luigi and Gigli, Nicola and Savar{\'e}, Giuseppe},
  year={2005},
  publisher={Springer}
}

@article{ho2020denoising,
  title={Denoising diffusion probabilistic models},
  author={Ho, Jonathan and Jain, Ajay and Abbeel, Pieter},
  journal={Advances in Neural Information Processing Systems},
  volume={33},
  pages={6840--6851},
  year={2020}
}

@article{song2019generative,
  title={Generative modeling by estimating gradients of the data distribution},
  author={Song, Yang and Ermon, Stefano},
  journal={Advances in Neural Information Processing Systems},
  volume={32},
  year={2019}
}

@book{oksendal2013sde,
  author    = {Oksendal, Bernt},
  title     = {Stochastic differential equations: An introduction with applications},
  year      = {2013},
  publisher = {Springer Science \& Business Media},
  address   = {Berlin, Heidelberg}
}

@article{tang2025score,
author = {Wenpin Tang and Hanyang Zhao},
title = {Score-based diffusion models via stochastic differential equations},
volume = {19},
journal = {Statistics Surveys},
publisher = {Amer. Statist. Assoc., the Bernoulli Soc., the Inst. Math. Statist., and the Statist. Soc. Canada},
pages = {28 -- 64},
year = {2025},
doi = {10.1214/25-SS152},
URL = {https://doi.org/10.1214/25-SS152}
}

@article{li2023generalization,
  title={On the generalization properties of diffusion models},
  author={Li, Puheng and Li, Zhong and Zhang, Huishuai and Bian, Jiang},
  journal={Advances in Neural Information Processing Systems},
  volume={36},
  pages={2097--2127},
  year={2023}
}

@inproceedings{lee2023convergence,
  title={Convergence of score-based generative modeling for general data distributions},
  author={Lee, Holden and Lu, Jianfeng and Tan, Yixin},
  booktitle={International Conference on Algorithmic Learning Theory},
  pages={946--985},
  year={2023},
  organization={PMLR}
}

@article{de2022convergence,
  title={Convergence of denoising diffusion models under the manifold hypothesis},
  author={De Bortoli, Valentin},
  journal={arXiv preprint arXiv:2208.05314},
  year={2022}
}

@article{lee2022convergence,
  title={Convergence for score-based generative modeling with polynomial complexity},
  author={Lee, Holden and Lu, Jianfeng and Tan, Yixin},
  journal={Advances in Neural Information Processing Systems},
  volume={35},
  pages={22870--22882},
  year={2022}
}

@article{song2021maximum,
  title={Maximum likelihood training of score-based diffusion models},
  author={Song, Yang and Durkan, Conor and Murray, Iain and Ermon, Stefano},
  journal={Advances in neural information processing systems},
  volume={34},
  pages={1415--1428},
  year={2021}
}

@article{anderson1982reverse,
  title={Reverse-time diffusion equation models},
  author={Anderson, Brian DO},
  journal={Stochastic Processes and their Applications},
  volume={12},
  number={3},
  pages={313--326},
  year={1982},
  publisher={Elsevier}
}

@ARTICLE{huang2025survey,
  author={Huang, Yi and Huang, Jiancheng and Liu, Yifan and Yan, Mingfu and Lv, Jiaxi and Liu, Jianzhuang and Xiong, Wei and Zhang, He and Cao, Liangliang and Chen, Shifeng},
  journal={IEEE Transactions on Pattern Analysis and Machine Intelligence}, 
  title={Diffusion Model-Based Image Editing: A Survey}, 
  year={2025},
  volume={47},
  number={6},
  pages={4409-4437},
  doi={10.1109/TPAMI.2025.3541625}}

@article{yang2023comprehensivesurvey,
author = {Yang, Ling and Zhang, Zhilong and Song, Yang and Hong, Shenda and Xu, Runsheng and Zhao, Yue and Zhang, Wentao and Cui, Bin and Yang, Ming-Hsuan},
title = {Diffusion Models: A Comprehensive Survey of Methods and Applications},
year = {2023},
issue_date = {April 2024},
publisher = {Association for Computing Machinery},
address = {New York, NY, USA},
volume = {56},
number = {4},
issn = {0360-0300},
url = {https://doi.org/10.1145/3626235},
doi = {10.1145/3626235},
journal = {ACM Comput. Surv.},
month = nov,
articleno = {105},
numpages = {39}
}

@article{ramesh2022hierarchical,
  title={Hierarchical text-conditional image generation with clip latents},
  author={Ramesh, Aditya and Dhariwal, Prafulla and Nichol, Alex and Chu, Casey and Chen, Mark},
  journal={arXiv preprint arXiv:2204.06125},
  volume={1},
  number={2},
  pages={3},
  year={2022}
}

@article{saharia2022photorealistic,
  title={Photorealistic text-to-image diffusion models with deep language understanding},
  author={Saharia, Chitwan and Chan, William and Saxena, Saurabh and Li, Lala and Whang, Jay and Denton, Emily L and Ghasemipour, Kamyar and Gontijo Lopes, Raphael and Karagol Ayan, Burcu and Salimans, Tim and others},
  journal={Advances in neural information processing systems},
  volume={35},
  pages={36479--36494},
  year={2022}
}

@inproceedings{rombach2022high,
  title={High-resolution image synthesis with latent diffusion models},
  author={Rombach, Robin and Blattmann, Andreas and Lorenz, Dominik and Esser, Patrick and Ommer, Bj{\"o}rn},
  booktitle={Proceedings of the IEEE/CVF conference on computer vision and pattern recognition},
  pages={10684--10695},
  year={2022}
}

@misc{openai2024sora,
  title = {Sora: Creating Video from Text},
  author = {OpenAI},
  year = {2024},
  howpublished = {\url{https://openai.com/sora}},
  note = {Accessed: 2024-11-01}
}

@inproceedings{song2021denoising,
  title={Denoising Diffusion Implicit Models},
  author={Song, Jiaming and Meng, Chenlin and Ermon, Stefano},
  booktitle={International Conference on Learning Representations},
  year={2021}
}

@inproceedings{chen2023improved,
  title={Improved analysis of score-based generative modeling: User-friendly bounds under minimal smoothness assumptions},
  author={Chen, Hongrui and Lee, Holden and Lu, Jianfeng},
  booktitle={International Conference on Machine Learning},
  pages={4735--4763},
  year={2023},
  organization={PMLR}
}

@article{tang2024contractive,
  title={Contractive diffusion probabilistic models},
  author={Tang, Wenpin and Zhao, Hanyang},
  journal={arXiv preprint arXiv:2401.13115},
  year={2024}
}

@article{block2020generative,
  title={Generative modeling with denoising auto-encoders and langevin sampling},
  author={Block, Adam and Mroueh, Youssef and Rakhlin, Alexander},
  journal={arXiv preprint arXiv:2002.00107},
  year={2020}
}

@article{chen2024overview,
  title={An overview of diffusion models: Applications, guided generation, statistical rates and optimization},
  author={Chen, Minshuo and Mei, Song and Fan, Jianqing and Wang, Mengdi},
  journal={arXiv preprint arXiv:2404.07771},
  year={2024}
}

@inproceedings{chen2023score,
  title={Score approximation, estimation and distribution recovery of diffusion models on low-dimensional data},
  author={Chen, Minshuo and Huang, Kaixuan and Zhao, Tuo and Wang, Mengdi},
  booktitle={International Conference on Machine Learning},
  pages={4672--4712},
  year={2023},
  organization={PMLR}
}

@inproceedings{nichol2021improved,
	title={Improved denoising diffusion probabilistic models},
	author={Nichol, Alexander Quinn and Dhariwal, Prafulla},
	booktitle={International conference on machine learning},
	pages={8162--8171},
	year={2021},
	organization={PMLR}
}

@book{tsybakov2008nonparametric,
  title={Introduction to nonparametric estimation},
  author={Tsybakov, Alexandre B},
  year={2008},
  series    = {Springer Series in Statistics},
  publisher = {Springer},
  address   = {New York, NY}
}

@article{paton2023investigating,
  title={Investigating Generalisation in Score-Based Generative Models},
  author={Paton, Peter Christofides and Aky{\i}ld{\i}z, {\"O}mer Deniz and Kantas, Nikolas},
  year={2023}
}

@article{pidstrigach2022score,
  title={Score-based generative models detect manifolds},
  author={Pidstrigach, Jakiw},
  journal={Advances in Neural Information Processing Systems},
  volume={35},
  pages={35852--35865},
  year={2022}
}

@article{bo2022large,
  title={Large sample mean-field stochastic optimization},
  author={Bo, Lijun and Capponi, Agostino and Liao, Huafu},
  journal={SIAM Journal on Control and Optimization},
  volume={60},
  number={4},
  pages={2538--2573},
  year={2022},
  publisher={SIAM}
}

\end{document}